\documentclass[11pt]{article} 
\usepackage{amsfonts}
\usepackage{amsmath}
\usepackage[maxbibnames=99]{biblatex}
\usepackage{amssymb}
\usepackage{subcaption}
\usepackage{graphicx}
\usepackage{authblk}
\usepackage[margin=1.0in]{geometry}
\usepackage{wrapfig}
\usepackage{float}
\usepackage[font={small}]{caption}
\graphicspath{ {/}}

\usepackage{algorithm}
\usepackage{algorithmic}

\usepackage{hyperref}
\usepackage{cilventomacros}

\begin{document}

\title{Fairness Under Composition} 
\date{}
\author{Cynthia Dwork\thanks{Harvard John A. Paulson School of Engineering and Applied Science, Harvard University and Radcliffe Institute for Advanced Study. This work was supported in part by Microsoft Research and the Sloan Foundation.} ~ and 
Christina Ilvento\thanks{Harvard John A. Paulson School of Engineering and Applied Science, Harvard University. This work was supported in part by the Smith Family Fellowship and Microsoft Research.}}

\renewcommand\Authands{ and }

\maketitle
\begin{abstract}

Algorithmic fairness, and in particular the fairness of scoring and classification algorithms, has become a topic of increasing social concern and has recently witnessed an explosion of research in theoretical computer science, machine learning, statistics, the social sciences, and law.  
Much of the literature considers the case of a single classifier (or scoring function) used once, in isolation.  In this work, we initiate the study of the fairness properties of systems composed of algorithms that are fair in isolation; that is, we study {\em fairness under composition}.
We identify pitfalls of na\"{\i}ve composition and give general constructions for fair composition, demonstrating both that classifiers that are fair in isolation do not necessarily compose into fair systems and also that seemingly unfair components may be carefully combined to construct fair systems.
We focus primarily on the individual fairness setting proposed in [Dwork, Hardt, Pitassi, Reingold, Zemel, 2011], but also extend our results 
to a large class of group fairness definitions 
popular in the recent literature, exhibiting
 several cases in which group fairness definitions give misleading signals under composition.
 \end{abstract}

\section{Introduction}
As automated decision-making extends its reach ever more deeply into our lives, there is increasing concern that such decisions be fair. The rigorous theoretical study of fairness in algorithmic classification was initiated by Dwork \textit{et al} in \cite{dwork2012fairness} and subsequent works investigating alternative definitions, fair representations, and impossibility results have proliferated in the machine learning, economics and theoretical computer science literatures.\footnote{See also \cite{pedreshi2008discrimination} \cite{kamiran2009classifying} and \cite{kamishima2011fairness}, which predate \cite{dwork2012fairness} and are motivated by similar concerns.}  The notions of fairness broadly divide into {\em individual fairness}, requiring that individuals who are similar with respect to a given classification task (as measured by a task-specific similarity metric) have similar probability distributions on classification outcomes; and {\em group fairness}, which requires that different demographic groups experience the same treatment in some average sense.

In a bit more detail, a {\em classification task} is the problem of mapping {\em individuals} to {\em outcomes}; for example, a decision task may map individuals to outcomes in $\{0,1\}$.  A classifier is a possibly randomized algorithm solving a classification task. 
In this work we initiate the study {\em fairness under composition}: what are the fairness properties of systems built from classifiers that are fair in isolation?  Under what circumstances can we ensure fairness, and how can we do so? A running example in this work is online advertising. If a set of advertisers compete for the attention of users, say one for tech jobs and one for a grocery delivery service, and each chooses fairly whether to bid (or not), is it the case that the advertising system (including budget handling and tie-breaking) will also be fair?

We identify and examine several types of composition and draw conclusions about auditing systems for fairness, constructing fair systems, and definitions of fairness for systems.  In the remainder of this section we summarize our results and discuss related work.  A full version of this paper, containing complete proofs of all our results, appears in the Appendix.

\paragraph{Task-Competitive Compositions} (Section \ref{section:absmultipletask}). We first consider the problem of two or more tasks competing for individuals, motivated by the online advertising setting described above. 
We prove that two advertisers for different tasks, each behaving fairly (when considered independently), will not necessarily produce fair outcomes when they compete. Intuitively (and empirically observed by \cite{lambrecht2016algorithmic}), the attention of individuals similarly qualified for a job may effectively have different costs due to these individuals' respective desirability for other advertising tasks, like household goods purchases. That is, individuals claimed by the household goods advertiser will not see the jobs ad, regardless of their job qualification.  These results are not specific to an auction setting and are robust to choice of ``tie-breaking'' functions that select among multiple competing tasks (advertisers). 
Nonetheless, we give a simple mechanism, \randomizeandclassifynospace, that solves the fair task-competitive classification problem 
using classifiers for the competing tasks each of which is fair in isolation, in a black-box fashion and without modification. In the full paper in Section \ref{section:positiveweights} we give a second technique for modifying the fair classifier of the lower bidder (loser of the tie-breaking function) in order to achieve fairness. 

\paragraph{Functional Compositions} (Section \ref{section:abssametask}).  When can we build fair classifiers by computing on values that were fairly obtained?  Here we must understand what is the salient outcome of the computation. For example, when reasoning about whether the college admissions system is fair, the salient outcome may be whether a student is accepted to at least one college, and not whether the student is accepted to a specific college\footnote{In this simple example, we assume that all colleges are equally desirable, but it is not difficult to extend the logic to different sets of comparable colleges.}. Even if each college uses a fair classifier, the question is whether the ``OR'' of the colleges' decisions is fair.  Furthermore, an acceptance to college may not be meaningful without sufficient accompanying financial aid. Thus in practice, we must reason about the OR of ANDs of acceptance and financial aid across many colleges. We show that although in general there are no guarantees on the fairness of functional compositions of fair components, there are some cases where fairness in ORs can be satisfied. Such reasoning can be used in many applications where long-term and short-term measures of fairness must be balanced. In the case of feedback loops, where prior positive outcomes can improve the chances of future positive outcomes, functional composition provides a valuable tool for determining at which point(s) fairness must be maintained and determining whether the existing set of decision procedures will adhere to these requirements.


\paragraph{Dependent Compositions} (Section \ref{section:abscohort}).  
There are many settings in which each individual's classifications are dependent on the classifications of others. For example, if a company is interviewing a set of job candidates in a particular order, accepting a candidate near the beginning of the list precludes any subsequent candidates from even being considered. Even if each candidate is considered fairly in isolation, dependence between candidates can result in highly unfair outcomes. For example, individuals who are socially connected to the company through friends or family are likely to hear about job openings first and thus be considered for a position before candidates without connections. We show that selecting a cohort of people -- online or offline -- requires care to prevent dependencies from impacting an independently fair selection mechanism. We address this in the offline case with two randomized constructions, \permutethenclassify and \weightedsamplingnospace.  These algorithms can be applied in the online case, even under adversarial ordering, provided the size of the universe of individuals is known; when this is not known there is no solution.


\paragraph{Nuances of group-based definitions}(Section \ref{section:absgroup}). 
Many fairness definitions in the literature seek to provide fairness guarantees based on group-level statistical properties. For example, Equal Opportunity~\cite{hardt2016equality} requires that, conditioned on qualification, the probability of a positive outcome is independent of protected attributes such as race or gender. Group Fairness definitions have practical appeal in that they are possible to measure and enforce empirically without reference to a task-specific similarity metric. We extend our results to group fairness definitions and we also show that these definitions do not always yield consistent signals under composition. 
In particular, we show that the intersectional subgroup concerns (which motivate \cite{kearns2017gerrymandering,hebert2017calibration}) are exacerbated by composition. For example, an employer who uses group fairness definitions to ensure parity with respect to race and gender may fail to identify that ``parents'' of particular race and gender combinations are not treated fairly. Task-competitive composition exacerbates this problem, as the employers may be prohibited from even collecting parental status information, but their hiring processes may be composed with other systems which legitimately differentiate based on parental status. 

Finally, we also show how na\"{i}ve strategies to mitigate these issues in composition may result in learning a nominally fair solution that is clearly discriminating against a socially meaningful subgroup not officially called out as ``protected,'' from which we conclude that understanding the behavior of fairness definitions under composition is critical for choosing which definition is meaningful in a given setting.

\paragraph{Implications of Our Results.} 
Our composition results have several practical implications. First, testing individual components without understanding of the whole system will be insufficient to safely draw either positive or negative conclusions about the fairness of the system.  
Second, composition properties are an important point of evaluation for any definitions of fairness or fairness requirements imposed by law or otherwise. Failing to take composition into account when specifying a group-based fairness definition may result in a meaningless signal under composition, or worse may lead to ingraining poor outcomes for certain subgroups while still nominally satisfying fairness requirements.
Third, understanding of the salient outcomes on which to measure and enforce fairness is critical to building meaningfully fair systems. 
Finally, we conclude that there is significant potential for improvement in the mechanisms proposed for fair composition and many settings in which new mechanisms could be proposed.

\subsection{Related Work}\label{section:relatedwork}
Fairness retained under post-processing in the single-task one-shot setting is central in \cite{zemel2013learning,madras2018learning,dwork2012fairness}. 
The definition of individual fairness we build upon in this work was introduced by Dwork \textit{et al} in \cite{dwork2012fairness}. 
Learning with oracle access to the fairness metric is considered by \cite{gillen2018online,kim2018fairness}. 
A number of group-based fairness definitions have been proposed, and Ritov \textit{et al} provide a combined discussion of the parity-based definitions in \cite{ritov2017conditional}. 
In particular, their work includes discussion of Hardt \textit{et al}'s Equality of Opportunity and Equal Odds definitions and Kilbertus \textit{et al}'s Counterfactual Fairness \cite{hardt2016equality,kilbertus2017avoiding}. 
Kleinberg \textit{et al} and Chouldechova independently described several impossibility results related to simultaneously satisfying multiple group fairness conditions in single classification settings \cite{DBLP:journals/corr/KleinbergMR16},\cite{chouldechova2017fair}.

Two concurrent lines of work aiming to bridge the gap between individual and group consider ensuring fairness properties for large numbers of large groups and their (sufficiently large) intersections~\cite{kearns2017gerrymandering,hebert2017calibration}.
While these works consider the one-shot, single-task setting,
we will see that group intersection properties are of particular importance under composition. 
Two subsequent works in this general vein explore approximating individual fairness with the help of an oracle that knows the task-specific metric \cite{kim2018fairness,gillen2018online}.  
Several works also consider how feedback loops can influence fair classification \cite{DBLP:journals/corr/HuC17,liu2018delayed}.

There are several empirical or observational studies which document the effects of multiple task composition. For example, Lambrecht and Tucker study how intended gender-neutral advertising can result in uneven delivery due to high demand for the attention of certain demographics \cite{lambrecht2016algorithmic}. Datta \textit{et al} also document differences in advertising based on gender, although they are agnostic as to whether the cause is due to multiple task composition or discriminatory behavior on the part of the advertisers or platform \cite{datta2015automated}.
Whether it is truly ``fair'' that, say, home goods advertisers bid more highly for the attention of women than for the attention of men, may be debatable, although there are clearly instances in which differential targeting is justified, such as maternity clothes.  This {\em actuarial fairness} is the industry practice, so we pose a number of examples in this framework and analyze the implications of composition.

\section{Preliminary Definitions and Assumptions}\label{section:absabspreliminaries}
\subsection{General Terminology}
We refer to classifiers as being ``fair in isolation'' or ``independently fair'' to indicate that with no composition, the classifier satisfies a particular fairness definition. In such cases  expectation and probability are taken over the randomness of the classification procedure and, for group fairness, selection of elements from the universe. We denote the universe of individuals relevant for a task as $U$, and we generally use $u,v,w \in U$ to refer to universe elements. 
We generally consider binary classifiers in this work, and use $p_w$ to denote the probability of assigning the positive outcome (or simply 1) to the element $w$ for a particular classifier. We generally write $C: U \times \{0,1\}^* \rightarrow \{0,1\}$, where $ \{0,1\}^*$ represents the random bits of the classifier. This allows us to comfortably express the probability of positive classification $\E_r[C(u)]$ as well as the output of the classifier under particular randomness $C(u,r)$. In this notation, $p_u = \E_r[C(u)]$. When considering the distribution on outputs of a classifier $C$, we use $\dist{C}:U \rightarrow \Delta(\{0,1\})$. 
When two or more classifiers or tasks are compared, we either use a subscript $_i$ to indicate the $i^{th}$ classifier or task, or a prime $(')$ to indicate the second classifier or task. For example $\{C,C'\}$, $\{C_i | i \in [k]\}$, $\{T,T'\}$, $\{T_i| i \in [k]\}$.

\subsection{Individual Fairness}
Throughout this work, our primary focus is on {\em individual fairness}, proposed by Dwork \textit{et al} in \cite{dwork2012fairness}.
As noted above, a {\em classification task} is the problem of mapping {\em individuals} in a universe to {\em outcomes}.

\begin{definition}[Individual Fairness  \cite{dwork2012fairness}]\label{abs:def:individualfairness}Let $d:\Delta(O) \times \Delta(O)\rightarrow [0,1]$ denote the total variation distance on distributions over $O$\footnote{\cite{dwork2012fairness} also considered other notions of distributional distance.}. Given a universe of individuals $U$, and a task-specific metric $\D$ for a classification task $T$ with outcome set $O$, a randomized classifier $C:U\times \{0,1\}^* \rightarrow O$, such that $\dist{C}: U \rightarrow \Delta(O)$, 
is \textit{individually fair} 
if and only if for all $u,v \in U$, $\D(u,v) \geq d(\dist{C}(u),\dist{C}(v))$.
\end{definition}

Note that when $|O|=2$ we have $d(\dist{C}(u),\dist{C}(v)) = |\E_r[C(u)] - \E_r[C(v)]|=|p_u-p_v|$.
In several proofs we will rely on the fact that it is possible to construct individually fair classifiers with particular distance properties (see Lemma \ref{lemma:fairadditions} and corollaries in Section \ref{section:preliminaries}).

\subsection{Group Fairness}\label{section:absgroupdefs}

In principle, all our individual fairness results extend to group fairness definitions; however, there are a number of technicalities and issues unique to group fairness definitions, which we discuss in Section \ref{section:absgroup}.
Group fairness is often framed in terms of {\em protected attributes} $\A$, such as sex, race, or socio-economic status, while allowing for differing treatment based on a set of \textit{qualifications} $\mathcal{Z}$, such as, in the case of advertising, the willingness to buy an item.
Conditional Parity, a general framework proposed in \cite{ritov2017conditional} for discussing these definitions, conveniently captures many of the popular group fairness definitions popular in the literature including Equal Odds and Equal Opportunity \cite{hardt2016equality}, and Counterfactual Fairness \cite{kusner2017counterfactual}. 
\begin{definition}[Conditional Parity \cite{ritov2017conditional}]\label{abs:def:conditionalparity} 
A random variable $\mathbf{x}$ satisfies parity with respect to $\mathbf{a}$ conditioned on $\mathbf{z}=z$ if the distribution of $\mathbf{x}$ $|$ $(\mathbf{a},\{\mathbf{z}=z\})$ is constant in $\mathbf{a}$:\\
$\Pr[\mathbf{x}=x \text{ }|\text{ }( \mathbf{a}=a,\mathbf{z}=z)] = \Pr[\mathbf{x}=x \text{ }|\text{ }( \mathbf{a}=a',\mathbf{z}=z)]$
for any $a,a' \in \mathcal{A}$.  Similarly, $\mathbf{x}$ satisfies parity with respect to $\mathbf{a}$ conditioned on $\mathbf{z}$ (without specifying a value of $\mathbf{z}$) if it satisfies parity with respect to $\mathbf{a}$ conditioned on $\mathbf{z}=z$ for all $z \in \mathcal{Z}$. All probabilities are over the randomness of the prediction procedure and the selection of elements from the universe.
\end{definition}

\section{Multiple-Task Composition}\label{section:absmultipletask}
First, we consider the problem of composition of classifiers for multiple tasks where the outcome for more than one task is decided. 
Multiple Task Fairness, defined next, requires fairness to be enforced independently and simultaneously for each task.

\begin{definition}[Multiple Task Fairness]\label{abs:def:multipletaskfairness} For a set $\mathcal{T}$ of $k$ tasks with metrics $\mathcal{D}_1, \ldots, \mathcal{D}_k$,  a (possibly randomized) system $\mathcal{S}: U \times r \rightarrow \{0,1\}^k$, which assigns outputs for task $i$ in the $i^{th}$ coordinate of the output, satisfies multiple task fairness if for all $i \in [k]$ and all $u,v \in U$
$\mathcal{D}_i(u,v) \geq |\E[\mathcal{S}_i(u)]-\E[\mathcal{S}_i(v)]|$ 
where $\E[\mathcal{S}_i(u)]$ is the expected outcome for the $i^{th}$ task in the system $\mathcal{S}$ and where the expectation is over the randomness of the system and all its components. 
\end{definition}

\subsection{Task-Competitive Composition}
We now pose the relevant problem for multiple task fairness: competitive composition.

\begin{definition}[Single Slot Composition Problem]\label{abs:def:singleslotproblem}
A (possibly randomized) system $\mathcal{S}$ is said to be a solution to the single slot composition problem 
for a set of $k$ tasks $\mathcal{T}$ with metrics $\mathcal{D}_1, \ldots, \mathcal{D}_k$, 
if $\forall u \in U$, $\mathcal{S}$ assigns outputs  for each task $\{x_{u,1}, \ldots, x_{u,k}\} \in \{0,1\}^k$ such that
$\sum_{i \in [k]}x_{u,i} \leq 1$, 
and $\forall i \in [k]$, and $\forall$ $u,v \in U$,
$\mathcal{D}_i(u,v) \geq |\E[x_{u,i}]-\E[x_{v,i}]|$.
\end{definition}

The single slot composition problem captures the scenario in which an advertising platform may have a single slot to show an ad but need not show any ad. Imagine that this advertising system only has two types of ads: those for jobs and those for household goods. If a person is qualified for jobs and wants to purchase household goods, the system must pick at most one of the ads to show. In this scenario, it may be unlikely that the advertising system would choose to show no ads, but the problem specification does not require that any positive outcome is chosen. 

To solve the single-slot composition problem we must build a system which chooses at most one of the possible tasks so that fairness is preserved for each task across all elements in the universe. Clearly if classifiers for each task may \textit{independently} and fairly assign outputs, the system as a whole satisfies multiple task fairness.
However, most systems will require trade-offs between tasks. 
Consider a na\"{i}ve solution to the single-slot problem for ads: each advertiser chooses to bid on each person with some probability, and if both advertisers bid for the same person, the advertiser with the higher bid gets to show her ad. Formally, we define a \tiebreakingfunction~and Task-Competitive Composition:
\begin{definition}[Tie-breaking Function]\label{abs:def:tiebreaking} A (possibly randomized) \textit{\tiebreakingfunction} $\Tiebreak: U \times \{0,1\}^* \times \{0,1\}^k \rightarrow [k] \cup \{0\}$ takes as input an individual $w\in U$ and a $k-$bit string $x_w$ and outputs the index of a ``1" in $x_w$ if such an index exists and 0 otherwise. 
\end{definition}

\begin{definition}[Task-Competitive Composition]\label{abs:def:choicecomp}
Consider a set $\T$ of $k$ tasks, and a \tiebreakingfunction~as defined above.
Given a set $\mathcal{C}$ of classifiers for the set of tasks, define $y_w = \{y_{w,1}, \ldots, y_{w,k}\}$ where $y_{w,i}=C_i(w)$. The \competitivecomposition~of the set $\mathcal{C}$ is defined as
$y_{w}^* = \Tiebreak(w,y_w)$ 
for all $w\in U$.
\end{definition}

Definition~\ref{abs:def:choicecomp} yields a system~$S$ defined by
$S(w) = 0^k$ if $y_w = 0^k$ and $S(w) = e_{\Tiebreak(w,y_w)}$ (the $\Tiebreak(w,y_w)$ basis vector of dimension~$k$) if $y_w \not= 0^k$.  We evaluate its fairness by examining the Lipschitz requirement $|\Pr[y^*_u = i]-\Pr[y^*_v = i]| \le \mathcal{D}_i$ for all $u,v \in U$ and $i\in [k]$.

Task-competitive composition can reflect many scenarios other than advertising, which are discussed in greater detail in the full paper. Note that the \tiebreakingfunction~need not encode the same logic for all individuals and may be randomized.

\begin{theorem}\label{abs:theorem:choiceunfairness}
For any two tasks $T$ and $T'$ with nontrivial metrics $\mathcal{D}$ and $\mathcal{D'}$ respectively, there exists a set $\C$ of classifiers which are individually fair in isolation but when combined with \competitivecomposition~violate multiple task fairness for any \tiebreakingfunction.
\end{theorem}

\begin{proof} (Sketch) We sketch the proof for a simpler setting in which
the tie-breaking function strictly prefers task $T$, that is whenever the classifiers for $T$ an $T'$ both return 1, task~$T$ is chosen, and there exists a pair $u,v \in U$ such that $\D(u,v) \neq 0$ and $\D'(u,v)=0$\footnote{See Section \ref{section:multipletask} for a complete treatment of competitive composition.}.

Our strategy is to construct $C$ and $C'$ such that the distance between a pair of individuals is stretched for the `second' task.

Let $p_u$ denote the probability that $C$ assigns 1 to $u$, and analogously $p_v,p_u',p_v'$. 
The probabilities that $u$ and $v$ are assigned 1 for the task $T'$ are
$\Pr[\mathcal{S}(u)_{T'} = 1] = (1-p_u)p_u'$ and
$\Pr[\mathcal{S}(v)_{T'} = 1] = (1-p_v)p_v'$. 
The difference between them is 
\[\Pr[\mathcal{S}(u)_{T'} = 1] - \Pr[\mathcal{S}(v)_{T'} = 1] = (1-p_u)p_u' - (1-p_v)p_v'
= p_u'  - p_v' + p_vp_v' - p_up_u'\]
By assumption $\mathcal{D'}(u,v)=0$, so for any choice of $p_u'=p_v'>0$ and for any choice of $p_u \neq p_v$, this quantity is not zero, giving the desired contradiction.

\end{proof}
The intuition for unfairness in such a strictly ordered composition is that each task inflicts its preferences on subsequent tasks, and this intuition extends to more complicated tie-breaking functions and individuals with positive distances in both tasks.

Our intuition suggests that the situation in Theorem~\ref{abs:theorem:choiceunfairness} is not contrived and occurs often in practice, and moreover that small relaxations will not be sufficient to alleviate this problem, as the phenomenon has been observed empirically \cite{datta2015automated,lambrecht2016algorithmic,kuhn2012gender}. 
We include a small simulated example in Appendix \ref{section:empirical} to illustrate the potential magnitude and frequency of such fairness violations.

\subsection{Simple Fair Multiple-task Composition}
Fortunately, there is a general purpose mechanism for the single slot composition problem which requires no additional information in learning each classifier and no additional coordination between the classifiers.\footnote{See section \ref{section:positiveweights} for another mechanism which requires coordination between the classifiers.} The rough procedure for \randomizeandclassify (specified in detail in Section \ref{section:fairmultipletask} Algorithm \ref{alg:randomizethenclassify}) is to fix a fair classifier for each task, fix a probability distribution over the tasks, sample a task from the distribution, and then run the fair classifier for that task. 
\randomizeandclassify has several nice properties: it requires no coordination in the training of the classifiers, it preserves the ordering and relative distance of elements by each classifier, and it can be implemented by a platform or other third party, rather than requiring the explicit cooperation of all classifiers.
The primary downside of \randomizeandclassify is that it reduces allocation (the total number of positive classifications) for classifiers trained with the expectation of being run independently.

\section{Functional Composition}\label{section:abssametask}

In {\em Functional Composition}, the outputs of multiple classifiers are combined through logical operations to produce a single output for a single task. A significant consideration in functional composition is determining which outcomes are relevant for fairness and at which point(s) fairness should be measured.
For example, (possibly different) classifiers for admitting students to different colleges are composed to determine whether the student is accepted to at least one college.  In this case,  the function is ``OR," the classifiers are for the same task, and hence conform to the same metric, and this is the same metric one might use for defining fairness of the system as a whole.  Alternatively, the system may compose the classifier for admission with the classifier for determining financial aid.  In this case the function is ``AND," the classifiers are for different tasks, with different metrics, and we may use scholastic ability or some other appropriate output metric for evaluating overall fairness of the system.

\subsection{Same-task Functional Composition}
In this section, we consider the motivating example of college admissions. When secondary school students apply for college admission, they usually apply to more than one institution to increase their odds of admission to at least one college. 
Consider a universe of students $U$ applying to college in a particular year, each with intrinsic qualification $q_u \in [0,1]$, $\forall u \in U$. We define $\D(u,v)=|q_u-q_v|$ $\forall u,v \in U.$ 
$\mathcal{C}$ is the set of colleges and assume each college $C_i \in \mathcal{C}$ admits students fairly with respect to $\D$. The system of schools is considered OR-fair if the indicator variable $x_u$ which indicates whether or not student $u$ is admitted to at least one school satisfies individual fairness under this same metric. More formally, 

\begin{definition}[OR Fairness]\label{abs:def:orfairness}
Given a (universe, task) pair with metric $\mathcal{D}$, and a set of classifiers $\mathcal{C}$
we define the indicator
\[x_u = \begin{cases}1 \text{ if } \sum_{C_i \in \mathcal{C}}C_i(x) \geq 1\\ 0 \text{ otherwise}\end{cases}\]
which indicates whether at least one positive classification occurred.
Define $\dist{x}_u=\Pr[x_u=1] = 1-\prod_{C_i \in \mathcal{C}} (1-\Pr[C_i(u)=1])$. Then the composition of the set of classifiers $\mathcal{C}$
satisfies \textit{OR Fairness} if 
$\mathcal{D}(u,v) \geq d(\dist{x}_u,\dist{x}_v)$ for all $u,v \in U$.
\end{definition}

The OR Fairness setting matches well to tasks where individuals primarily \textit{benefit} from one positive classification for a particular task.\footnote{We may conversely define NOR Fairness to take $\neg x_u$, and this setting more naturally corresponds to cases where not being classified as positive is desirable.}
As mentioned above, examples of such tasks include gaining access to credit or a home loan, admission to university, access to qualified legal representation, access to employment, {\it etc}\footnote{\cite{DBLP:journals/corr/BowerKNSVV17} considers what boils down to AND-fairness for Equal Opportunity and presents an excellent collection of evocative example scenarios.}. Although in some cases more than one acceptance may have positive impact, for example a person with more than one job offer may use the second offer to negotiate a better salary, the core problem is (arguably) whether or not at least one job is acquired.

Returning to the example of college admissions, even with the strong assumption that each college fairly evaluates its applicants, there are still several potential sources of unfairness in the resulting system. In particular, if students apply to different numbers of colleges or colleges with different admission rates, we would expect that their probabilities of acceptance to at least one college will be different. The more subtle scenario from the perspective of composition is when students apply to the \textit{same} set of colleges.

Even in this restricted setting, it is still possible for a set of classifiers for the same task to violate OR fairness.
The key observation is that for elements with positive distance, the difference in their expectation of acceptance by at least one classifier does not diverge linearly in the number of classifiers included in the composition. 
As the number of classifiers increases, the probabilities of positive classification by at least one classifier for any pair eventually converge.
However, in practice, we expect students to apply to perhaps five or 10 colleges, so it is desirable to characterize when small systems are robust to such composition.

\begin{theorem}\label{abs:theorem:orsame}
For any (universe, task) pair with a non-trivial metric $\mathcal{D}$,
there exists a set of individually fair classifiers $\mathcal{C}$ which do not satisfy OR Fairness, even if each element in $U$ is classified by all $C_i \in \mathcal{C}$.
\end{theorem}

The proof of Theorem \ref{abs:theorem:orsame}  follows from a straightforward analysis of the difference in probability of at least one positive classification.\footnote{See Section \ref{section:sametask} for the complete proof.} 
The good news is that there exist non-trivial conditions for sets of small numbers of classifiers where OR Fairness is satisfied:  
\begin{lemma}\label{abs:lemma:sametaskgeneral}
Fix a set $\mathcal{C}$ of fair classifiers, and let $x_w$ for $w\in U$ be the indicator variable as defined in Theorem \ref{theorem:ordifferent}.
If $\E[x_w]\geq 1/2$ for all $w \in U$,  
then the set of classifiers $\mathcal{C} \cup \{C'\}$ satisfies OR fairness if $C'$ satisfies individual fairness under the same metric and $\Pr[C'(w)=1]\geq \frac{1}{2}$ for all $w \in U$.

\end{lemma}

This lemma is useful for determining that a system is free from \sametaskdivergence, as it is possible to reason about an ``OR of ORs''. 

Functional composition can also be used to reason about settings where classification procedures for different tasks are used to determine the outcome for a single task. For example, in order to attend a particular college, a student must be admitted \textit{and} receive sufficient financial aid to afford tuition and living expenses. Financial need and academic qualification clearly have different metrics, and in such settings, a significant challenge is to understand how the input metrics relate to the relevant output metric. Without careful reasoning about the interaction between these tasks, it is very easy to end up with systems which violate individual fairness, even if they are constructed from individually fair components. (See Section \ref{section:multipletaskfunctional} Theorem \ref{theorem:andunfairness} for more details.)

\section{Dependent Composition}\label{section:abscohort}
Thus far, we have restricted our attention to the mode of operation in which classifiers act on the entire universe of individuals at once and each individual's outcome is decided independently. In practice, however, this is an unlikely scenario, as classifiers may be acting as a selection mechanism for a fixed number of elements, may operate on elements in arbitrary order, or may operate on only a subset of the universe.
In this section, we consider the case in which the classification outcomes received by individuals are not independent. Slightly abusing the term ``composition,'' these problems  
can be viewed as a composition of the classifications of elements of the universe.
We roughly divide these topics into Cohort Selection problems, when a set of exactly $n$ individuals must be selected from the universe, and Universe Subset problems, when only a subset of the relevant universe for the task is under the influence of the classifier we wish to analyze or construct.
Within these two problems we  consider several relevant settings:

\textbf{Online versus offline:} 
Advertising decisions for online ads must be made immediately upon impression and employers must render employment decisions quickly or risk losing out on potential employees or taking too long to fill a position.

\textbf{Random versus adversarial ordering:} 
The order in which individuals apply for an open job may be influenced by their social connections with existing employees, which impacts how quickly they hear about the job opening.

\textbf{Known versus unknown subset or universe size:}  
An advertiser may know the average number of interested individuals who visit a website on a particular day, but be uncertain on any particular day of the exact number. 

\textbf{Constrained versus unconstrained selection:} in many settings there are arbitrary constraints placed on selection of individuals for a task which are unrelated to the qualification or metric for that task. For example, to cover operating costs, a college may need at least $n/2$ of the $n$ students in a class to be able to pay full tuition.

In dependent composition problems, it is important, when computing distances between distributions over outcomes, to pay careful attention to the source of randomness. Taking inspiration from the experiment setup found in many cryptographic definitions, we  formally define two problems, Universe Subset Classification and Cohort Selection, in Section \ref{section:cohort}. In particular, it is important to understand the randomness used to decide an ordering or a subset, as once an ordering or subset is fixed, reasoning about fairness is impossible, as a particular individual may be arbitrarily included or excluded.

\subsection{Basic Offline Cohort Selection}
First we consider the simplest version of the cohort selection problem: choosing a cohort of $n$ individuals from the universe $U$ when the entire universe is known and decisions are made offline. 
A simple solution is to choose a permutation of the elements in $U$ uniformly at random, and then apply a fair classifier $C$ until $n$ are selected or selecting the last few elements from the end of the list if $n$ have not yet been selected. With some careful bookkeeping, we show that this mechanism is individually fair for any individually fair input classifier. (See Section \ref{section:cohort} Algorithms \ref{alg:permutethenclassify} and \ref{alg:weightedsampling}.)

\subsection{More complicated settings}
In this extended abstract, we omit a full discussion of the more complicated dependent composition scenarios, but briefly summarize several settings to build intuition.

\begin{theorem}\label{thm:adversarialunknownlength}
If the ordering of the stream is adversarial, but $|U|$ is unknown, then there exists no solution to the online cohort selection problem.
\end{theorem}
The intuition for the proof follows from imagining that a fair classification process exists for an ordering of size $n$ and realizing that this precludes fair classification of a list of size $n+1$, as the classification procedure cannot distinguish between the two cases.

\paragraph{Constrained cohort selection}
Next we consider the problem of selecting a cohort with an external requirement that some fraction of the selected set is from a particular subgroup. That is, given a universe $U$, and $p \in [0,1]$, and a subset $A \subset U$, select a cohort of $n$ elements such that at least a $p$ fraction of the elements selected are in $A$. This problem captures situations in which external requirements cannot be ignored. For example, if a certain budget must be met, and only some members of the universe contribute to the budget, or if legally a certain fraction of people selected must meet some criterion (as in, demographic parity). In the full version, we characterize a broad range of settings where the constrained cohort selection problem cannot be solved fairly.

To build intuition, suppose the universe $U$ is partitioned into sets $A$ and $B$, where $n/2=|A| = |B|/5$. Suppose further that the populations have the same distribution on ability, so that the set $B$ is a ``blown up" version of $A$, meaning that for each element $u \in A$ there are 5 corresponding elements $V_u = \{v_{u,1},...,v_{u,5}\}$ such that $\mathcal{D}(u,v_{u,i})= 0$, $1 \le i \le 5$, $\forall u,u' \in A~ V_u \cap V_{u'} = \emptyset$, and $B = \cup_{u \in A} V_u$. 
Let $p = \frac{1}{2}$.  The constraint requires all of $A$ to be selected; that is, each element of $A$ has probability 1 of selection. In contrast, the average probability of selection for an element of $B$ is $\frac{1}{5}$. Therefore, there exists $v\in B$ with selection probability at most $1/5$.  Letting $u\in A$ such that $v \in V_u$, we have $\mathcal{D}(u,v)=0$ but the difference in probability of selection is at least $\frac{4}{5}$. 
We give a more complete characterization of the problem and impossibilities in Section \ref{section:constrainedcohort}.

\section{Extensions to Group Fairness}\label{section:absgroup}
In general, the results discussed above for composition of individual fairness extend to group fairness definitions; however, there are several issues and technicalities unique to group fairness definitions which we now discuss.

\paragraph{Technicalities.} 
Consider the following simple universe: for a particular $z \in \mathcal{Z}$, group $B$ has only elements with medium qualification $q_m$, group $A$ has half of its elements with low qualification $q_l$ and half with high qualification $q_h$. Choosing $p_h=1$, $p_m=.75$, and $p_l=.5$ satisfies Conditional Parity for a single application. 
However, for the OR of two applications, the the squares diverge ($.9375\neq.875$), violating conditional parity (see Figure \ref{abs:fig:bimodal}).

Note, however, that all of 
 the individuals with $\mathbf{z}=z$ have been drawn closer together under composition, and none have been pulled further apart. 

This simple observation implies that in some cases we may observe failures under composition for conditional parity, even when individual fairness is satisfied.
In order to satisfy Conditional Parity under OR-composition, the classifier could sacrifice accuracy by treating all individuals with $\mathbf{z}=z$ equally. However, this necessarily discards useful information about the individuals in $A$ to satisfy a technicality.

\paragraph{Subgroup Subtleties.} There are many cases where failing to satisfy conditional parity under \competitivecomposition~is clearly a violation of our intuitive notion of group fairness. However, conditional parity is not always a reliable test for fairness at the subgroup level under composition.
In general, we expect conditional parity based definitions of group fairness to detect unfairness in multiple task compositions reasonably well when there is an obvious interaction between protected groups and task qualification, as observed empirically in \cite{lambrecht2016algorithmic} and \cite{datta2015automated}. For example, let's return to our advertising example where home-goods advertisers have no protected set, but high-paying jobs have gender as a protected attribute. Under composition, home-goods out-bidding high-paying jobs ads for women will clearly violate the conditional parity condition for the job ads (see Figure \ref{abs:fig:group_relatedabs}).

However, 
suppose that, in response to gender disparity caused by task-competitive composition, classifiers iteratively adjust their bids to try to achieve Conditional Parity. This may cause them to \textit{learn themselves} into a state that satisfies Conditional Parity with respect to gender, but behaves poorly for a socially meaningful subgroup (see Figure \ref{abs:fig:iterativebidsupdated}.) 
For example, if home goods advertisers aggressively advertise to women who are new parents (because their life-time value ($\mathcal{Z}$) to the advertiser is the highest of all universe elements), then a competing advertiser for jobs, noticing that its usual strategy of recruiting all people with skill level  $\mathbf{z'}=z'$ equally is failing to reach enough women, bids more aggressively  on women. 
By bidding more aggressively, the advertiser increases the probability of showing ads to women (for example by outbidding low-value competition), but not to women who are bid for by the home goods advertiser (a high-value competitor), resulting in a high concentration of ads for women who are {\em not} mothers, while still failing to reach women who {\em are} mothers. Furthermore, the systematic exclusion of mothers from job advertisements can, over time, be even more problematic,  
as it may contribute to the stalling of careers.
In this case, the system discriminates against mothers without necessarily discriminating against fathers.

Although problematic (large) subgroup semantics are part of the motivation for \cite{kearns2017gerrymandering,hebert2017calibration} and exclusion of subgroups is not only a composition problem, the danger of composition is that the features describing this subset may be missing from the feature set of the jobs classifier, rendering the protections proposed in \cite{kearns2017gerrymandering} and \cite{hebert2017calibration} ineffective.
In particular, we expect sensitive attributes like parental status are unlikely to appear (or are illegal to collect) in employment-related training or testing datasets, but may be legitimately targeted by other competing advertisers.

\begin{figure}[h]
\begin{center}
\centerline{\includegraphics[width=.45\textwidth]{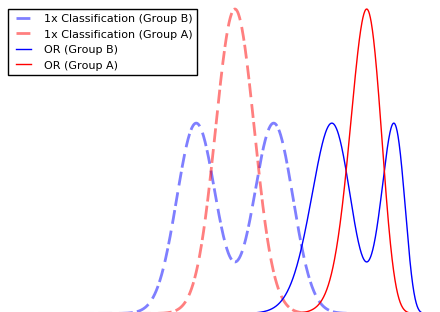}}
    \caption{An illustration of the shift in groups from a single classification to the OR of two applications of the same classifier. 
    Although the two groups originally had the same mean probability of positive classification, this breaks down under OR composition. } 
\label{abs:fig:bimodal}
\end{center}
\vskip -0.4in
\end{figure}

\begin{figure}
\begin{center}
\centerline{\includegraphics[width=\textwidth]{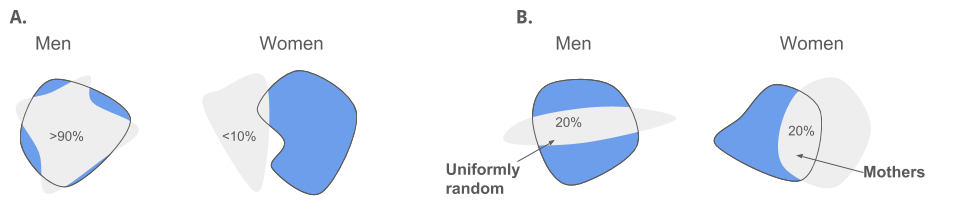}}
\caption{\textbf{A.} When the two tasks are related, one will `claim' a larger fraction of one gender than another, leading to a smaller fraction of men remaining for classification in the other task (shown in blue). Conditional parity will detect this unfairness. \textbf{B.} When the tasks are unrelated, one task may `claim' the same fraction of people in each group, but potentially select a socially meaningful subgroup, eg parents. Conditional parity will fail to detect this subgroup unfairness, unless subgroups, including any subgroups targeted by classifiers composed with, are explicitly accounted for.}
\label{abs:fig:group_relatedabs}
\end{center}
\vskip -0.4in
\end{figure}

\begin{figure*}
        \centering
        \begin{subfigure}[b]{0.475\textwidth}
            \centering
            \includegraphics[width=\textwidth]{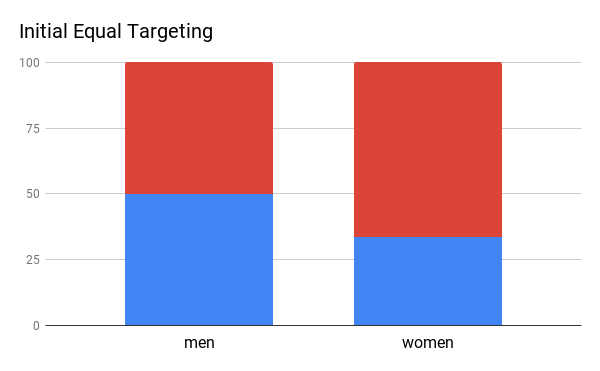}
            \caption{Initial equal targeting of qualified men and women results in violation of conditional parity, as there are unequal rates of ads shown (blue).}    
            \label{fig:initialsettingscoarse}
        \end{subfigure}
        \hfill
        \begin{subfigure}[b]{0.475\textwidth}  
            \centering 
            \includegraphics[width=\textwidth]{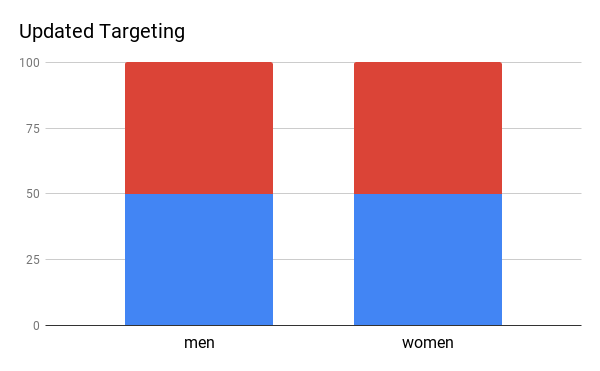}
            \caption
            {By increasing the targeting of women, the jobs advertiser ``fixes'' conditional parity at the coarse group level.}    
            \label{fig:updatedsettingscoarse}
        \end{subfigure}
        \vskip\baselineskip
        \begin{subfigure}[b]{0.475\textwidth}   
            \centering 
            \includegraphics[width=\textwidth]{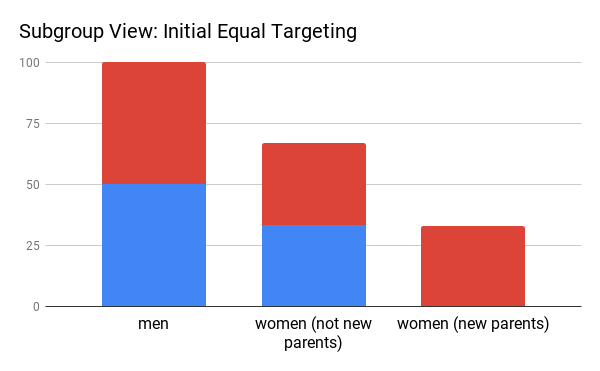}
            \caption{At the subgroup level, it's clear that the lack of conditional parity is due to ``losing'' all of the new parent women to the home-goods advertiser.}    
            \label{abs:fig:initialsettingssub}
        \end{subfigure}
        \quad
        \begin{subfigure}[b]{0.475\textwidth}   
            \centering 
            \includegraphics[width=\textwidth]{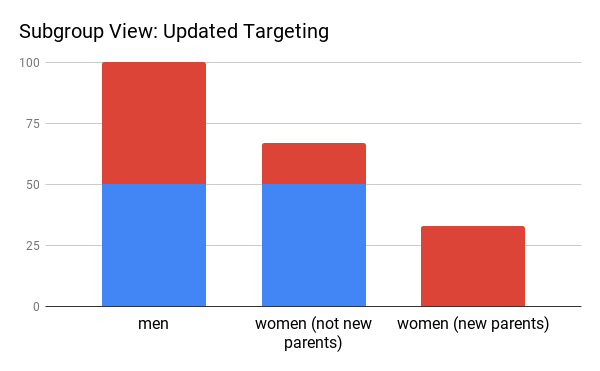}
            \caption{New targeting strategy increases ads shown to non new-parent women, but continues to exclude new parent women.}   
            \label{abs:fig:updatedsettingscoarse}
        \end{subfigure}
        \caption
        {Home-goods advertisers aggressively target mothers, out-bidding the jobs advertiser. When the jobs advertiser bids more aggressively on `women' (b) the overall rate of ads shown to `women' increases, but mothers may still be excluded (d), so $\Pr[\text{ad }| \text{qualified, woman}]> \Pr[\text{ad }| \text{ qualified, mother}]$.} 
        \label{abs:fig:iterativebidsupdated}
    \end{figure*}
    
\clearpage
\printbibliography
\appendix
\section{Full Paper}
\subsection{Introduction}
As automated decision-making extends its reach ever more deeply into our lives, there is increasing concern that such decisions be fair. The rigorous theoretical study of fairness in algorithmic classification was initiated by Dwork \textit{et al} in \cite{dwork2012fairness} and subsequent works investigating alternative definitions, fair representations, and impossibility results have proliferated in the machine learning, economics and theoretical computer science literatures.\footnote{See also \cite{pedreshi2008discrimination} \cite{kamiran2009classifying} and \cite{kamishima2011fairness}, which predate \cite{dwork2012fairness} and are motivated by similar concerns.}  The notions of fairness broadly divide into {\em individual fairness}, requiring that individuals who are similar with respect to a given classification task (as measured by a task-specific similarity metric) have similar probability distributions on classification outcomes; and {\em group fairness}, which requires that different demographic groups experience the same treatment in some average sense.

In a bit more detail, a {\em classification task} is the problem of mapping {\em individuals} to {\em outcomes}; for example, a decision task may map individuals to outcomes in $\{0,1\}$.  A classifier is a possibly randomized algorithm solving a classification task.  A running example throughout this work is online advertising.  In this case a task might be the problem of deciding whether or not to show a given job advertisement to an individual, and we may have an advertising system in which ads are shown repeatedly (or not), and many different advertisers, say, for a job, a grocery delivery service, and various items of clothing, may be competing in an auction for the attention of individual users.  In this latter case we have multiple competing advertising tasks.

In this work we initiate the study {\em fairness under composition}: what are the fairness properties of systems built from classifiers that are fair in isolation?  Under what circumstances can we ensure fairness, and how can we do so?
We identify and examine several types of composition and draw conclusions about auditing systems for fairness, constructing fair systems and definitions of fairness for systems.  In the remainder of this section we summarize our results and discuss related work.  

\paragraph{Task-Competitive Compositions} (Section \ref{section:multipletask}). We first consider the problem of two or more tasks competing for individuals, motivated by the online advertising setting described above. 
We prove that two advertisers for different tasks, each behaving fairly (when considered independently), will not necessarily produce fair outcomes when they compete. Intuitively (and empirically observed by \cite{lambrecht2016algorithmic}), the attention of individuals similarly qualified for a job may effectively have different costs due to these individuals' respective desirability for other advertising tasks, like household goods purchases. That is, individuals claimed by the household goods advertiser will not see the jobs ad, regardless of their qualification.  These results are not specific to an auction setting and are robust to choice of ``tie-breaking'' functions that select among multiple competing tasks (advertisers). 
Nonetheless, we give a simple mechanism, \randomizeandclassifynospace, that solves the fair task-competitive classification problem using classifiers for the competing tasks each of which is fair in isolation, in a black-box fashion and without modification. In the full paper in Section \ref{section:positiveweights} we give a second technique for modifying the fair classifier of the lower bidder (loser of the tie-breaking function) in order to achieve fairness. 

\paragraph{Functional Compositions} (Section \ref{section:sametask}).  When can we build fair classifiers by computing on values that were fairly obtained?  Here we must understand what is the salient outcome of the computation. For example, when reasoning about whether the college admissions system is fair, the salient outcome may be whether a student is accepted to at least one college, and not whether the student is accepted to a specific college\footnote{In this simple example, we assume that all colleges are equally desirable, but it is not difficult to extend the logic to different sets of comparable colleges.}. Even if each college uses a fair classifier, the question is whether the ``OR'' of the colleges decisions is fair.  Furthermore, an acceptance to college may not be meaningful without sufficient accompanying financial aid. Thus in practice, we must reason about the OR of ANDs of acceptance and financial aid across many colleges. We show that although in general there are no guarantees on the fairness of functional compositions of fair components, there are some cases where fairness in ORs can be satisfied. Such reasoning can be used in many applications where long-term and short-term measures of fairness must be balanced. In the case of feedback loops, where prior positive outcomes can improve the chances of future positive outcomes, functional composition provides a valuable tool for determining at which point(s) fairness must be maintained and determining whether the existing set of decision procedures will adhere to these requirements.

\paragraph{Dependent Compositions} (Section \ref{section:cohort}). 
There are many settings in which each individual's classifications are dependent on the classifications of others. For example, if a company is interviewing a set of job candidates in a particular order, accepting a candidate near the beginning of the list precludes any subsequent candidates from even being considered. Even if each candidate is considered fairly in isolation, dependence between candidates can result in highly unfair outcomes. For example, individuals who are socially connected to the company through friends or family are likely to hear about job openings first and thus be considered for a position before candidates without connections. We show that selecting a cohort of people -- online or offline -- requires care to prevent dependencies from impacting an independently fair selection mechanism. We address this in the offline case with two randomized constructions, \permutethenclassify and \weightedsamplingnospace.  These algorithms can be applied in the online case, even under adversarial ordering, provided the size of the universe of individuals is known; when this is not known there is no solution.

\paragraph{Nuances of group-based definitions}(Section \ref{section:group}). 
Many fairness definitions in the literature seek to provide fairness guarantees based on group-level statistical properties. For example, Equal Opportunity~\cite{hardt2016equality} requires that, conditioned on qualification, the probability of a positive outcome is independent of protected attributes such as race or gender. Group Fairness definitions have practical appeal in that they are possible to measure and enforce empirically without reference to a task-specific similarity metric. We extend our results to group fairness definitions and we also show that these definitions do not always yield consistent signals under composition. 
In particular, we show that the intersectional subgroup concerns (which motivate \cite{kearns2017gerrymandering,hebert2017calibration}) are exacerbated by composition. For example, an employer who uses group fairness definitions to ensure parity with respect to race and gender may fail to identify that ``parents'' of particular race and gender combinations are not treated fairly. Task-competitive composition exacerbates this problem, as the employers may be prohibited from even collecting parental status information, but their hiring processes may be composed with other systems which legitimately differentiate based on parental status. 

Finally, we also show how na\"{i}ve strategies to mitigate these issues in composition may result in learning a nominally fair solution that is clearly discriminating against a socially meaningful subgroup not officially called out as ``protected,'' from which we conclude that understanding the behavior of fairness definitions under composition is critical for choosing which definition is meaningful in a given setting.

\paragraph{Implications of Our Results.} 
Our composition results have several practical implications. First, testing individual components without understanding of the whole system will be insufficient to draw either positive or negative conclusions about the fairness of the system.  
Second, composition properties are an important point of evaluation for any definitions of fairness or fairness requirements imposed by law or otherwise. Failing to take composition into account when specifying a group-based fairness definition may result in a meaningless signal under composition, or worse may lead to ingraining poor outcomes for certain subgroups while still nominally satisfying fairness requirements.
Third, understanding of the salient outcomes on which to measure and enforce fairness is critical to building meaningfully fair systems. 
Finally, we conclude that there is significant potential for improvement in the mechanisms proposed for fair composition and many settings in which new mechanisms could be proposed.

\subsection{Related Work}\label{section:relatedwork}
Fairness retained under post-processing in the single-task one-shot setting is central in \cite{zemel2013learning,madras2018learning,dwork2012fairness}. 
The definition of individual fairness we build upon in this work was introduced by Dwork \textit{et al} in \cite{dwork2012fairness}. 
Learning with oracle access to the fairness metric is considered by \cite{gillen2018online,kim2018fairness}. 
A number of group-based fairness definitions have been proposed, and Ritov \textit{et al} provide a combined discussion of the parity-based definitions in \cite{ritov2017conditional}. 
In particular, their work includes discussion of Hardt \textit{et al}'s Equality of Opportunity and Equal Odds definitions and Kilbertus \textit{et al}'s Counterfactual Fairness \cite{hardt2016equality,kilbertus2017avoiding}. 
Kleinberg \textit{et al} and Chouldechova independently described several impossibility results related to simultaneously satisfying multiple group fairness conditions in single classification settings \cite{DBLP:journals/corr/KleinbergMR16},\cite{chouldechova2017fair}.

Two concurrent lines of work aiming to bridge the gap between individual and group consider ensuring fairness properties for large numbers of large groups and their (sufficiently large) intersections~\cite{kearns2017gerrymandering,hebert2017calibration}.
While these works consider the one-shot, single-task setting,
we will see that group intersection properties are of particular importance under composition. 
Two subsequent works in this general vein explore approximating individual fairness with the help of an oracle that knows the task-specific metric \cite{kim2018fairness,gillen2018online}.  
Several works also consider how feedback loops can influence fair classification \cite{DBLP:journals/corr/HuC17,liu2018delayed}.

There are several empirical or observational studies which document the effects of multiple task composition. For example, Lambrecht and Tucker study how intended gender-neutral advertising can result in uneven delivery due to high demand for the attention of certain demographics \cite{lambrecht2016algorithmic}. Datta \textit{et al} also document differences in advertising based on gender, although they are agnostic as to whether the cause is due to multiple task composition or discriminatory behavior on the part of the advertisers or platform \cite{datta2015automated}.
Whether it is truly ``fair'' that, say, home goods advertisers bid more highly for the attention of women than for the attention of men, may be debatable, although there are clearly instances in which differential targeting is justified, such as maternity clothes.  This {\em actuarial fairness} is the industry practice, so we pose a number of examples in this framework and analyze the implications of composition.


\subsection{Preliminary Definitions and Assumptions}\label{section:preliminaries}
\subsubsection{Shared Terminology}
We refer to classifiers as being ``fair in isolation'' to indicate that with no composition, the classifier satisfies a particular fairness definition. In such cases  expectation and probability are taken over the randomness of the classification procedure and, for group fairness, selection of elements from the universe. We denote the universe of individuals relevant for a task as $U$, and we generally use $u,v,w \in U$ to refer to universe elements. 
We generally consider binary classifiers in this work, and use $p_w$ to denote the probability of assigning the positive outcome (or simply 1) to the element $w$ for a particular classifier. We generally write $C: U \times \{0,1\}^* \rightarrow \{0,1\}$, where $ \{0,1\}^*$ represents the random bits of the classifier. This allows us to comfortably express the probability of positive classification $\E_r[C(u)]$ as well as the output of the classifier under particular randomness $C(u,r)$. In this notation, $p_u = \E_r[C(u)]$. When considering the distribution on outputs of a classifier $C$, we use $\dist{C}:U \rightarrow \Delta(\{0,1\})$. 

When two or more classifiers or tasks are compared, we either use a subscript $_i$ to indicate the $i^{th}$ classifier or task, or a prime $(')$ to indicate the second classifier or task. For example $\{C,C'\}$, $\{C_i | i \in [k]\}$, $\{T,T'\}$, $\{T_i| i \in [k]\}$.

\subsubsection{Individual Fairness}
Throughout this work, our primary focus is individual fairness, proposed by Dwork \textit{et al} in \cite{dwork2012fairness}.

\begin{definition}[Individual Fairness  \cite{dwork2012fairness}]\label{def:individualfairness}Let $d:\Delta(O) \times \Delta(O)\rightarrow [0,1]$ denote the total variation distance on distributions over $O$. Given a universe of individuals $U$, and a metric $\D$ for a classification task $T$ with outcome set $O$, a randomized classifier $C:U\times \{0,1\}^* \rightarrow O$, such that $\dist{C}: U \rightarrow \Delta(O)$, and a distance measure $d:\Delta(O) \times \Delta(O) \rightarrow \R$, $C$ is \textit{individually fair} 
if and only if for all $u,v \in U$, $\D(u,v) \geq d(\dist{C}(u),\dist{C}(v))$.\footnote{\cite{dwork2012fairness} also considered other measures such as max divergence.}
\end{definition}

Individual fairness is a very strong definition, as it requires that each individual's constraints be accounted for.
Although the guarantees of individual fairness are desirable, the main practical barrier for adoption in practice is the need for a task-specific similarity metric. For the purposes of our discussion of composition, we defer questions of how to find such a metric, and instead assume that we have access to a complete metric for each task and universe under consideration.
To keep our analyses intuitively simple, we will use total variation distance for $d$ unless otherwise specified. In the case $|O|=2$, this allows us to use simple differences in probability to determine distances between individuals in the outcome space. That is, $d(\dist{C}(u),\dist{C}(v)) = |\E_r[C(u)] - \E_r[C(v)]|=|p_u-p_v|$.

\paragraph{Trivial Metrics and Universes}
A trivial metric is a metric in which all individuals are either equal, or maximally distant. A trivial metric may still contain significant information regarding equivalent pairs, so there may still be some settings where a trivial metric can still provide meaningful fairness guarantees. 

\begin{definition}[Trivial Metric]\label{defn:trivialmetric}
A metric $\mathcal{D}$ is considered \textit{trivial} if for all $u,v\in U$ $\mathcal{D}(u,v) \in \{0,1\}$.
\end{definition}

Fairness with respect to a trivial metric requires that we 
treat all elements equally or satisfy something akin to perfect prediction - that is, we can perfectly separate the universe into two classes for prediction. In practice, such metrics are unlikely, and as such we primarily reason about settings with non-trivial metrics.

\paragraph{Construction of Individually Fair Classifiers}
In \cite{dwork2012fairness}, fair classifiers are constructed by solving a linear program to minimize the loss of an objective function subject to the distance constraints of the metric. There is always a solution to such a linear program, although the loss may be high: treat all elements of the universe equally. 
Throughout this work, we will frequently use the fact that individually fair classifiers with particular distance properties exist in proofs. We therefore include the following lemma and corollaries which allow us to construct classifiers with positive distance between elements, and reason about the maximum distance between a pair of elements for a fair classifier.

\begin{lemma}\label{lemma:fairadditions}
Let $V$ be a (possibly empty) subset of $U$. If there exists a classifier $C: V \times \{0,1\}^* \rightarrow \{0,1\}$ such that $\D(u,v) \geq d(\dist{C}(u),\dist{C}(v))$ for all $u,v \in V$, then for 
any $x \in U \backslash V$ 
there exists classifier $C': V \cup \{x\} \times \{0,1\}^* \rightarrow \{0,1\}$ such that $\D(u,v) \geq d(\dist{C}(u),\dist{C}(v))$  for all $u,v \in U$, which has identical behavior to $C$ on $V$.
\end{lemma}
\begin{proof}
For $V = \emptyset$, any value $p_x$ suffices to fairly classify $x$.
For $|V|=1$, choosing any $p_x$ such that $|p_v - p_x| \leq \D(v,x)$ for $v \in V$ suffices.

For $|V|\geq 2$, apply the procedure outlined in Algorithm \ref{alg:fairaddmetric} 
taking $p_t$ to be the probability of positive classification of $x$'s nearest neighbor in $V$ under $C$. As usual, we take $p_w$ to be the probability that $C$ positively classifies element $w$.

\begin{algorithm}[tb]
  \caption{$\fairadd(\D,V,p_t,C,x)$}
  \label{alg:fairaddmetric}
\begin{algorithmic}
  \STATE {\bfseries Input:} metric $\D$ for universe $U$, a subset $V \subset U$, target probability $p_t$, an individually fair classifier  $C: V \times \{0,1\}^* \rightarrow \{0,1\}$, a target element $x \in U \backslash V$ to be added to $C$.
  \STATE Initialize $L \leftarrow V$
  \STATE $\hat{p}_x \leftarrow p_t$
  \FOR{$l \in L$}
  \STATE $dist \leftarrow \D(l,x)$ 
  \IF{$dist < p_l - \hat{p}_x$}
  \STATE $\hat{p}_x \leftarrow p_l - dist$
  \ELSIF{$dist < \hat{p}_x - p_l$}
  \STATE $\hat{p}_x \leftarrow p_l + dist$
  \ENDIF
  \ENDFOR
  \STATE \textbf{return } $\hat{p}_x$
\end{algorithmic}
\end{algorithm}

Notice that Algorithm \ref{alg:fairaddmetric} only modifies $\hat{p}_x$, and that  $\hat{p}_x$ is only changed if a distance constraint is violated. Thus it is sufficient to confirm that on each modification to $\hat{p}_x$, no distance constraints between $x$ and elements in the opposite direction of the move are violated. 

Without loss of generality, assume that $\hat{p}_x$ is decreased to move within an acceptable distance of $u$, that is $\hat{p}_x \geq p_u$. It is sufficient to show that for all $v$ such that $p_v > \hat{p}_x$ that no distances are violated. Consider any such $v$. By construction $\hat{p}_x-p_u=\D(u,x)$, and $p_v - p_u \leq \D(u,v)$. From triangle inequality, we also have that $\D(u,v) \leq \D(u,x) + \D(x,v)$. Substituting, and using that $p_v \geq \hat{p}_x \geq p_u$:
\[\D(u,v) \leq \D(u,x) + \D(x,v) \]
\[\D(u,v) - \D(u,x) \leq \D(x,v) \]
\[\D(u,v) - (\hat{p}_x-p_u) \leq \D(x,v) \]
\[(p_v - p_u)  - (\hat{p}_x-p_u)\leq \D(u,v) - (\hat{p}_x-p_u) \leq \D(x,v) \]
\[p_v -\hat{p}_x \leq \D(x,v) \]

Thus the fairness constraint for $x$ and $v$ is satisfied, and $C'$ is an individually fair classifier for $V \cup \{x\}$. 

\end{proof}

Lemma \ref{lemma:fairadditions} allows us to build up a fair classifier in time $O(|U|^2)$ from scratch, or to add to an existing fair classifier for a subset. We state several useful corollaries:
\begin{corollary}\label{cor:fullclassifier} Given a subset $V \subset U$ and a classifier $C: V \times \{0,1\}^* \rightarrow \{0,1\}$ such that $\D(u,v) \geq d(\dist{C}(u),\dist{C}(v))$ for all $u,v \in V$, there exists an individually fair classifier $C':U \times \{0,1\}^*\rightarrow \{0,1\}$ which is individually fair for all elements $u,v \in U$ and has identical behavior to $C$ on $V$.
\end{corollary}
Corollary \ref{cor:fullclassifier} follows immediately from  applying Algorithm \ref{alg:fairaddmetric} to each element of $U \backslash V$ in arbitrary order.

\begin{corollary}\label{cor:maxdist} Given a metric $\D$, for any pair $u,v \in U$, there exists an individually fair classifier $C: U \times \{0,1\}^* \rightarrow \{0,1\}$ such that $d(\dist{C}(u),\dist{C}(v)) = \D(u,v)$.
\end{corollary}
Corollary \ref{cor:maxdist} follows simply from starting from the classifier which is fair only for a particular pair and places them at their maximum distance under $\D$ and then repeatedly applying Algorithm \ref{alg:fairaddmetric} to the remaining elements of $U$.  
From a distance preservation perspective, this is important; if there is a particular `axis' within the metric where distance preservation is most important, then maximizing the distance between the extremes of that axis can be very helpful for preserving the most relevant distances.
\begin{corollary}\label{cor:ratio} Given a metric $\D$ and $\alpha \in \R^+$, for any pair $u,v \in U$, there exists an individually fair classifier $C: U \times \{0,1\}^* \rightarrow \{0,1\}$ such that $p_u/p_v = \alpha$, where $p_u=\E[C(u)]$ and likewise $p_v=\E[C(v)]$.
\end{corollary}
Corollary \ref{cor:ratio} follows from choosing $p_u/p_v=\alpha$ without regard for the difference between $p_u$ and $p_v$, and then adjusting. Take $\beta|p_v-p_u| =\D(u,v) $, and choose $\hat{p}_u=\beta p_u$ and $\hat{p}_v=\beta p_v$ so that $|\beta p_v - \beta p_u| = \beta |p_v - p_u| \leq \D(u,v)$, but the ratio $\frac{\beta p_u}{\beta p_v}=\frac{p_u}{p_v}=\alpha$ remains unchanged.

\subsubsection{Group Fairness}\label{section:groupdefs}
In Section \ref{section:group}, we will expand our results to notions of group fairness. 
The motivations for group fairness are two-fold.  Proportional representation is often desirable in its own right; alternatively, the {\em absence} of proportional allocation of goods can signal discrimination in the allocation process, typically against historically mistreated or under-represented groups. Thus, group fairness is often framed in terms of {\em protected attributes} $\A$, such as sex, race, or socio-economic status, while allowing for differing treatment based on a set of \textit{qualifications} $\mathcal{Z}$, such as, in the case of advertising, the willingness to buy an item.
Conditional Parity, a general framework proposed in \cite{ritov2017conditional} for discussing these definitions, conveniently captures many of the popular group fairness definitions popular in the literature including Equal Odds and Equal Opportunity \cite{hardt2016equality}, and Counterfactual Fairness \cite{kusner2017counterfactual}. 
\begin{definition}[Conditional Parity \cite{ritov2017conditional}]\label{def:conditionalparity} 
A random variable $\mathbf{x}$ satisfies parity with respect to $\mathbf{a}$ conditioned on $\mathbf{z}=z$ if the distribution of $\mathbf{x}$ $|$ $(\mathbf{a},\{\mathbf{z}=z\})$ is constant in $\mathbf{a}$:\\
$\Pr[\mathbf{x}=x \text{ }|\text{ }( \mathbf{a}=a,\mathbf{z}=z)] = \Pr[\mathbf{x}=x \text{ }|\text{ }( \mathbf{a}=a',\mathbf{z}=z)]$
for any $a,a' \in \mathcal{A}$.  Similarly, $\mathbf{x}$ satisfies parity with respect to $\mathbf{a}$ conditioned on $\mathbf{z}$ (without specifying a value of $\mathbf{z}$) if it satisfies parity with respect to $\mathbf{a}$ conditioned on $\mathbf{z}=z$ for all $z \in \mathcal{Z}$. All probabilities are over the randomness of the prediction procedure and the selection of elements from the universe.
\end{definition}
The definition captures the intuition that, conditioned on qualification, the setting of protected attributes should not {\em on average} impact the classification result. Note that this is not a guarantee about treatment at the individual level; it speaks only to group-level statistical properties in expectation.  In contrast, Individual Fairness makes strict requirements on the outcomes for each pair of individuals.

A weakness of group fairness definitions, addressed by Individual Fairness, is the problem of subgroup unfairness: a classifier that satisfies Conditional Parity with respect to race and gender independently may fail to satisfy Conditional Parity with respect to the \textit{conjunction} of race and gender.
Furthermore, the protected attributes $(\A)$ may not be sufficiently rich to describe every ``socially meaningful" group one might wish to protect from discrimination. For example, preventing discrimination against women is insufficient if it allows discrimination against women who are mothers, or who dress in a particular style.
To address this, two concurrent lines of work consider fairness for collections of large, possibly intersecting sets \cite{kearns2017gerrymandering,hebert2017calibration}. 
As we will see, composition exacerbates this problem uniquely for group fairness definitions (but not for Individual Fairness).

\subsubsection{Differential Privacy}
Dwork \textit{et al} noted the similarity of individual fairness to Differential Privacy \cite{dwork2012fairness}.
\begin{definition}[Pure Differential Privacy \cite{dwork2011differential}]\label{defn:pureDP}
A mechanism $\mathcal{M}$ is said to be $\varepsilon$-differentially private if for all  databases $x$ and $x'$ differing in a single element and for all $Z$ in the output space of $\mathcal{M}$:
\[\Pr[\mathcal{M}(x) \in Z] \leq e^{\varepsilon}\Pr[\mathcal{M}(x') \in Z]\]
\end{definition}

Loosely speaking, differential privacy requires that the output of the mechanism cannot depend too much on any one particular entry in the database. In this work we are primarily concerned with two properties of differential privacy:\footnote{For a more complete introduction to Differential Privacy, see \cite{dwork2012fairness,dwork2014algorithmic} and  \href{https://simons.berkeley.edu/talks/differential-privacy-fundamentals-forefront}{Cynthia Dwork's Simons Tutorial} and  \href{https://simons.berkeley.edu/talks/katrina-ligett-2013-12-11}{Katrina Ligett's  Simons Tutorial}.} 

 First, differential privacy is preserved under arbitrary post-processing. That is, any output from a mechanism which satisfies differential privacy can be arbitrarily post-processed, and privacy is still preserved. For example, the mechanism may output reals, which are subsequently rounded to integers, and privacy is not harmed.  The analogy for fairness would be for the case in which individuals are first labeled by a classifier (the case of an employment platform they may be labeled as programmers of high, medium, or low skill), and subsequent actions (invitation to interview for a programming job) taken depend only on the classification label.  The intuition, introduced in~\cite{dwork2012fairness} is that if the initial classification is fair, in that similarly qualified programmers have similar probabilities of being labeled highly skilled, then the system for inviting candidates to apply for programming jobs will be fair.
 
 Differential privacy composes nicely even \textit{without coordination} between analysts or databases. Without any coordination between analysts or databases, $\varepsilon_i-$differentially private mechanisms, adaptively chosen satisfy $\sum_{i}\varepsilon_i-$Differential Privacy. 
    The important takeaway is that
    Differential Privacy never suffers from catastrophic privacy loss under small degrees of composition.
Nonetheless, the cumulative privacy loss bounds can be tight, meaning that the log of the ratios of the probabilities of a sequence of output events can be as large as $\sum_i\varepsilon_i$.  We will return to this point when we discuss {\em functional composition} (Section~\ref{section:sametask}).

Although Differential Privacy can be useful for constructing fair classifiers in the one-shot setting (\cite{dwork2012fairness} Theorem 5.2), the composition guarantees have very different semantics. 

\subsection{Functional Composition}\label{section:sametask}

In {\em Functional Composition}, multiple classifiers are combined through logical functions to produce a single output for a single task. 
For example, (possibly different) classifiers for admitting students to different colleges are composed to determine whether the student is accepted to at least one college.  In this case,  the function is ``OR," the classifiers are for the same task, and hence conform to the same metric, and this is the same metric one might use for defining fairness of the system as a whole.  Alternatively, the system may compose the classifier for admission with the classifier for determining financial aid.  In this case the function is ``AND," the classifiers are for different tasks, with different metrics, and we may use scholastic ability or some other appropriate output metric for evaluating overall fairness of the system.

\subsubsection{Same-task Functional Composition}
In same-task composition, the same classification task is  repeated, either by the same entity, or by different entities. The outputs are then considered together to understand the fairness properties of the system. 

We begin by taking a particular composition type, OR. Relevant settings for this problem include a student applying to several colleges or a home-buyer applying to multiple banks for a loan. In such systems the important outcome is whether an individual achieves at least one positive classification. In this definition, we capture the case in which there is only one metric for the classifiers in the composition and that metric is the same as the metric for the final output. In later definitions, we will be more agnostic as to the number and type of metrics.
\begin{definition}[OR Fairness]\label{def:orfairness}
Given a (universe, task) pair with metric $\mathcal{D}$, and a set of classifiers $\mathcal{C}$
we define the indicator
\[x_u = \begin{cases}1 \text{ if } \sum_{C_i \in \mathcal{C}}C_i(x) \geq 1\\ 0 \text{ otherwise}\end{cases}\]
which indicates whether at least one positive classification occurred.

Define $\dist{x}_u=\Pr[x_u=1] = 1-\prod_{C_i \in \mathcal{C}} (1-\Pr[C_i(u)=1])$. Then the composition of the set of classifiers $\mathcal{C}$
satisfies \textit{OR Fairness} if 
$\mathcal{D}(u,v) \geq d(\dist{x}_u,\dist{x}_v)$ for all $u,v \in U$.
\end{definition}

The OR Fairness setting matches well to tasks where individuals primarily \textit{benefit} from one positive classification. We may conversely define NOR Fairness to take $\neg x_u$, and this setting more naturally corresponds to cases where not being classified as positive is desirable. In some settings, we may even wish to satisfy both OR and NOR fairness simultaneously. 
As mentioned above, examples of such tasks include gaining access to credit or a home loan, admission to university, access to qualified legal representation, access to employment, {\it etc}\footnote{\cite{DBLP:journals/corr/BowerKNSVV17} considers what boils down to AND-fairness for Equal Opportunity~\cite{hardt2016equality} and presents an excellent collection of evocative example scenarios.}. Although in some cases more than one acceptance may have positive impact, for example a person with more than one job offer may use the second offer to negotiate a better salary, the core problem is whether or not at least one job is acquired.
Similarly, an advertisement for a new job posting may have slightly more impact on a person who is exposed to it twice rather than once, but has incomparably less impact on a person who never sees the ad. 
When the appropriate ``dosage'' of positive classifications is known, for example, if it is known that $k$ job offers are needed for effective salary negotiation, Definition \ref{def:orfairness} can be adjusted to a threshold function requiring that at least $k$ classifiers respond positively. 

In this section, we consider the motivating example of college admissions. When secondary school students apply for college admission, they usually apply to more than one institution to increase their odds of admission to at least one college and to increase their options regarding the type or location of school to attend. 

Consider a universe of students $U$ applying to college in a particular year, each with intrinsic qualification $q_u \in [0,1]$, $\forall u \in U$. We define $\D(u,v)=|q_u-q_v|$ $\forall u,v \in U.$
$\mathcal{C}$ is the set of colleges and assume each college $C_i \in \mathcal{C}$ admits students fairly with respect to $\D$. The system of schools is considered OR-fair if the indicator variable $x_u$ which indicates whether or not student $u$ is admitted to at least one school satisfies individual fairness under this same metric. Even with the strong assumption that each college fairly evaluates its applicants, there are still several potential sources of unfairness in the resulting system.

\paragraph{Differing Degrees of Composition}
Although it is a bit obvious, the first source of unfairness we investigate is when a different number of classifiers in the set are applied to different elements of the universe.

Students who are able to apply to more colleges (due to being able to afford the application fees or being the recipient of better college counseling or having more time to spend on applications) improve their chances of admission to college over those who do not in all but the most contrived cases.\footnote{Systems where all students are admitted with probability 0 or probability 1 are such contrived cases; they map well to the concept of ``Perfect Prediction,'' and are not considered.} Given any $u,v\in U$ such that $q_u=q_v \in [0,1]$, if $u$ and $v$ apply to different numbers of schools, it cannot in general be the case that $\E[x_u]=\E[x_v]$, if all schools admit both $u$ and $v$ with non-trivial probability, violating individual fairness.

\begin{theorem}\label{theorem:ordifferent} For any (universe, task) pair with a non-trivial metric $\D$, there exists a set of individually fair classifiers $\mathcal{C}$ that do not satisfy OR Fairness if each element may be classified by different sets of classifiers with different cardinalities. 
\end{theorem}
\begin{proof}
As each element $w\in U$ may be classified by  sets of classifiers of different cardinality, we denote $\mathcal{C}^w\subseteq \mathcal{C}$ as the set of classifiers which act on $w$.

Consider the set of randomized classifiers $\mathcal{C}$ where all classifiers are identical and assign outcome 1 to elements $u,v \in U$ with probabilities $p_u$ and $p_v$, respectively. Without loss of generality, assume $p_u\leq p_v$, $p_v-p_u=\D(u,v)$ and $p_v>0$ (such a $C$ exists per Lemma \ref{lemma:fairadditions}). If $u$ is classified once (that is, $|\mathcal{C}^u|=1$), and $v$ is classified twice ($|\C^v|=2$, then \[|\E[x_v]-\E[x_u]| = |(1-(1-p_v)^2)-p_u|\]\[ = |p_v-p_v^2+ (p_v-p_u)| \geq \D(u,v)\] which violates individual fairness and completes the proof.
\end{proof}

Intuitively, if equally qualified students (or even nearly equally qualified students) apply to different numbers of schools or different types of schools in equal numbers, then their probabilities of acceptance to at least one school will diverge.

\paragraph{Equal Degrees of Composition}
Assuming that the system requires that all elements be classified by the same number and even the same set of classifiers, it is still possible for a set of classifiers for the same task to violate OR fairness.
The key observation is that for elements with positive distance, the difference in their expectation of acceptance by at least one classifier does not diverge linearly in the number of classifiers included in the composition. 
For example, consider $u$ and $v$ with $q_u=0.5$ and $q_v=0.01$; if two classifiers each assign 1 with probability $p_u=q_u$ to $u$ and $p_v = q_v$ to $v$ 
then the probability of positive classification by either of the two classifiers will be $0.75$ for $u$ and $\approx 0.02$ for $v$,
diverging from their original distance of $|q_u-q_v|=0.49$. 

Such divergence is most clearly exhibited with small numbers of classifiers. As the number of classifiers increases, the probabilities of positive classification by at least one classifier for any pair (so long as one of the pair is accepted with positive probability by sufficiently classifiers), will eventually converge as they approach one.
However, in practice, we expect students to apply to perhaps five or 10 colleges, so it is desirable to characterize when small systems are immune to such divergence. 
We demonstrate these issues in two steps: first, we show how to construct sets of individually fair classifiers which do not satisfy \ORfairnesssp for all (universe, task) pairs under same-task composition to more formally illustrate the problem. Second, we partially characterize a large class of sets of classifiers which will satisfy \ORfairnesssp under same-task composition.

\begin{theorem}\label{theorem:orsame}
For any (universe, task) pair with a non-trivial metric $\mathcal{D}$,
there exists a set of individually fair classifiers $\mathcal{C}$ which do not satisfy OR Fairness, even if each element in $U$ is classified by all $C_i \in \mathcal{C}$.

\end{theorem}
\begin{proof}
By assumption of non-triviality of the metric $\D$, there exist $u,v \in U$ such that $1>\mathcal{D}(u,v)>0$. 
Construct $C$ such that
$d(\dist{C}(u),\dist{C}(v)) = \mathcal{D}(u,v)$ for some pair $u,v \in U$ and $\E[C(u)] + \E[C(v)] < 1$. (Lemma \ref{lemma:fairadditions} provides the necessary procedure.) Write $p_u=\E[C(u)]$ and $p_v=\E[C(v)]$ as before.

Take the set $\mathcal{C}$ to be two identical copies of $C$. Then $\E[x_u] = 1-(1-p_u)^2$ and $\E[x_v] = 1-(1-p_v)^2$.

Then:
\[|\E[x_u] - \E[x_v]| =| (1 - p_v)^2 - (1-p_u)^2|\]
\[|\E[x_u] - \E[x_v]| =| (1 - 2p_v+ p_v^2) - (1-2p_u+p_u^2)|\]
\[|\E[x_u] - \E[x_v]| =| 2(p_u-p_v) - (p_u^2-p_v^2)|\]
\[|\E[x_u] - \E[x_v]| =| 2(p_u-p_v) - (p_u-p_v)(p_u+p_v)|\]
By choice of $p_u$ and $p_v$,  $|p_u-p_v|=\mathcal{D}(u,v)$, so without loss of generality
\[|\E[x_u] - \E[x_v]| =| 2\mathcal{D}(u,v) -(p_u-p_v)(p_u+p_v) |\]
Notice that $ (p_u-p_v)(p_u+p_v) < \mathcal{D}(u,v)$ as long as $p_u+p_v \leq  1$. Thus 
\[|\E[x_u] - \E[x_v]|  > \mathcal{D}(u,v)\]
which completes the proof.
\end{proof}

To build intuition, consider the simple case where the one ``worst'' element is accepted with probability $\eta \ll \frac{1}{2}$ by all classifiers.
As all other elements with probability of acceptance greater than or equal to $\frac{1}{2}$ are repeatedly classified, their probability of at least one acceptance quickly approaches 1, exceeding their maximum distance from the ``worst'' element under $\mathcal{D}$ whose probability of acceptance increases much more slowly. 
In general, we expect classifiers to attempt to maximize the allowed distance for at least some pairs in order to increase their discriminatory power between ``good'' and ``bad'' elements for the task, increasing the chance of such \sametaskdivergence. Even if we rule out elements identically mapped to zero or $o(1)$ probability, we need only consider the divergence of $(1-p_u)^n$ and $(1-p_v)^n$ for sets of $n$ classifiers to see that this problem exists in many real-world scenarios when distances are maximized or nearly maximized, within the constraints of individual fairness, between some pairs of elements. Particularly for settings like loan applications (where an extended loan search with many credit inquiries may impact an individual's credit score), small stretches in distance may have significant practical implications. Figure \ref{fig:sametaskdivergence} illustrates an example of this scenario.
\begin{figure}[]
    \centering
    \includegraphics[width=0.65\textwidth]{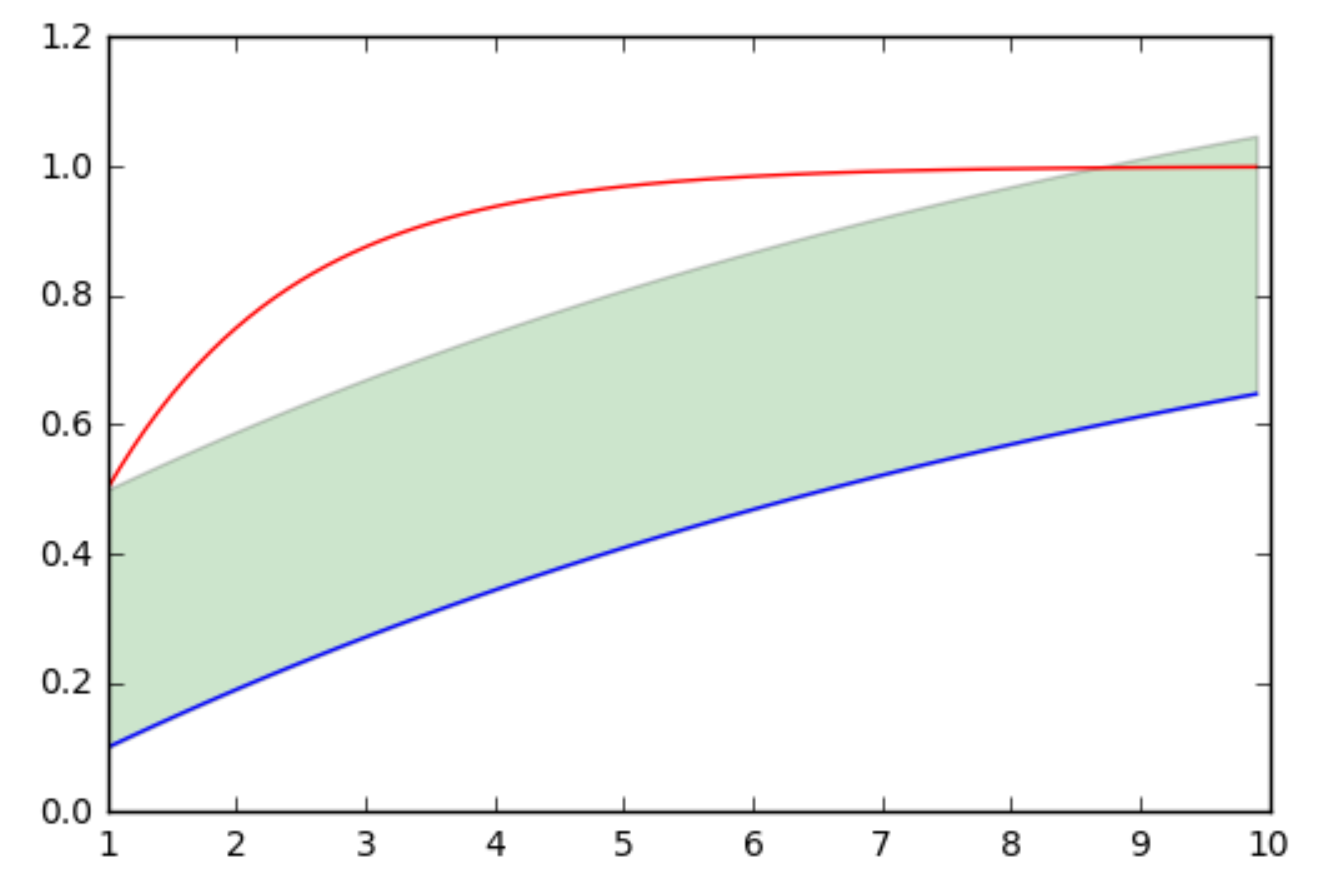}
    \caption{Comparison of $1-(1-p_u)^n$ and $1-(1-p_v)^n$ for $p_u=0.1$ and $p_v=0.5$. The shaded green region indicates the region within the original fair distance bound. Value for $p_u$ exceeds the bounded region for $n\in (1,8]$. } \label{fig:sametaskdivergence}
\end{figure}

The good news is that we can characterize non-trivial conditions for sets of small numbers of classifiers where OR Fairness is satisfied with the help of the following lemma.
\begin{lemma}\label{lemma:sametaskgeneral}
Fix a set $\mathcal{C}$ of classifiers, and let $x_w$ for $w\in U$ be the indicator variable as defined in Theorem \ref{theorem:ordifferent}.
If $\E[x_w]\geq 1/2$ for all $w \in U$,  
then the set of classifiers $\mathcal{C} \cup \{C'\}$ satisfies OR fairness if $C'$ satisfies individual fairness under the same metric and $\Pr[C'(w)=1]\geq \frac{1}{2}$ for all $w \in U$.

\end{lemma}
\begin{proof}
Consider a pair $u,v\in U$. By assumption, $|\E[x_u] - \E[x_v]| \leq \mathcal{D}(u,v)$ and $|\E[C'(u)] - \E[C'(v)]| \leq \mathcal{D}(u,v)$. 
Define
\[x_{w}' = \begin{cases}1 \text{ if } x_w + C'(w) \geq 1 \\ 0 \text{ otherwise}
\end{cases}
\]

Therefore $\Pr[x_{w}'=1]= 1- \Pr[x_w=0 \text{ and } C'(w)=0]$. Define $p_u'$ to be the probability that $u$ is accepted by $C'$ and $p_{u}$ the probability that $u$ is accepted by at least one of the set $\mathcal{C}$, and analogously define $p_v'$ and $p_v$.

It suffices that
\[|(1-(1-p_u)(1-p_u')) - (1-(1-p_v)(1-p_v'))| \leq \mathcal{D}(u,v)\]
to ensure that the system satisfies OR fairness. With some small simplifications, we have
\[|(1-(1-p_u)(1-p_u')) - (1-(1-p_v)(1-p_v'))| 
= | p_vp_v'- p_{u}p_u'+p_u-p_v+p_u' - p_v'| \]
Define $t=p_u-p_v$ and $t'=p_u'-p_v'$. Without loss of generality, assume that either $t,t'>0$ or $t$ and $t'$ have different signs. (Notice that if this doesn't hold by our original arbitrary choice of $u$ and $v$, we can switch the ordering to make it so). Therefore,
\[ | p_vp_v'- p_{u}p_u'+p_u-p_v+p_u' - p_v'|= | p_vp_v'- p_{u}p_u'+t+t'| \]
\[= | p_vp_v'- (p_{v}+t)(p_v'+t')+t+t'| \]
\[= | -p_vt' - p_v't - tt'+t+t'| \]
Note that $p_v,p_v'\geq \frac{1}{2}$, so 
\[| -p_vt' - p_v't - tt'+t+t'| \leq | \frac{1}{2}t' + \frac{1}{2}t - tt'|  \]
By assumption $|t|, |t'| \leq \mathcal{D}(u,v)$. By definition of $t,t'$, either $t,t'>0$ or $t$ and $t'$ have different signs.  
Thus 
\[| \frac{1}{2}t' + \frac{1}{2}t - tt'|  \leq \mathcal{D}(u,v)\]
which concludes the proof.

\end{proof}

Lemma \ref{lemma:sametaskgeneral} turns out to be quite useful for determining that a system is free from \sametaskdivergence, even when the classifiers do not initially seem to satisfy the requirements for the lemma. 
Consider a set of classifiers $\mathcal{C}$ such that $\mathcal{C}' \subseteq \mathcal{C}$ do not satisfy the requirements for Lemma \ref{lemma:sametaskgeneral}. However, if we group the classifiers together as an ``OR of ORs''
\[C_{j,k}(w)=\begin{cases} 1 \text{ if }\sum_{i \in \{j,\ldots, k\}}C_i(w) \geq 1 \\ 0 \text{ otherwise} \end{cases}\]
so that each grouped classifier has $\E[C_{j,k}(x_u)] \geq \frac{1}{2}$, we may now apply Lemma \ref{lemma:sametaskgeneral} to the \textit{grouped} classifiers. 
This ``OR of heavy ORs'' can be broadened to an ``OR of heavy functions'' (where ``heavy" corresponds to having value one with probability at least $1/2$) in cases where each input to the ``OR'' is an arbitrary Boolean function of classifier outcomes.
In practice, we expect that this test will be simpler to implement than fully analyzing the set of classifiers.

So far we have considered \ORfairnesssp as the primary setting for functional composition. The reader has likely already guessed that our observations extend to other operators. We omit a complete set of proofs in the same-task setting as they are very similar to the above. Figures \ref{fig:and} and \ref{fig:xor} show illustrative cases (as Figure \ref{fig:sametaskdivergence} does for \ORfairness). 
\subsubsection{Multiple Task Functional Composition}\label{section:multipletaskfunctional}

As we saw in the previous section, functional composition of classifier outputs for a single task can result in violations of individual fairness, even if the classifiers were individually fair in isolation. We now extend this idea to cases where more than one task is used in the functional composition. Whereas OR Fairness is highly intuitive for single-task settings, AND Fairness is highly applicable in multiple-task settings.
For example, in order to attend university, a student must both be admitted to the university \textit{and} be able to pay tuition, either through private scholarships, family assistance or university financial aid. At first glance, this may seem like a single-task composition problem, however, it is highly likely that the financial aid office and the university admissions department may consider themselves beholden to different metrics. The admissions department may want to evaluate applications ``need-blind'' and not consider financial status at all in order to admit the most academically qualified candidates; the financial aid office may want to maximize the number of students able to attend given a fixed amount of aid money. It's not difficult to imagine that these two metrics compose poorly.

For a warm-up example, let us consider the case of 10 students, all with the same academic qualifications, who are admitted to the university based on their academic qualifications. Two of them have family funding or a private scholarship to cover their tuition, five  need 25\% tuition assistance from the university, and the remaining three need 100\% tuition assistance from the university. If the financial aid office only has a limited amount of funding, they could, fairly, in their opinion, offer all students a full financial aid package with some probability $p \in (0,1)$.\footnote{We specifically address fairly allocating limited `slots' or resources in Section \ref{section:cohort}.} However, students who have alternative means will be able to attend regardless of the financial aid they receive, whereas students who do not have alternative means will only be able to attend with probability $p$. Although the financial aid and admissions classifiers both appear fair independently, we are faced with the problem of how to reconcile fairness under composition. It's not clear which, if either, of the input metrics is the right metric to use to enforce fairness on the system as a whole, and so we must consider systems where the relevant metric for the output may not be included in the metrics of any of the input tasks. In this particular case, the relevant output metric should perhaps be more closely aligned with academic qualification than financial background. The definition of AND Fairness, below, considers this setting. 

\begin{definition}[AND Fairness]
Given a universe $U$ and a set of $k$ tasks $\T$ with metrics $\mathcal{D}_1,\ldots, \D_k$ and an output metric $\D^\Outcome$, a set of classifiers $\mathcal{C}$ satisfies \textit{AND Fairness} if the indicator variable
\[x_u = \begin{cases}1 \text{ if } \prod_{C_i \in \mathcal{C}}C_i(x) \geq 1\\ 0 \text{ otherwise}\end{cases}\]
satisfies 
$\mathcal{D}^*(u,v) \geq d(\dist{x}_u,\dist{x}_v)$ for all $u,v \in U$, where $\dist{w}_u = \Pr[x_w=1]$ for $w \in U$.
\end{definition}

We expect that there are relevant scenarios where $\D^\Outcome \in \T$ and others in which $\D^\Outcome \notin \T$, so we make no explicit requirement in the definition. For example, in the case of college admissions, $\D^\Outcome$ may be taken to be the metric for academic qualification. In contrast, for the task of buying a home, the metric $\D^\Outcome$ may be distinct from the metrics for securing financing and finding a willing seller. 
The next theorem shows that when the output metric doesn't have strictly larger distances than all of the input metrics for all pairs, then individual fairness can easily be violated by composing classifiers that are individually fair in isolation.

\begin{theorem}\label{theorem:andunfairness} Let $\T$ be a set of $k$ tasks with nontrivial metrics $\mathcal{D}_1,\ldots, \D_k$ respectively and let $\D^\Outcome$ represent the relevant outcome metric. If there exists at least one pair $u,v \in U$ and one pair of tasks $T_i, T_j$ for $i,j \in [k]$, $i \neq j$ such that 
\begin{enumerate}
    \item $\D^\Outcome(u,v) \leq \D_i(u,v), \D_j(u,v)$
    \item $D_i(u,v),D_j(u,v)>0$
\end{enumerate}  
there exists a set of classifiers $\mathcal{C}$ that satisfy individual fairness separately, but do not satisfy AND Fairness under composition.
\end{theorem}
\begin{proof}
Fix a pair $u,v \in U$ which have positive distance for two tasks $T_i,T_j$ and $\D^\Outcome(u,v) \leq \D_i(u,v),\D_j(u,v)$.

First, select a classifier $C$ for task $i$ such that $p_u-p_v = \D_i(u,v)$ (Lemma \ref{lemma:fairadditions} provides the necessary construction procedure). We will show how to select $p_u',p_v'$ for the classifier $C'$ for task $j$ such that $|p_u'-p_v'| \leq \D_j(u,v)$, but the composition violates AND Fairness. Notice that the difference probability of positive classification under AND is equivalent to
\[d_{AND}(u,v) = |p_up_u' - p_vp_v'|\]

Given the constraints of the theorem statement, there are two possible cases. 
\begin{enumerate}
    \item $\D_i(u,v)>\D^\Outcome(u,v)$ (or symmetrically $\D_j(u,v)>\D^\Outcome(u,v)$). 
    This case is trivial; if $|p_u - p_v|=\D_i(u,v)>\D^\Outcome(u,v)$, we can simply select $p_v'=p_u'=1$ to violate AND Fairness with respect to $\D^\Outcome$.
    \item $\D_i(u,v)=\D_j(u,v)=\D^\Outcome(u,v)$.  We instead select $p_u',p_v'$ such that
    $p_u' - p_v' = \D_j(u,v)=\D^\Outcome(u,v)$. Rearranging the equation above, we now have 
    \[d_{AND}(u,v) = p_up_u' - p_vp_v'\]
    \[d_{AND}(u,v) = p_up_u' - p_v(p_u'-\D^\Outcome(u,v))\]
    Then substituting our original choice for $p_u-p_v=\D_i(u,v)=\D^\Outcome(u,v)$
    \[d_{AND}(u,v) = (p_v + \D^\Outcome(u,v))p_u' - p_v(p_u'-\D^\Outcome(u,v))\]
    \[d_{AND}(u,v) = p_u'\D^\Outcome(u,v) + p_v\D^\Outcome(u,v)\]
    Thus choosing $p_u'$ such that $p_u'+p_v>1$ is sufficient to violate the distance for AND Fairness with respect to $\D^\Outcome$. (Choosing $p_u'=1$ is sufficient to achieve this.)
\end{enumerate}
\end{proof}

Notice that this theorem is only a loose characterization of the cases that will violate AND Fairness, as it only takes advantage of distance under two classifiers.

\begin{figure}
    \centering
    \includegraphics[width=0.65\textwidth]{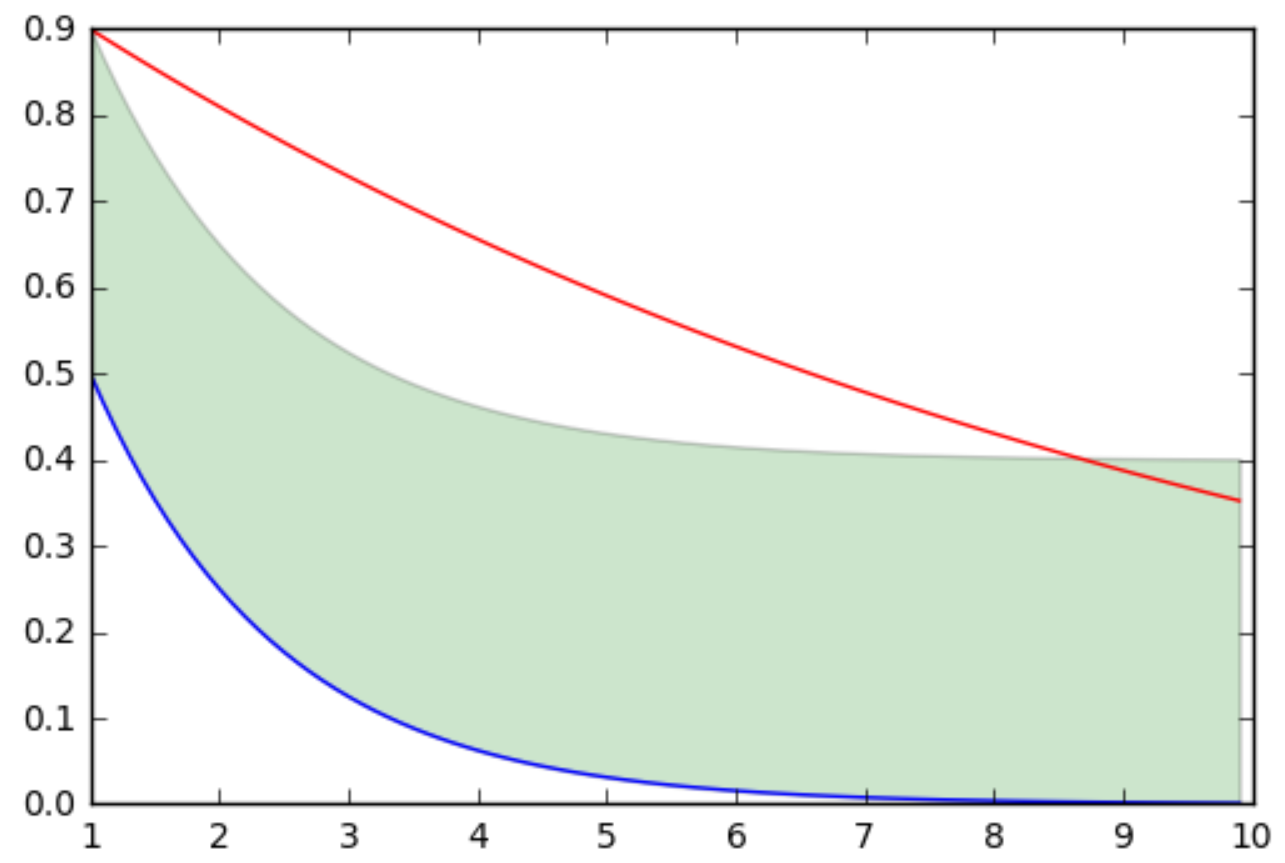}
    \caption{Comparison of the AND same-task composition with $p_u=0.9$ and $p_v=0.5$. The red and blue lines track $p_u^n$ and $p_v^n$. The shaded green region indicates the region within the original fair distance bound. Value for $p_u$ exceeds the bounded region for $n\in (1,8]$. Notice that unlike OR fairness, $p_u,p_v\geq 0.5$ can result in unfairness under same-task AND composition.} \label{fig:and}
\end{figure}

\begin{figure}
    \centering
    \includegraphics[width=0.65\textwidth]{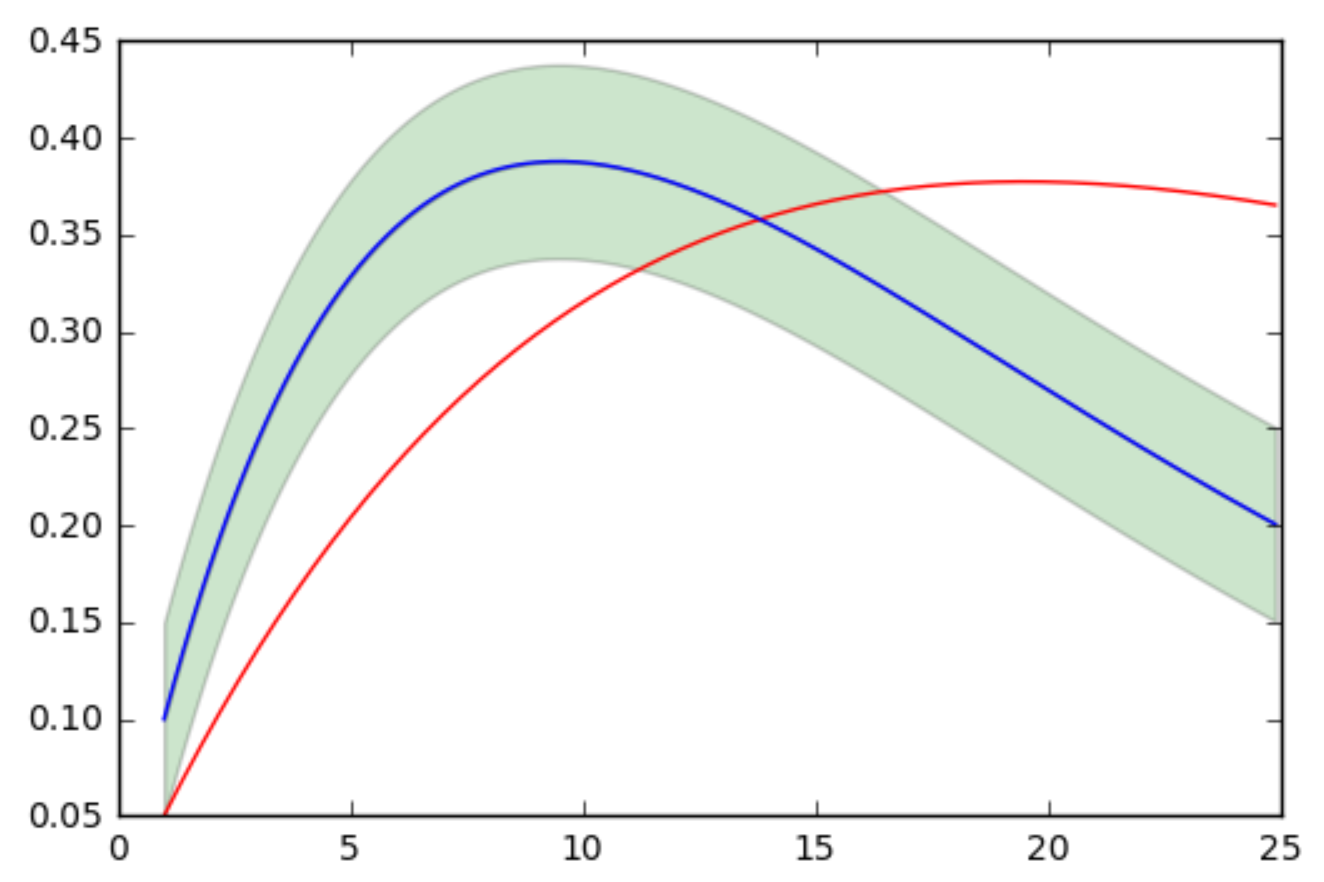}
    \caption{Comparison of the XOR same-task composition, that is, exactly one positive classification, with $p_u=0.05$ and $p_v=0.1$. The red and blue lines track $np_u(1-p_u)^{n-1}$ and $np_v(1-p_v)^{n-1}$. The shaded green region indicates the region within the original fair distance bound. Notice that there is no clear region with fewer than 25 applications for which the two converge for more than a short region.} \label{fig:xor}
\end{figure}

\subsubsection{Implications of functional compositions}
We have shown that na\"{\i}ve application of multiple fair classifiers for the same task may result in unfairness. Our characterization of unfairness in the same-task setting is not exhaustive, but it does give us intuition about how to reason about systems where multiple, independent classifiers for the same task must interact. In particular, we have shown for settings like college admissions or loan applications, where the number of applications are generally less than 10, that each bank or college acting fairly in isolation is not enough to ensure fairness of the system as a whole. 

The observations on same-task composition are not all doom and gloom, however. Our characterizations indicate that  many natural tasks where the number of repeated classifications is large, like advertising, are unlikely to suffer from same-task composition issues, assuming no individuals are systematically excluded from the system. 

In the case of multiple task functional compositions, the results are a bit less optimistic, but roughly match our intuition that in scenarios which require consideration of ``irrelevant'' attributes for the task (such as ability to pay tuition) it may be difficult to guarantee fairness. We discuss other similarly constrained scenarios in Section \ref{section:cohort}.

Comparing functional composition to differential privacy, it is important to understand that each component satisfying individual fairness separately (and for different metrics) is not analogous to the composition properties of differential privacy. With differential privacy, we assume a single privacy loss random variable which evolves gracefully with each release of information, increasing in expectation over time. However, with fairness, we may see that fairness loss increases or decreases (depending on the number and type of compositions) in idiosyncratic ways.  Moreover, we may need to simultaneously satisfy many different task-specific `fairness budgets,' and a bounded increase in distance based on one task may be catastrophically large for another.

\subsection{Multiple-Task Composition}\label{section:multipletask}
Next, we turn our attention to composition of classifiers for multiple tasks where the outcome for more than one task is decided. The first question we must consider is how to evaluate fairness constraints on systems that affect outcomes for multiple tasks. Multiple Task Fairness, defined next, requires fairness to be enforced independently and simultaneously for each task.

\begin{definition}[Multiple Task Fairness] For a set of $k$ tasks $\mathcal{T}$ with metrics $\mathcal{D}_1, \ldots, \mathcal{D}_k$,  a (possibly randomized) system $\mathcal{S}: U \times r \rightarrow \{0,1\}^k$, which assigns outputs for task $i$ in the $i^{th}$ coordinate of the output, satisfies multiple task fairness if for all $i \in [k]$ and all $u,v \in U$
\[\mathcal{D}_i(u,v) \geq |\E[\mathcal{S}_i(u)]-\E[\mathcal{S}_i(v)]|\]
where $\E[\mathcal{S}_i(u)]$ is the expected outcome for the $i^{th}$ task in the system $\mathcal{S}$ and where the expectation is over the randomness of the system and all its components. 
\end{definition}
 
Enforcing multiple task fairness makes sense when the tasks, and therefore outcomes, are distinct and incomparable. 
For example, consider an advertising system which shows users ads for either high paying jobs or home appliances. 
If two users are similarly qualified for high paying jobs, they should see a similar number of ads for high paying jobs, regardless of their intentions to buy home appliances. 
Essentially we do not want to allow positive distance for a task $T_i$ to be used to increase the distance over outcomes for a different task $T_j$.

\subsubsection{Task-Competitive Composition}
We now pose the relevant problem for multiple task fairness: choice or competitive composition. 
Clearly if classifiers for each task may \textit{independently} and fairly assign outputs, the system as a whole satisfies multiple task fairness.
However, most systems will require trade-offs between tasks. 
For example, two activities may be offered at the same time, or a website may only have one slot in which to show an advertisement. 
We therefore define the following problem:

\begin{definition}[Single Slot Composition Problem]
A (possibly randomized) system $\mathcal{S}$ is said to be a solution to the single slot composition problem 
for a set of $k$ tasks $\mathcal{T}$ with metrics $\mathcal{D}_1, \ldots, \mathcal{D}_k$, 
if $\forall u \in U$ $\mathcal{S}$ assigns outputs  for each task $\{x_{u,1}, \ldots, x_{u,k}\} \in \{0,1\}^k$ such that
\[\sum_{i \in [k]}x_{u,i} \leq 1\]
and $\forall$ $i \in [k]$, and $\forall$ $u,v \in U$
\[\mathcal{D}_i(u,v) \geq |\E[x_{u,i}]-\E[x_{v,i}]|\]
\end{definition}

The single slot composition problem captures a scenario in which a system can choose at most one of a set of possible outcomes, but need not choose any outcome. For example, an advertising platform may have a single slot to show an ad. Imagine that this advertising system only has two types of ads: those for jobs and those for household goods. If a person is qualified for jobs and wants to purchase household goods, the system must pick at most one of the ads to show. In this scenario, it's unlikely that the advertising system would choose to show no ads, but the problem specification does not require that any positive outcome is chosen. 

To solve the single-slot composition problem we must build a system which chooses at most one of the possible tasks so that fairness is preserved for each task across all elements in the universe. This problem can be extended to consider up to $k-1$ slots, but as in our discussion of OR-fairness, we only formally consider the single-slot version for clarity and ease of reading. 

\paragraph{Na\"{\i}ve Multiple-Task Composition}

The simplest scenario to consider is a single instance of the single-slot composition problem. For our motivating example, we'll consider two advertisers competing for a single advertisement slot on a website.

Task-Competitive Composition, defined below, captures the essence of several natural simple compositions. 

As only one ad can show at once, we first define the notion of a \tiebreakingfunction:
\begin{definition}[Tie-breaking Function] A (possibly randomized) \textit{\tiebreakingfunction} $\Tiebreak: U \times \{0,1\}^* \times \{0,1\}^k \rightarrow [k] \cup \{0\}$ takes as input an individual $w\in U$ and a $k-$bit string $x_w$ and outputs the index of a ``1" in $x_w$ if such an index exists and 0 otherwise. 
\end{definition}

Note that the \tiebreakingfunction~need not encode the same logic for all individuals, or conform to any particular notion of internal consistency. That is, $\Tiebreak$ may encode that $w$ prefers outcome $A$ to outcome $B$, and outcome $B$ to outcome $C$, \textit{and} outcome $C$ to outcome $A$. 
The \tiebreakingfunction~may also be randomized. That is, with probability $p_{A,B}$, outcome $A$ is preferred to outcome $B$. When the probability of the preference is $1$ or $0$, we refer to it as a strict preference, as the output is strictly preferred for that particular element.
In an ad setting, for example, a strict preference might indicate that one advertiser consistently outbids the other. This strict preference might apply for all elements in the universe, or only a subset.

Implicit in this definition is that if there is no tie to be broken, the single positive classification is preferred. 
This is a reasonable model both for advertising situations (the advertising platform prefers to have revenue from showing as many ads as possible) and in situations where both outputs are desirable. The \tiebreakingfunction~also captures situations where ordering of classifiers (or decisions) is based on a fixed policy or there is time pressure to respond to one classifier before moving on to another.

\begin{definition}[Task-Competitive  Composition]\label{def:choicecomp}
Consider a set $\T$ of $k$ tasks, and a \tiebreakingfunction~as defined above.
Given a set $\mathcal{C}$ of classifiers for the set of tasks, define $y_w = \{y_{w,1}, \ldots, y_{w,k}\}$ where $y_{w,i}=C_i(w)$. The \competitivecomposition~of the set $\mathcal{C}$ is defined as
\[y_{w}^* = \Tiebreak(w,y_w)\]
for all $w\in U$.
\end{definition}

Task-competitive composition can reflect cases where classifiers are applied in a strict ordering until a positive classification is reached or where classifiers are applied simultaneously and a single output is selected.
For example, in the case of loan applications, a \competitivecomposition~could be used to reflect the process of applying for loans one at a time, using strict preference to indicate ordering. In the case of advertising, the \tiebreakingfunction~can express the probability that one advertiser outbids another. For notation convenience in the two task setting, we refer to $\Tiebreak_w(T)$ as the probability that $T$ is chosen when both $T$ and $T'$ are options.

Before stating and proving the more general theorem, we address the simple case in which all $w \in U$ have the same strict preference for task $T$. 

\begin{lemma}\label{lemma:funnelunfairness}
For any two tasks $T$ and $T'$ such that the metrics for each task ($\mathcal{D}$ and $\mathcal{D'}$ respectively) are not identical and are non-trivial on a universe $U$, and if there is a strict preference for $T$, that is $\Tiebreak_w(T)=1 $ $\forall w \in U$, then there exists a pair of classifiers $\mathcal{C}=\{C,C'\}$ which are individually fair in isolation but when combined with \competitivecomposition~violate multiple task fairness.
\end{lemma}
\begin{proof}
We construct a pair of classifiers $\mathcal{C}=\{C,C'\}$ which are individually fair in isolation for the tasks $T$ and $T'$, but do not satisfy multiple task fairness when combined with \competitivecomposition~with a strict preference for $T$ for all $w\in U$. 
\Competitivecomposition~ensures that at most one task can be classified positively for each element, so our strategy is to construct $C$ and $C'$ such that the distance between a pair of individuals is stretched for the `second' task.

By non-triviality of $\mathcal{D}$, there exist $u,v$ such that $\mathcal{D}(u,v)\neq 0$. Fix such a pair $u,v$ and let $p_u$ denote the probability that $C$ assigns 1 to $u$, and analogously $p_v,p_u',p_v'$. We use these values as placeholders, and show how to set them to prove the lemma. 
 
Because of the strict preference for $T$, the probabilities that $u$ and $v$ are assigned 1 for the task $T'$
\[\Pr[\mathcal{S}(u)_{T'} = 1] = (1-p_u)p_u'\]
\[\Pr[\mathcal{S}(v)_{T'} = 1] = (1-p_v)p_v'\]
The difference between them is 
\[\Pr[\mathcal{S}(u)_{T'} = 1] - \Pr[\mathcal{S}(v)_{T'} = 1] = (1-p_u)p_u' - (1-p_v)p_v'\]
\[ = p_u' - p_up_u' - p_v' + p_vp_v'\]
\[ = p_u'  - p_v' + p_vp_v' - p_up_u'\]
Notice that if $\mathcal{D'}(u,v)=0$, which implies that $p_u'=p_v'$, and $p_u \neq p_v$, then this quantity is non-zero, giving the desired contradiction for all fair $C'$ and any $C$ that assigns $p_u \neq p_v$, which can be constructed per Corollary \ref{cor:maxdist}.

However, if $\mathcal{D'}(u,v)\neq 0$, take $C'$ such that $|p_u' - p_v'| = \mathcal{D'}(u,v)$  and denote the distance $|p_u'-p_v'|=m'$,
and without loss of generality, assume that $p_u' > p_v'$ and $p_u < p_v$, 

\[\Pr[\mathcal{S}(u)_{T'} = 1] - \Pr[\mathcal{S}(v)_{T'} = 1] = m' + p_vp_v' - p_up_u'\]
Then to violate fairness for $T'$, it suffices to show that $p_vp_v'> p_up_u'$. Write $p_v = \alpha p_u$ where $\alpha > 1$,
\[\alpha p_up_v' > p_up_u'\]
\[\alpha p_v' > p_u'\]
Thus it is sufficient to show that we can choose $p_u,p_v$ such that $\alpha > \frac{p_u'}{p_v'}$. 
Constrained only by the requirements that $p_u< p_v$ and $|p_u - p_v| \leq \mathcal{D}(u,v)$, we may choose $p_u,p_v$ to obtain an arbitrarily large $\alpha= \frac{p_v}{p_u}$ by Corollary \ref{cor:ratio}. Thus there exist a pair of fair classifiers $C,C'$ which when combined with strictly ordered \competitivecomposition~violate multiple task fairness.
\end{proof}
The intuition for unfairness in such a strictly ordered composition is that each task inflicts its preferences on subsequent tasks. 
This is most clearly seen when an equal pair for the second task are unequal for the first. 
Once the first classifier acts, so long as the distance between the two is positive, the pair have unequal probabilities of even being considered by the second classifier, which breaks their equality for that task for \textit{any} fair classifier. 

We now extend Lemma \ref{lemma:funnelunfairness} to the more general setting, in which there need not be a strict preference, and find that the problems with unfairness generalize to this case.

\begin{theorem}\label{theorem:choiceunfairness}

For any two tasks $T$ and $T'$ with nontrivial metrics $\mathcal{D}$ and $\mathcal{D'}$ respectively, there exists a set $\C$ of classifiers which are individually fair in isolation but when combined with \competitivecomposition~violate multiple task fairness for any \tiebreakingfunction.
\end{theorem}
\begin{proof}
Consider a pair of classifiers $C,C'$ for the two tasks. 
Let $p_u$ denote the probability that $C$ assigns 1 to $u$, and analogously let $p_v,p_u',p_v'$ denote this quantity for the other classifier and element combinations. 
As noted before, for convenience of notation, write $\Tiebreak_u(T)$ to indicate the preference for each (element, outcome) pair, that is the probability that given the choice between $T$ or the alternative outcome $T'$, $T$ is chosen. 
Note that in this system, for each element $\Tiebreak_u(T) + \Tiebreak_u(T')=1$.

Note that if $\Tiebreak_w(T) = 1$ $ \forall w \in U$ or $\Tiebreak_w(T') = 1$ $ \forall w \in U$, the setting is exactly as described in Lemma \ref{lemma:funnelunfairness}. Thus we need only argue for the two following cases:
\begin{enumerate}
    \item Case $\Tiebreak_u(T)=\Tiebreak_v(T) \neq 1$. We can write an expression for the probability that each element is assigned to task $T$:
    \[\Pr[\mathcal{S}(u)_T=1] = p_u(1-p_u') + p_up_u'\Tiebreak_u(T)\]
    \[\Pr[\mathcal{S}(v)_T=1] = p_v(1-p_v') + p_vp_v'\Tiebreak_v(T)\]
    So the difference in probabilities is
    \[\Pr[\mathcal{S}(u)_T=1] -\Pr[\mathcal{S}(v)_T=1] = p_u(1-p_u') + p_up_u'\Tiebreak_u(T) - p_v(1-p_v') - p_vp_v'\Tiebreak_v(T)\]
    \[ = p_u- p_v +p_vp_v'- p_up_u' + p_up_u'\Tiebreak_u(T) - p_vp_v'\Tiebreak_v(T)\]
    \[ = p_u- p_v +(p_vp_v'- p_up_u')(1-\Tiebreak_u(T)) \]

    By our assumption that $\Tiebreak_u(T) \neq 1 $, we proceed analogously to the proof of Lemma \ref{lemma:funnelunfairness} choosing $C'$ such that $p_vp_v'>p_up_u'$ and choosing $C$ to ensure that $p_u - p_v = \mathcal{D}(u,v)$ to achieve unfairness for $T$. 

    \item Case $\Tiebreak_u(T) \neq \Tiebreak_v(T)$. Assume without loss of generality that \\ $\Tiebreak_u(T) \neq 1$. Recall the difference in probability of assignment of 1 for the first task in terms of $\Tiebreak$:
    \[ = p_u- p_v +p_vp_v'(1-\Tiebreak_v(T))- p_up_u'(1-\Tiebreak_u(T)) \]
    Choose $C$ such that $p_u - p_v = \mathcal{D}(u,v)$ (or if there is no such individually fair $C$, choose the individually fair $C$ which maximizes the distance between $u$ and $v$). So it suffices to show that we can select $C'$ such that $p_vp_v'(1-\Tiebreak_v(T))- p_up_u'(1-\Tiebreak_u(T)) > 0$.
    As before, write $p_u = \alpha p_v$ where $\alpha >1$.
    We require:
    \[p_vp_v'(1-\Tiebreak_v(T))  > \alpha p_vp_u'(1-\Tiebreak_u(T))\]
    \[p_v'(1-\Tiebreak_v(T))  > \alpha p_u'(1-\Tiebreak_u(T))\]
    Writing $\beta = (1-\Tiebreak_v(T))/(1-\Tiebreak_u(T))$  (recall that $\Tiebreak_u(T) \neq 1$ so there is no division by zero), we require
    \[p_v'\beta  > \alpha p_u'\]
    \[\beta/\alpha  > p_u'/p_v'\]

    Constrained only by $|p_u' - p_v'| \leq \mathcal{D'}(u,v)$, we can choose $p_u',p_v'$ to be any arbitrary positive ratio per Corollary \ref{cor:ratio}, thus we can select a satisfactory $C'$ to exceed the allowed distance.
\end{enumerate}

Thus we have shown that for the cases where the \tiebreakingfunctions~are identical for $u$ and $v$ and when the \tiebreakingfunctions~are different, there always exists a pair of classifiers $C,C'$ which are fair in isolation, but when combined in \competitivecomposition do not satisfy multiple task fairness which completes the proof.
\end{proof}

The natural intuition for \competitivecomposition~might be that tie-breaking preferences could ease unfairness for some classifiers. 
However, in many natural tasks we actually expect preferences to work against us. Indeed, when one task is strictly preferred over the other, a very natural case, \competitivecomposition~always splits pairs of individuals who are unequal in the preferred task and equal in the other task. For example, consider the case of free school breakfasts and a new SAT preparation class offered before school. Students qualified for the SAT class must decide if they would rather eat breakfast or attend SAT class. The natural human preference not to be hungry will likely win. 
The right solution here is to offer breakfast in a way that doesn't conflict with SAT class attendance (eg, offering bagged breakfast that can be taken to class).

Another important consideration in \competitivecompositions~is \textit{whose} \tiebreakingfunction~is used. We might initially assume the choice is made by the individuals classified, but in fact, it could be made by the classifiers (either independently or jointly) or the system itself. Advertising auctions are a good example where the \tiebreakingfunction~is related to the bid for each person by each advertiser, not necessarily each person's preference to see the ad.

Although the formal statement of Theorem \ref{theorem:choiceunfairness} only implies that individually fair classifiers exist that exhibit unfairness under \competitivecomposition, our intuition  suggests that this happens often in practice and that small relaxations will not be sufficient to alleviate this problem, as the phenomenon has been observed empirically \cite{datta2015automated,lambrecht2016algorithmic,kuhn2012gender}. To see this, we revisit the proof of Theorem \ref{theorem:choiceunfairness}, in particular the requirement that $\beta / \alpha > p_u'/p_v'$ to build our intuition.

Recall that $\beta = \Tiebreak_v(T')/\Tiebreak_u(T')$, $\alpha = p_u/p_v>1$. Thus the constraint can be rewritten
\[\frac{\Tiebreak_v(T')p_v}{\Tiebreak_u(T')p_u}>\frac{p_u'}{p_v'}\]
\[\Tiebreak_v(T')p_vp_v'>\Tiebreak_u(T')p_up_u'\]
Imagine the case where $p_u>p_v$, but $p_u'<p_v'$. If $\Tiebreak_u(T')<\Tiebreak_v(T')$ that is, the elements tend to prefer the tasks for which they are highly qualified, then there are many solutions where the inequality holds. 
We include a small empirical example in Appendix \ref{section:empirical} to illustrate the potential magnitude and frequency of such fairness violations.

\paragraph{Simple Fair Multiple-task Composition}\label{section:fairmultipletask}
We now show how to fairly compose tasks for the single slot composition problem. Perhaps the most obvious solution is to remove the conflict in the tasks. Ideally each task can be classified separately and the outcome decided without influencing other tasks. In practice, we know that this is not always feasible.

In some special scenarios, we could choose to optimize the classifiers together with knowledge of the \tiebreakingfunction, both utility functions and both metrics. This allows each classifier to appropriately respond to the other to achieve fairness without sacrificing too much utility in some, but not all cases of \competitivecomposition. However, this would require significant coordination and cooperation on the part of those responsible for each task, so this is unlikely to be practical in some situations.

Fortunately, in some situations there is a general purpose mechanism for the single slot composition problem which requires no additional information in learning each classifier and no additional coordination between the classifiers.

\begin{algorithm}[tb]
   \caption{\randomizeandclassify}
   \label{alg:randomizethenclassify}
\begin{algorithmic}
   \STATE {\bfseries Input:} universe element $u \in U$, set of fair classifiers $\mathcal{C}$ (possibly for distinct tasks)
   operating on $U$, probability distribution over tasks $\mathcal{X} \in \Delta(\mathcal{C})$
   \STATE $x \leftarrow 0^{|\mathcal{C}|}$
   \STATE $C_t \sim \mathcal{X}$
   \IF{$C_t(u)=1$}
   \STATE $x_{t}=1$ 
   \ENDIF
   \STATE return $x$
\end{algorithmic}
\end{algorithm}

\begin{theorem}\label{theorem:simplemultiple}
For any set of $k$ tasks $\mathcal{T}$ with metrics $\mathcal{D}_1,\ldots,\mathcal{D}_k$, the system $\mathcal{S}$ described in Algorithm \ref{alg:randomizethenclassify}, \randomizeandclassifynospace,

achieves multi-task fairness for the single slot composition problem given any set of classifiers $\mathcal{C}$ for the tasks which are individually fair in isolation.
\end{theorem}
\begin{proof}
Consider the procedure outlined in Algorithm \ref{alg:randomizethenclassify}. For each element, the procedure outputs
a single positive classification by construction, so the procedure satisfies that constraint of the single slot composition problem.

Note that as the same probability distribution $\mathcal{X}$ and set of classifiers are used for each element $w \in U$, each element has equal probability of having task $T$ selected and the subsequent classifications for that task are fair. So the probability of positive classification in any task is $\Pr[t \sim \mathcal{X} = T]*\Pr[C_T(w) = 1]$. So the difference in probability of positive classification for an arbitrary task $T$ is 
\[\Pr[t \sim \mathcal{X} = T]*\Pr[C_T(u) = 1] - \Pr[t \sim \mathcal{X} = T]*\Pr[C_T(v) = 1]\] 
\[= \Pr[t \sim \mathcal{X} = T](\Pr[C_T(u) = 1] -\Pr[C_T(v) = 1]))  \]

which satisfies individual fairness as long as $C_T$ is individually fair in isolation.

Thus the system which applies RandomizeThenClassify to every element in the universe is a solution to the single slot composition problem as long as each $C \in \mathcal{C}$ is individually fair in isolation.
\end{proof}

\randomizeandclassify has several nice properties. First, it requires no coordination in the training of the classifiers. In particular, it does not require any sharing of objective functions. Second, it preserves the ordering of elements by each classifier. That is, if $\Pr[C_i(u)=1] > \Pr[C_i(v)=1]$ then  $\Pr[\mathsf{RandomizeThenClassify}(u)_i=1] > \Pr[\mathsf{RandomizeThenClassify}(v)_i=1]$. Finally, it can be implemented by a platform or other third party, rather than requiring the explicit cooperation of all classifiers.
The primary downside of \randomizeandclassify is that it drastically reduces allocation (the total number of positive classifications) for classifiers trained with the expectation of being run independently. 

\subsubsection{Summary}\label{subsection:nuance}
One critique of the single slot problem is the idea that more qualified people should simply be shown more ads, or allowed to split time between slots. This is a matter of design, and our mechanisms only look at the very simple design paradigm of a single slot. It's not hard to imagine that a good user interface could manage two slots for users where competition is very high, but this would require a careful analysis of the cognitive load and impact on that user. However, at some point, there will not be room for additional slots, or scheduling flexibility, to allow attendance to all events, and at that point we will be in the same setting explored here.

We primarily consider the case of honest designers with good intentions.  However, failing to enforce multiple-task fairness allows for a significant expansion of the ``catalog of evils'' outlined in \cite{dwork2012fairness}. For example, let us assume that more women than men emphasize team work and organizational skills on their resumes. An employer seeking to hire more men than women for a technical role could aggressively advertise a second role in teamwork management (for which there is only one opening) for which many women will be qualified in order to prevent women from seeing the more desirable technical position ad. This ``generalized steering'' may allow the employer to divert members of a certain group away from a desirable outcome, in analogy to the illegal ``steering'' of minorities to less desirable credit card offerings. 

\begin{remark}

Given a pair of tasks $T$ and $T'$ with metrics $\D$ and $\D'$, our goal is to ensure that the system produces outputs for each task with distributions on outcomes that are 1-Lipschitz with respect to their respective metrics. Taking inspiration from Differential Privacy, one might try allocating a fairness loss `budget' between the (potentially interfering) classifiers for the two tasks. However, such a budget would have to take into account the distances under both tasks -- leading to an unnecessary reduction in optimization flexibility. For example, if a pair $u,v$ are close under $\D$, but far under $\D'$, the budget must be the minimum of the two to prevent potential unfairness for $T$ (this follows from Theorem \ref{theorem:andunfairness}). Algorithm \ref{alg:randomizethenclassify} allows more flexibility than such a budgeting solution without additional coordination in learning each classifier.
\end{remark}

\subsection{Dependent Composition}\label{section:cohort}

Thus far, we have restricted our attention to the mode of operation in which classifiers act on the entire universe of individuals at once and each individual's outcome is decided independently. In practice, however, this is an unlikely scenario, as classifiers may be acting as a selection mechanism for a fixed number of elements, may operate on elements in arbitrary order, or may operate on only a subset of the universe.

In this section, we consider the problems associated with selecting sets of individuals from the universe when outcomes may not be decided independently for each individual. 
Somewhat abusing the term ``composition,'' these problems  
can be viewed as a composition of the classifications of elements of the universe.
We roughly divide these topics into Cohort Selection problems, when a set of exactly $n$ individuals must be selected from the universe, and Universe Subset problems, when only a subset of the relevant universe for the task is under the influence of the classifier we wish to analyze or construct.

Within these two problems we will also consider several relevant settings:
\begin{itemize}
    \item Online versus offline: in many real-world settings, immediate classification response is critical. For example, advertising decisions for online ads must be made immediately upon impression and employers must render employment decisions quickly or risk losing out on potential employees or taking too long to fill a position.
    \item Random permutations versus adversarial ordering: when operating in the online setting, the ordering of individuals may be adversarial or a random permutation of the universe (or subset). In practice, we expect that ordering will most likely not be a random permutation on the universe. For example, the order in which individuals apply for a job opening may be influenced by their social connections with existing employees, which impacts how quickly they hear about the job opening. 
    \item Known versus unknown subset or universe size: it is rare that a single classifier dictates the outcomes for a precisely defined universe or subset, and instead they generally act on a subset or universe of unknown size. The subset size may not be known in advance if it is generated randomly, or if the classifier simply doesn't have access to hidden subset selection processes. For example, an advertiser may know the average number of individuals who visit a website on a particular day, but be uncertain on any particular day of the exact number, and the fraction of who are interested in the products or services they wish to advertise.
    \item Constrained versus unconstrained selection: in many settings there are arbitrary constraints placed on selection of individuals for a task which are unrelated to the qualification or metric for that task. For example, to cover operating costs, a college may need at least $n/2$ of the $n$ students in a class to be able to pay full tuition. 
\end{itemize}

In dependent composition problems, it is important to pay careful attention to the source of randomness used in computing distances between distributions over outcomes. Taking inspiration from the experiment setup found in many cryptographic definitions, we  formally define two problems: Universe Subset Classification and Cohort Selection. We introduce new notation in the definitions below, with additional exposition.
 
\begin{definition}[Universe Subset Classification Problem]\label{def:partitionclass}
Given a universe $U$, let $\mathcal{Y}$ be a distribution over subsets of $U$.
Let $\mathcal{X}=\{\mathcal{X}(V)\}_{V \subseteq U}$ be a family of distributions, one for each subset of $U$, where $\mathcal{X}(V)$ is a distribution on permutations of the elements of $V$.  Let $\Pi(2^U)$ denote the set of permutations on subsets of $U$.
Formally, for a system $\mathcal{S}:\Pi(2^U) \times \{0,1\}^* \rightarrow U^*$, we define the following experiment.

\textbf{$\mathsf{Experiment}(\mathcal{S},\mathcal{X},\mathcal{Y},u)$:} 
\begin{enumerate}
    \item Choose $r \sim \{0,1\}^*$
    \item Choose $V \sim \mathcal{Y}$
    \item Choose $\pi \sim \mathcal{X}(V)$
    \item Run $\mathcal{S}$ on $\pi$ with randomness $r$, and output $1$ if $u$ is selected (positively classified).
\end{enumerate}

The system $\mathcal{S}$ is individually fair and a solution to the Universe Subset Classification Problem for a particular $(\mathcal{X},\mathcal{Y})$ pair if for all $u,v \in U$
\[|\E[\mathsf{Experiment}(\mathcal{S},\mathcal{X},\mathcal{Y},u)]-\E[\mathsf{Experiment}(\mathcal{S},\mathcal{X},\mathcal{Y},v)]| \leq \D(u,v) \]
Note that for any distinct individuals $u, v \in U$, in any given run of the experiment $V$ may contain $u,$ $v,$ neither or both.
\end{definition}

In some cases we will use $\E_{V \sim \mathcal{Y}, \pi \sim \mathcal{X}, r}[\mathcal{S}(u)]$ to denote to the probability that the experiment selects or positively classifies $u$. 

We adopt the convention of specifying $\mathcal{S}$ independently of $\mathcal{X}$ and $\mathcal{Y}$ as these two distributions are likely not under the control of $\mathcal{S}$, and in practice may not even be known to $\mathcal{S}$. For example, an employer may create a resume screening system without knowledge of the ordering or the subset of eligible candidates who will apply within a week of posting a new job. $\mathcal{Y}$ may capture that local job-seekers are more likely to apply than those from out of state, and $\mathcal{X}$ may capture that job-seekers with social ties to current employees will apply before other local candidates. However, we still want the employer to fairly hire regardless of the ordering of the applicants.

Next, we introduce Cohort Selection, which is identical to the Universe Subset Classification Problem, with the additional requirement that the system must select a set of exactly $n$ elements from $U$.

\begin{definition}[Cohort Selection Problem]\label{def:basiccohort}
Given a universe $U$, an integer $n$ and a task with metric $\D$, select a set of $n$ individuals such that the probability of selection is 1-Lipschitz with respect to $\D$, where the probability of selection is taken over all randomness in the system. 
As above, let $\mathcal{Y}$ be a distribution over subsets of $U$. 
 Let $\mathcal{X}=\{\mathcal{X}(V)\}_{V \subseteq U}$ be a family of distributions, one for each subset of $U$, where $\mathcal{X}(V)$ is a distribution on permutations of the elements of $V$.  Let $\Pi(2^U)$ denote the set of permutations on subsets of $U$.
Formally, for a system $\mathcal{S}_n:\Pi(2^U) \times \{0,1\}^* \rightarrow U^n$, we define the following experiment.

Formally, for a system $\mathcal{S}_n:U \times r \rightarrow U^n$, we define the following experiment.

\textbf{$\mathsf{Experiment}(\mathcal{S}_n,\mathcal{X},\mathcal{Y},u)$:} 
\begin{enumerate}
    \item Choose $r \sim \{0,1\}^*$
    \item Choose $V \sim \mathcal{Y}$
    \item Choose $\pi \sim \mathcal{X}(V)$
    \item Run $\mathcal{S}_n$ on $\pi$ with randomness $r$, and output $1$ if $u$ is selected (positively classified).
\end{enumerate}

The system is individually fair and a solution to the Cohort Selection Problem if for all $u,v \in U$, $\mathcal{S}_n$ outputs a set of $n$ distinct elements of $U$ and \\
$|\E[\mathsf{Experiment}(\mathcal{S}_n,\mathcal{X},\mathcal{Y},u)]-\E[\mathsf{Experiment}(\mathcal{S}_n,\mathcal{X},\mathcal{Y},v)]|$
$\leq \D(u,v)$.

\end{definition}

Cohort Selection is Universe Subset Classification with the additional constraint  that the system must select exactly $n$ elements. 

\subsubsection{Basic Offline Cohort Selection}
First we consider the simplest version of the cohort selection problem (Definition \ref{def:basiccohort}): choosing a cohort of $n$ individuals from the universe $U$ when the entire universe is known and decisions are made offline. In this case, $\mathcal{Y}$ is very simple, with weight $1$ on the set $U$ (i.e. $\mathcal{Y}(V)=0$ for all $V \subsetneq U$), and $\mathcal{X}$ is not meaningful, as the system has access to the entire set, and can randomize the order of the elements. 

A simple solution is to choose a permutation of the elements in $U$ uniformly at random, and then apply a fair classifier $C$ until $n$ are selected. 
Algorithm \ref{alg:permutethenclassify} works through a list initialized to a random permutation $\pi(U)$, classifying elements one at a time and independently until either (1) $n$ elements have been selected or (2) the number of remaining elements in the list equals the number of remaining spots to be filled. Case (2) is referred to as the ``end condition''. Elements in the ``end condition'' are selected with probability 1.

\begin{algorithm}[H]
   \caption{\permutethenclassify}
   \label{alg:permutethenclassify}
\begin{algorithmic}
   \STATE {\bfseries Input:} $n \leftarrow $ the number of elements to select\\
$C \leftarrow$ a classifier $C: U \times \{0,1\}^* \rightarrow  \{0,1\}$ \\
$\pi \sim S_{|U|}$ a random permutation from the symmetric group on $|U|$ \\
$L \leftarrow \pi(U)$ An ordered set of elements\\
\STATE $M \leftarrow \emptyset$\\
 \WHILE{$|M|<n$:}
 \STATE $u \leftarrow pop(L)$ 
 \IF{ $C(u)=1$} 
 \STATE $M \leftarrow M \cup \{u\}$\\
 \ENDIF
 \IF{ $n-|M|\geq |L|$} 
 \STATE \textit{// the end condition} \\ 
 \STATE $M \leftarrow M \cup \{u\}$
 \ENDIF
 \ENDWHILE

\STATE return $M$

\end{algorithmic}
\end{algorithm}

\begin{figure}[H]
    \centering
    \includegraphics[width=0.7\textwidth]{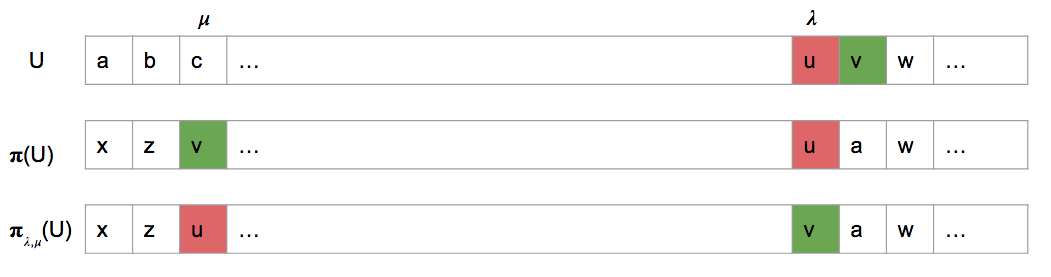}
    \caption{Example of permutation of universe $\pi(U)$ and $\pi_{\lambda,\mu}(U)$.
    } \label{fig:permutationlambdamu}
\end{figure}
\begin{theorem}\label{theorem:ptc}\permutethenclassifynospace~is a solution to the Cohort Selection Problem for any $C$ that is individually fair when operating on all elements of the universe.
\end{theorem}
\begin{proof}

Let $u,v$ be an arbitrary pair of distinct elements in $U$. Let $\pi$ be an arbitrary permutation of $U$, and let $\lambda$ and $\mu$ denote the location of $u$ and $v$ respectively in the list $L=\pi(U)$,
as shown in Figure \ref{fig:permutationlambdamu}. 

The proof proceeds by reasoning about the probability that $u$ and $v$ are selected at their given positions in $\pi$, and a related permutation which switches their positions $\pi_{\lambda,\mu}$ and using these relations to determine a bound on their differences in probability of selection.

To determine $\Pr[\lambda \text{ reached} | \pi]$ we need to determine all of the ways
that $\lambda-1-(n-1)=\lambda - n$ of the first $\lambda-1$ elements are not included in $M$. Define a \textit{configuration} to be a triple of disjoint sets, $\{T^+,T^-,T^E\}$ such that each is a subset of the elements preceding $\lambda$ in $\pi$, and the union is the entire set of elements preceding $\lambda$ in $\pi$. $T^+$ is the set of positively classified elements (excluding those in the end condition), $T^-$ is the set of negatively classified elements, and $T^E$ is the set of elements positively classified as part of the end condition. We say that a configuration is \emph{valid} for $\lambda$ if there is at least one remaining slot available, that is $|T^+\cup T^E|<n$. 

Denote all of the possible valid triples of elements by $\{(T_i^+,T_i^-,T_i^E)\}_{i \in [\xi]}$ where $\xi$ is the number of \emph{valid} triples. 
Let $\T^+$ be the collection of sets $\{T_i^+\}_{i \in [\xi]}$, and define $\T^-$ and $\T^E$ analogously.
Then $\T^+\cup \T^E$ and $\T^-$ are the sets of sets of included and excluded elements, so that $T^+_i \cup T^E_i\cup T^-_i$ specifies fully which of the elements before position $\lambda$ are included in $M$ in the $i^{th}$ configuration, and $\T^+$, $\T^E$, $\T^-$ contain all valid configurations of elements that ensure at least $\lambda - n$ elements are not included - that is, that there is at least one slot left for the element at $\lambda$. Notice that $T^E$ may be empty.  
We can now express the probability that $\lambda$ is reached as sum of probabilities over all possible configurations.
\[\Pr[\lambda \text{ reached}| \pi] = \sum_{i \in [\xi]}\prod_{x \in T^+_i}p_x\prod_{y \in T^-_i}(1-p_y)\prod_{z \in T^E_i}1\]
where, as before, we denote the probability that $C(w)=1$ as $p_w$ for all $w \in U$ for easier reading, and the probability is over the randomness of the classifier, as the permutation is fixed.

For a given permutation $\pi$, there are two possibilities - either $\lambda<\mu$ or $\lambda >\mu$. We bound the difference in probability of selection for $u$ and $v$ in each of these cases, and then use these bounds to conclude that the overall difference is not too large.

\textbf{Case 1:} If $\lambda< \mu$, then the probability of $\lambda$ being reached is completely independent of the outcome of the element at $\mu$. Consider the permutation $\pi_{\lambda,\mu}$ which is identical to $\pi$, except that the elements at positions $\lambda$ and $\mu$ are switched. 
Notice that if $\lambda$ is in the end condition, then the probability of $\lambda$ being selected is 1 in $\pi$, and the probability of $\mu$ being selected in $\pi_{\lambda,\mu}$ is also 1. Thus we have 
\[
  \Pr[u \in M | \pi] - \Pr[v \in M | \pi_{\lambda,\mu}] \leq
  (p_u - p_v)* \sum_{i\in [\xi]}\prod_{x \in T^+_i}p_x\prod_{y \in T^-_i}(1-p_y)   
\]
Define $\tau^*=\sum_{i \in [\xi]}\prod_{x \in T^+_i}p_x\prod_{y \in T^-_i}(1-p_y)\leq 1$.
Notice that $|\tau^*(p_u-p_v)| \leq |p_u-p_v|$.

\textbf{Case 2:}
When $\mu<\lambda$, we need a bit more analysis. Consider again the probability that $\lambda$ is reached, and now write it in terms of how the element at position $\mu$ is classified. For simplicity, we abuse notation and use $\mu$ to denote the element at location $\mu$.

\begin{align*}\Pr[\lambda \text{ reached}| \pi] =
 \sum_{\substack{i \in [\xi] \\ \mu \in T^+_i}}\prod_{x \in T^+_i}p_x\prod_{y \in T^-_i}(1-p_y)\prod_{z \in T^E_i}(1) 
  &+ \sum_{\substack{i \in [\xi] \\ \mu \in T^-_i}}\prod_{x \in T^+_i}p_x\prod_{y \in T^-_i}(1-p_y)\prod_{z \in T^E_i}(1)
\\ & + \sum_{\substack{i \in [\xi] \\ \mu \in T^E_i}}\prod_{x \in T^+_i}p_x\prod_{y \in T^-_i}(1-p_y)\prod_{z \in T^E_i}(1)
\end{align*}

Next, we pull out the portion of the products related to index $\mu$. 

\begin{align*}\Pr[\lambda \text{ reached}| \pi] = 
  p_\mu\sum_{\substack{i \in [\xi] \\ \mu \in T^+_i}}\prod_{x \in T^+_i\backslash \{ \mu\}}p_x \prod_{y \in T^-_i}(1-p_y) 
 &+ (1-p_\mu)\sum_{\substack{i \in [\xi] \\ \mu \in T^-_i}}\prod_{x \in T^+_i}p_x\prod_{y \in T^-_i\backslash \{\mu\}}(1-p_y)
\\ & + 1*\sum_{\substack{i \in [\xi] \\ \mu \in T^E_i}}\prod_{x \in T^+_i}p_x\prod_{y \in T^-_i}(1-p_y)
\end{align*}

Now consider the probability  the element at $\lambda$, is selected given that $\lambda$ is reached. If $\mu \in T^E$ then since $\lambda$ comes after, then $\lambda$ is selected with probability 1, as $\lambda$ must also be in the end condition. If $\mu \notin T^E$, then the probability of selecting $\lambda$ is either $1$ or $p_\lambda$ depending on whether the end condition is triggered by the time $\lambda$ is reached. Each configuration  $(T^+_i,T^-_i,T^E_i)$ specifies whether the end condition is reached by the time $\lambda$ is encountered, as it specifies the entire state of selections up to $\lambda$. Denote the indices of configurations which result in $\lambda$ in the end condition as $E$. That is, $E=\{i | \lambda \in T_i^E\}$.

Now, we can adapt the equations above to reflect the probability that $\lambda$ is selected given $\pi$:

\begin{align*}\Pr[\lambda \in M| \pi] = 
  p_\lambda p_\mu\sum_{\substack{i\in [\xi]\backslash E \\ \mu \in T^+_i}}\prod_{\substack{x \in  \\ T^+_i\backslash \{ \mu\}}}p_x \prod_{y \in T^-_i}(1-p_y) 
 &+ p_\mu\sum_{\substack{i\in E \\ \mu \in T^+_i}}\prod_{x \in T^+_i\backslash \{ \mu\}}p_x \prod_{y \in T^-_i}(1-p_y) 
\\  &+ p_\lambda(1-p_\mu)\sum_{\substack{i\in [\xi]\backslash E \\ \mu \in T^-_i}}\prod_{x \in T^+_i}p_x\prod_{y \in T^-_i\backslash \{\mu\}}(1-p_y)
\\  &+ (1-p_\mu)\sum_{\substack{i\in E \\ \mu \in T^-_i}}\prod_{x \in T^+_i}p_x\prod_{y \in T^-_i\backslash \{\mu\}}(1-p_y)
\\ &+ 1*\sum_{\substack{i\in [\xi] \\ \mu \in T^E_i}}\prod_{x \in T^+_i}p_x\prod_{y \in T^-_i}(1-p_y)
\end{align*}

Notice that for $\pi$ and $\pi_{\lambda,\mu}$, the sums of products exclusive of $p_\lambda$ and $p_\mu$ above are identical. For simplicity, define 
\begin{align*}
  \tau_1 &= \sum_{\substack{i\in [\xi]\backslash E \\ \mu \in T^+_i}}\prod_{x \in T^+_i\backslash \{ \mu\}}p_x \prod_{y \in T^-_i}(1-p_y) 
\\  \tau_2 &= \sum_{\substack{i\in E \\ \mu \in T^+_i}}\prod_{x \in T^+_i\backslash \{ \mu\}}p_x \prod_{y \in T^-_i}(1-p_y) 
\\  \tau_3 &= \sum_{\substack{i\in [\xi]\backslash E \\ \mu \in T^-_i}}\prod_{x \in T^+_i}p_x\prod_{y \in T^-_i\backslash \{\mu\}}(1-p_y)
\\  \tau_4 &= \sum_{\substack{i\in E \\ \mu \in T^-_i}}\prod_{x \in T^+_i}p_x\prod_{y \in T^-_i\backslash \{\mu\}}(1-p_y)
\\  \tau_5 &= 1*\sum_{\substack{i\in [\xi] \\ \mu \in T^E_i}}\prod_{x \in T^+_i}p_x\prod_{y \in T^-_i}(1-p_y)
\end{align*}

$\sum_{i \in [5]}\tau_i $ is equivalent to the probability that all elements before $\lambda$ excluding $\mu$ take on any of the valid configurations, eg, those configurations that lead to at least one slot being left by the time $\lambda$ is reached. Therefore $\sum_{i \in [5]}\tau_i \leq 1$. 
We can therefore rewrite more simply and substitute back in our original $u,v$ as 
\begin{align*}\Pr[u \in M| \pi] = 
  p_u p_v\tau_1 
 + p_v\tau_2
 + p_u(1-p_v)\tau_3
 + (1-p_v)\tau_4
 + 1*\tau_5
\end{align*}

Now consider the difference between the probability that $u$ is selected under $\pi$, and the probability that $v$ is selected under $\pi_{\lambda,\mu}$:

\begin{align*}\Pr[u \in M| \pi] - \Pr[v \in M| \pi_{\lambda,\mu}]   = \text{ }
 & p_u p_v\tau_1 
 + p_v\tau_2
 + p_u(1-p_v)\tau_3
 + (1-p_v)\tau_4
+ 1*\tau_5
 \\ &- ( p_u p_v\tau_1 
 + p_u\tau_2
 + p_v(1-p_u)\tau_3
 + (1-p_u)\tau_4
 + 1*\tau_5)
\end{align*}
and so
\[|\Pr[u \in M| \pi] - \Pr[v \in M| \pi_{\lambda,\mu}]|  = 
 | (p_v-p_u)\tau_2
 + (p_u-p_v)\tau_3
 + (p_u-p_v)\tau_4 | \leq |p_u-p_v|,
\]
where the last inequality follows from the sum of the $\tau$'s representing disjoint cases, yielding the desired bound on the distance.

Now we combine Cases 1 and 2 to reach our desired conclusion:
 the difference in probability that $u \in M$ and $v \in M$ is the sum of the difference in each permutation multiplied by the probability of each permutation being selected.
More formally, denote the set of all permutations on $[|U|]$ as $\Pi$:
\[\Pr[x \in M] = \sum_{\pi \in \Pi}\Pr[x \in M | \pi]\Pr[\pi]\]
\[\Pr[x \in M] - \Pr[y \in M] = \sum_{\pi \in \Pi}\Pr[\pi]\Pr[x \in M | \pi] -\sum_{\pi \in \Pi}\Pr[\pi]\Pr[y \in M | \pi]  \]
Notice that for each $\pi$, there is exactly one $\pi_{\lambda,\mu}$, so we can combine the sums:
\[\Pr[x \in M] - \Pr[y \in M] = \frac{1}{|\Pi|}\sum_{\pi \in \Pi}(\Pr[x \in M | \pi] - \Pr[y \in M | \pi_{\lambda,\mu}]) \]
Finally, using our bounds from Cases 1 and 2, we conclude
\[|\Pr[x \in M] - \Pr[y \in M]|\leq \frac{1}{|\Pi|}\sum_{\pi \in \Pi}|p_u-p_v| = |p_u-p_v|\]
\end{proof}

Although \permutethenclassify satisfies fairness, and is simple to implement, depending on how well the classifier $C$ has been adjusted for the number of elements to be selected versus the universe size, it may perform sub-optimally. For example, if $C$ was tuned to select only $O(\log(n))$ elements in expectation under normal independent classification, but ends up being used to select $O(n)$ elements with permute and classify, then there may be an excessive number of elements ($\approx O(n-\log(n))$) chosen arbitrarily in the end condition. 
We now propose a second mechanism, Weighted Sampling, to address this shortcoming.

\begin{algorithm}[tb]
   \caption{\weightedsampling}
   \label{alg:weightedsampling}
\begin{algorithmic}
   \STATE {\bfseries Input:} $n \leftarrow $ the number of elements to select\\
$C \leftarrow$ a classifier $C: U\times r \rightarrow  \{0,1\}$ \\
$L \leftarrow$ the set of all subsets of $U$ of size $n$\\

\FOR { $l\in L$}  
\STATE $w(l) \leftarrow \sum_{u \in l}\E[C(u)]$ // set the weight of each set
\STATE Define $\mathcal{X} \in \Delta(L)$ such that $\forall l \in L$, the weight of $l$ under $\mathcal{X}$ is $\frac{w(l)}{\sum_{l' \in L}w(l')}$\\
\STATE $M \sim \mathcal{X}$ // Sample a set of size $n$ according to $\mathcal{X}$\\
\ENDFOR
\STATE return $M$

\end{algorithmic}
\end{algorithm}

\begin{theorem}\label{thm:fairweightedsampling}For any individually fair classifier $C$ such that the $\Pr_{u \sim U, r \sim \{0,1\}^*}[C(u,r)=1] \geq 1/|U|$, weighted sampling is individually fair.
\end{theorem}
\begin{proof}
Fix an arbitrary individually fair classifier $C$ and two elements $u\neq v$ from $U$. The difference between the probability of $u$ and $v$ being included in $M$ under weighted sampling is 
\[d(u,v) = |\Pr_{M \sim \mathcal{X}}[u \in M] - \Pr_{M \sim \mathcal{X}}[v \in M]|\]
where probability is taken over the randomness of the weighted sampling mechanism.

Denote the set of subsets of $U$ of size $n$ in which both $u$ and $v$ are present as $\T^{u,v}$ and the subsets of $U$ of size $n$ in which exactly one of $u$ or $v$ is present as $\T^u$ and $\T^v$, respectively. Then we can write $\Pr_{M \sim \mathcal{X}}[u \in M] = \Pr_{M \sim \mathcal{X}}[M\in \T^{u,v}] + \Pr_{M \sim \mathcal{X}}[M \in \T^u]$, and likewise for $v$. So we can rewrite our difference as:
\[d(u,v) = |\Pr_{M \sim \mathcal{X}}[M \in \T^u] + \Pr_{M \sim \mathcal{X}}[M \in \T^{u,v}] - \Pr[M \in \T^v] - \Pr_{M \sim \mathcal{X}}[M \in \T^{u,v}]|,\] whence
\[d(u,v) = |\Pr_{M \sim \mathcal{X}}[M \in \T^u]   - \Pr_{M \sim \mathcal{X}}[M \in \T^v]  |\]

As expected, to reason about the distance we need only consider the sets where exactly one element appears. Consider the elements of $\T^u$ and $\T^v$. For every set $T_i^u \in \T^u$, there is a corresponding set $T_i^v \in \T^v$ which replaces $u$ with $v$, that is $\{T_i^u\backslash\{u\}\}\cup \{v\} = T_i^v$. Notice that there are no sets in $\T^v$ which cannot be formed in this way from $\T^u$ and vice versa.
So we can further split these into sums:

\[d(u,v) = |\sum_{T_i^u \in \T^u}\Pr[M = T_i^u]  -  \sum_{T_i^v \in \T^v}\Pr[M = T_i^v]  |\]

\[d(u,v) = |\sum_{T_i^u \in \T^u}\Pr[M = T_i^u]  - \Pr[M = \{T_i^u\backslash u \} \cup \{v\}]  |\]

For convenience, let $\eta$ denote the normalization factor $\sum_{j \in [|L|]}w(T_j)$. Recall that $w(u) = p_u$ $\forall u \in U$ using our previous notation, so we can simplify the above to

\[d(u,v) = |\sum_{T_i^u \in \T^u}\frac{1}{\eta}(p_u + \sum_{w \in T_i^u \backslash \{u\}}p_{w} ) - \frac{1}{\eta}(p_v + \sum_{w \in T_i^u \backslash \{u\}}p_w )  |\]
\[d(u,v) = |\frac{1}{\eta}\sum_{T_i^u \in \T^u}(p_u - p_v) |\]

\[d(u,v) = |\frac{|\T^u|}{\eta}(p_u - p_v) |\]

So as long as $|\T^u|/\eta \leq c$, we have a $c-$Lipschitz condition on the mechanism.

    \[|\T^u|/\eta = {|U|-2 \choose n-1}*\frac{1}{\eta}\leq c\]
    \[{|U|-2 \choose n-1} \leq c\eta\]
    \[{|U|-2 \choose n-1} \leq c\sum_{j \in [|L|]}w(T_j)\]

Recall that $L$ is the set of all sets of size $n$, so taking $\bar{w}$ as the average weight of the sets in $L$, we have
\[{|U|-2 \choose n-1} \leq {|U| \choose n }c\bar{w}\]
\[{|U|-2 \choose n-1}/ {|U| \choose n } \leq c\bar{w}\]

Expanding out the binomial coefficients, we have
\[\frac{(|U|-2)!}{(n-1)!(|U|-n-1)!}/\frac{|U|!}{n!(|U|-n)!} \leq c \bar{w}\]
\[\frac{(|U|-2)!}{|U|!} \frac{n!}{(n-1)!}\frac{(|U|-n)!}{(|U|-n-1)!} \leq c \bar{w}\]
\[\frac{|U|-n}{|U|-1} \frac{n}{|U|} \leq c \bar{w}\]
So for any $n\geq 1$,
\[\frac{n}{|U|} \leq c \bar{w}\]

So as long as the average weight of a set of size $n$ is larger than $n/|U|$, then the desired 1-Lipschitz condition is maintained. 
\end{proof}

There is a simple intuition for the requirement for the average set weight. Imagine there was a single element, $u$, with positive weight in a universe of $1000n$ elements. The sets including $u$ are the only sets with positive weight, and as such, $u$ is guaranteed to be selected, even if $u$'s original weight $p_u$ is negligible. This guaranteed selection can pull $u$ too far from its neighbors, who all have 0 probability of selection. To avoid this it suffices to have that there is enough weight across all of the elements to fill sets on average.

Comparing Weighted Sampling with \permutethenclassifynospace, Weighted Sampling does have an additional constraint on the fair classifier used with respect to average set weight, but in practice this is unlikely to be difficult to achieve. With respect to utility, if we assume a simple linear utility function (eg, the utility of an element is equivalent to $p_u*\alpha$) for some fixed constant $\alpha$, we conclude that the utility of Weighted Sampling is likely to exceed the utility of \permutethenclassifynospace. This  follows simply from the observation that the probability of selection for any cohort in Weighted Sampling is proportional to its weight, whereas the probability of selection for any cohort in \permutethenclassify is proportional to the weight of the elements selected outside of the end condition. However, with either mechanism, we can use an existing fair classifier to achieve a fair and nontrivial utility outcome for the fair cohort selection problem. With respect to computational complexity, \weightedsampling is expensive, as one must compute weights and sample from the ${|U| \choose n}$ possible subsets. However, in practice this may be alleviated by using a fair classifier which weights many ``irellevant'' elements 0, thus reducing the number of possible sets. 

\subsubsection{Online Cohort Selection}\label{section:stream}
Now that we have seen fair solutions for the offline cohort selection problem, we consider the online version of the problem. In the online version, $\mathcal{S}$ must respond immediately to each element encountered, so intuitively the choice of ordering is much more important.

 \begin{definition}[Online Cohort Selection Problem]\label{def:onlinecohort}
 A system $\mathcal{S}$ is a solution to the Online Cohort Selection Problem if it classifies the $i^{th}$ element before being given access to the $i+1^{st}$ and it solves the Cohort Selection Problem. \end{definition}

Having seen \permutethenclassifynospace, it's easy to see that if the ordering of the stream $\pi \sim \mathcal{X}$ is chosen uniformly at random from all permutations over the universe and the size of the universe is known, then there is a solution to Online Cohort Selection Problem.

\begin{theorem}\label{thm:randomlyorderedstream} If the ordering of the stream $\pi$ is drawn uniformly at random from the permutations over the elements of $U$, $S_{|U|}$, and the length of the stream is known, then if there exists a fair classifier for the task, there exists a solution to the online cohort selection problem.
\end{theorem}
\begin{proof}
Simulate Algorithm \ref{alg:permutethenclassify}, omitting the initial permutation step, with the fair classifier.
\end{proof}

However, if the ordering is adversarial, there may be no fair solutions, or the fair solutions may have trivial utility. In our setting, adversarial ordering is captured by the distribution $\mathcal{X(U)}$ which may place certain elements earlier or later in the orderings with high probability. For example, an adversarial $\mathcal{X(U)}$ may have the probability that $u$ is placed before $v$ greater than $\frac{3}{4}$, in the hope of giving $u$ a higher chance of being selected than $v$. 

\begin{theorem}\label{thm:adversarialknownlength}
If the ordering of the stream is adversarial, and the stream contains all elements of $U$, and $|U|$ is known, there  exists a solution to the strict stream cohort selection problem.
\end{theorem}
\begin{proof}
Sample uniformly at random  from the set of all strings of length $|U|$ with weight $n$, $s^* \sim \{ s \in \{0,1\}^{|U|}, |s| = n\}$ and select the elements at the positive coordinates of $s^*$.
\end{proof}
Although this solution is fair, the utility is clearly no better than choosing randomly. 

\begin{theorem}\label{thm:adversarialunknownlength}
If the ordering of the stream is adversarial, but $|U|$ is unknown, then there exists no solution to the online cohort selection problem.
\end{theorem}
\begin{proof}

Consider $V \subseteq U$ such that $|V|=n$. Choose $\pi$ a permutation on $n$ elements. Then each element has probability 1 of selection by assumption that $\M$ is a solution to the Cohort Selection Problem. Fix $w \in U$ such that $\D(w,u) < 1$ for some $u \in V$,  and $w \notin V$. Consider the ordering of $V \cup \{w\}$ which orders the elements of $V$ using $\pi$ and places $w$ in the last position. However, $w$ has zero probability of selection, because $\mathcal{M}$ always selects the first $n$ elements, so $\mathcal{M}$ cannot be individually fair, as $d(u,w)=1>\D(u,w)$.  
\end{proof}

It may be tempting to consider ``fixing'' this impossibility by only requiring that our system select $n$ individuals with high probability, allowing for some failure on small universes. However, extending this for many possible sizes of universe is non-trivial, and the ``fix'' breaks down.

\subsubsection{Constrained Cohort Selection}\label{section:constrainedcohort}
Next we consider the problem of selecting a cohort with an external requirement that some fraction of the selected set is from a particular subgroup.

\begin{definition}[The Constrained Cohort Selection Problem] 
Given a universe $U$, $p \in [0,1]$, a subset $A \subset U$, and a metric for the task $\D$, solve the cohort selection problem with the added requirement that at least a $p$ fraction of the members of the selected cohort are in $A$.
\end{definition}

This problem captures situations in which external requirements cannot be ignored. For example, if a certain budget must be met, and only some members of the universe contribute to the budget, or if legally a certain fraction of people selected must meet some criterion (as in, demographic parity).

Before we consider the more difficult problem of satisfying individual fairness, note that to satisfy intra-group Fairness, that is, $d(u,v) \leq \D(u,v)$ for all $u,v \in A$ and for all $u,v\in \{U \backslash A\}$, one straightforward method would be to run \permutethenclassify on each group separately with $n_A = np$ and $n_{B}=n-np$. (For notational convenience, we henceforth write $ U \backslash A = B$). In some settings, this solution may be better than imposing no fairness constraint at all even though it is not truly individually fair. 
However, satisfying universal individual fairness is a far more difficult task, and for non-trivial constraints and universes may be impossible.

To understand the cases where constrained cohort selection is impossible, we first introduce the notion of $\gamma-$equivalence, which will allow us to describe sufficiently interesting distance relations across the subgroups.

\begin{figure}
\vskip 0.2in
\begin{center}
\centerline{\includegraphics[width=40mm]{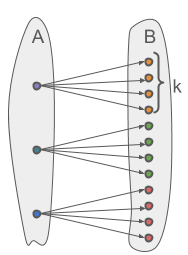}}
\caption{$B$ is $\gamma-$equivalent to $A$ if for every set of nodes matched in $B$, $\D(a,b_i)\leq \gamma$.}
\label{fig:gamma_equivalence}
\end{center}
\vskip -0.2in
\end{figure}

\begin{definition}[$\gamma-$equivalence]
Given a metric $\D$ over $U$ and $A,B \subseteq U$ such that $A \cap B = \emptyset$, $B$ is said to be $\gamma-$ equivalent to $A$ if and only if
there exists a bipartite graph $(A,B,E)$ with all elements of $A$ represented by nodes with out-degree $k$ and all elements of $B$ represented by nodes with in-degree $1$ such that the neighborhood of each $x \in A$ contains only elements  $y_i \in B$ such that $\D(x,y_i)\leq \gamma$.
\end{definition}

We will find $\gamma-$equivalence to be a useful rough approximation of how similar two groups are. We start with a simple warm-up lemma to describe how dissimilarly $A$ and $B$ may be treated if $B$ is $\gamma-$equivalent to $A$. 
From the constraint, we have that $pn/|A|$ is a lower bound for the average acceptance rate of $A$ and $(1-p)n/|B|$ is an upper bound on the average acceptance rate for $B$. We now this to $\gamma-$equivalence in the following observation and lemma.

\begin{observation}\label{lemma:gammaequivalentimpossibility}
The constrained cohort selection problem for a universe $U$ with groups $A$ and $B$ such that $B$ is $\gamma$-equivalent to $A$, where
$\gamma < |\frac{pn}{|A|} - \frac{(1-p)n}{|B|}|$, has no  individually fair solution.
\end{observation}
Observation \ref{lemma:gammaequivalentimpossibility} is a special case of Lemma \ref{lemma:betagammaimpossibility}, proved next. Roughly speaking, Observation \ref{lemma:gammaequivalentimpossibility} captures the intuition that if the distribution of talent between two groups is equivalent, but the acceptance rates for each group are very different, then the system cannot be individually fair.

To build intuition, suppose the universe $U$ is partitioned into sets $A$ and $B$, where $n/2=|A| = |B|/5$. Suppose further that the populations have the same distribution on ability, so that the set $B$ is a ``blown up" version of $A$, meaning that for each element $u \in A$ there are 5 corresponding elements $V_u = \{v_{u,1},...,v_{u,5}\}$ such that $\mathcal{D}(u,v_{u,i})= 0$, $1 \le i \le 5$, $\forall u,u' \in A~ V_u \cap V_{u'} = \emptyset$, and $B = \cup_{u \in A} V_u$. Let $p = \frac{1}{2}$.  The constraint requires all of $A$ to be selected; that is, each element of $A$ has probability 1 of selection; in contrast, the average probability of selection for an element of $B$ is $\frac{1}{5}$. Therefore, there exists $v\in B$ with selection probability at most $1/5$.  Letting $u\in A$ such that $v \in V_u$, $\mathcal{D}(u,v)=0$ but the difference in probability of selection is at least $\frac{4}{5}$. 

To give ourselves a tighter characterization, we prove the following theorem which uses $\gamma-$ equivalence of subgroups of $B$ to achieve a tighter bound.
\begin{theorem}\label{lemma:betagammaimpossibility}
The constrained cohort selection problem for a universe $U$ with groups $A$ and $B=U \backslash A$ such that there is a partitioning of $B= \{B_1,\ldots,B_t\}$ such that each subset $B_i$ is
$\gamma_i$-equivalent to $A$, and where  
$(1-p)n/|B| < pn/|A| + \sum_{i \in [t]}\beta_i\gamma_i$
for $\beta_i = |B_i|/|B|$ has no individually fair solution.
\end{theorem}
\begin{proof} 
Assume for the sake of contradiction that there exists a mechanism $\M$ which satisfies individual fairness for the constrained cohort selection problem instance above. The average probability of acceptance for each group can be written
\[\mu_A = \frac{1}{|A|}\sum_{x \in A}p_x \geq \frac{pn}{|A|}\]
\[\mu_B = \frac{1}{|B|}\sum_{y \in B}p_y \leq \frac{(1-p)n}{|B|}\]
where $p_w$ for $w \in U$ denotes the probability that $w$ is accepted by $\mathcal{M}$. The inequalities arise from the constraint that at least a $p-$fraction of the elements chosen must be from $A$. 
Now consider the subset $B_i$ of $B$, which contains the elements which are $\gamma_i-$equivalent to $A$. 
Fix an arbitrary $i\in [t]$. Let $G_i=(A,B_i,E_i)$ denothe the bipartite graph whose existence is given by the definition of $\gamma_i-$equivalence of $B_i$ to $A$. For all $u \in A$, let $\Gamma_i(u) \subseteq B_i$ denote the neighbors of $u$ in $G_i$.

We can write 
\[\mu_{B_i} = \frac{1}{k_i|A|}\sum_{x \in A}\sum_{y \in \Gamma_i(x)}p_y \]
Where $k_i$ corresponds to the number of elements in $B_i$ which are mapped to by each element of $A$, that is $k_i|A|=|B_i|$. 
Because $\M$ is assumed to be individually fair, we know that $p_y \in p_x \pm \gamma_i$ for the $x$ such that $ y \in \Gamma_i(x)$. So for each $x\in A$, define $r_{y,x}= -p_x + p_y$ for each $y \in \Gamma_i(x)$. So we have
\[\mu_{B_i}=\frac{1}{k_i|A|}\sum_{x \in A}\sum_{y \in \Gamma_i(x)}p_y= \frac{1}{k_i|A|}\sum_{x \in A}\sum_{y \in \Gamma_i(x)}p_x+r_{y,x}\]
\[= \frac{1}{k_i|A|}\sum_{x \in A}(k_ip_x+\sum_{y \in \Gamma_i(x)}r_{y,x}) = \mu_A + \frac{1}{k_i|A|}\sum_{x \in A}\sum_{y \in \Gamma_i(x)}r_{y,x}\]
By construction,$r_{y,x} \in \pm \gamma_i$, so we can bound the final sum:
\[\mu_{B_i}\in \mu_A \pm \frac{k_i|A|}{k_i|A|}\gamma_i = \mu_A \pm \gamma_i\]

Now, consider how we would write $\mu_B$ as a weighted sum of $\mu_{B_i}$:
\[\mu_B = \sum_{i \in [t]}\beta_i\mu_{B_i} \]
\[\mu_B \in \sum_{i \in [t]}\beta_i(\mu_A \pm \gamma_i) \]
\[\mu_B \in \mu_A + \sum_{i \in [t]}\beta_i(\pm \gamma_i) \]
\[\mu_B \in [\mu_A - \sum_{i \in [t]}\beta_i\gamma_i, \text{ } \mu_A +\sum_{i \in [t]}\beta_i\gamma_i] \]
So the difference $\mu_A-\mu_B$ is at most $\sum_{i \in [t]}\beta_i\gamma_i$. However, by assumption $\mu_A - \mu_B > \sum_{i \in [t]}\beta_i\gamma_i$, yielding a contradiction.

\end{proof}

Notice that if there is a single subset in the partition, Theorem \ref{lemma:betagammaimpossibility} is equivalent to Lemma \ref{lemma:gammaequivalentimpossibility}. Essentially the Theorem states that if there is a significant fraction of the group $B$ which is $\gamma-$close to $A$, then the average difference in probability cannot be too much larger than $\gamma$. Although the characterization isn't completely tight, it can still be useful for building our intuition. Rearranging the terms above, we have that $p$, the fraction of those selected that must be in $A$, cannot exceed $ \frac{n + |B|\sum_{i \in [t]}\beta_i\gamma_i}{(|B|/|A|+1)n}$.

For example, imagine that 10\% of students applying to a university who cannot pay (B) have distributions which are $0-$equivalent to students who can pay (A). Another 50\% of students in $B$ are $.1-$equivalent to A, and the remaining 40\% of students in $B$ are $.25-$equivalent to A. These distributions are quite different, but even so the difference in average acceptance rate can be at most $.4*.25+.5*.1+.1*0=.15$. If $|B|=1000$ and $|A|=100$, and we wish to select $n=550$ students, then the fraction of students required to be from $A$ cannot exceed $ \frac{550 + 1000(.15)}{(1000/100+1)550} \approx .11$. Any university that required 25\% of students to be able to pay in order to meet their budgetary constraints, under these distributions and relative group sizes, would see from Theorem \ref{lemma:betagammaimpossibility} that there is no individually fair solution for their cohort selection problem.

For practical use, it may be necessary to expand $A$ and $B$ by duplicating all elements in order to achieve whole number mappings, but that it does not impact the logic of the proof and the bounds.

\subsubsection{Universe Subsets}
There is an important distinction to be made between when the classifier has the ability to assign outcomes to the entire universe, or only to a subset. 
Initially, Definitions \ref{def:partitionclass}, \ref{def:basiccohort} and \ref{def:onlinecohort} may seem impossible to satisfy when $\mathcal{Y}$, the distribution on subsets, is non-trivial. As the classifier can only act on $V$, there may be unfairness or allocations in $U \backslash V$ that cannot be remedied or matched by any action taken by the classifier under consideration. In particular, if any elements are completely missing from $V$, that is, there is some element $w \in U$ which is contained in no subsets with positive weight under $\mathcal{Y}$, then fair solutions may be difficult to achieve. For example, if one school district can afford to provide musical instruments and teachers for an after-school orchestra, but the other cannot, then if students aren't allowed to transfer between school districts it will be difficult to ensure individual fairness without some (potentially unrealistically expensive) intervention in the second school district. 

\begin{proposition}\label{prop:unavailableoutcomes}(Informal) If the elements of $U\backslash V$ are mapped to outcomes unavailable to $C$ or their outcomes are unknown, then no choice of $C$ is guaranteed to solve the Universe Subset Classification Problem for nontrivial distributions $\mathcal{Y}$ on subsets of $U$.
\end{proposition}

The proposition follows from the simple observation that if elements of $U\backslash V$ are mapped to unreachable outcomes (for example, a resource which $C$ cannot provide for a particular task), then there is no distribution over outcomes $C$ can utilize to satisfy similar treatment of similar individuals if $\mathcal{Y}$ maps some elements to $V$ with higher probability than others.

We now  show that there are some, admittedly limited, cases where the classifier \textit{can} still ensure individual fairness for the whole universe. 
Before we describe these settings, we introduce a weaker notion of fairness, Subset Individual Fairness, which we will use to reason about how to behave fairly when the rest of the system is reasonably well behaved on the subset of the universe on which it operates. 

\begin{definition}[Subset Individual Fairness] 
Given a subset $V \subseteq U$, a task and a metric $\D$, a possibly randomized system $\mathcal{S}:V \times \{0,1\}^* \rightarrow \{0,1\}$ is \textit{Subset Individually Fair} on $V$ if for all $u,v \in V$ the distribution over outcomes is 1-Lipschitz wrt $\D$, that is,

\[\E_{\pi \sim \mathcal{X}(V)}|\E_{r}[\mathcal{S}(u,r)] - \E_{r}[\mathcal{S}(v,r)]\,| \leq \D(u,v)\] for all $u,v \in V$.
\end{definition}

It is useful to have a notion of Subset Individual Fairness, as there are some scenarios where components that satisfy Subset Individual Fairness may be easier to compose into fair systems. Indeed, when we consider the system version of Lemma \ref{lemma:fairadditions}, Subset Individual Fairness suffices to allow fair classification of elements in the rest of the universe. 

\paragraph{Comparable outcomes}\label{section:comparableoutcomes}
Consider the problem of assigning high school students to public schools. Some fraction of the universe of potential high school students will be diverted to private schools and may have zero probability of attending public school. However, our goal is still to ensure individual fairness for the whole universe of high school students, not just those attending public school. The issue is treating those students in the public schools similarly to those in the private schools. This scenario is more challenging, as the classifier under control of the public school system is not the sole point of determination for outcomes for the entire universe -- it only controls the outcomes for students attending public schools. 

Imagine that among private schools, there are schools which focus on the humanities and general education, and schools which focus on science, and the assignment procedure between these schools is fair with respect to students' talents in these subjects. We assume the metric captures students preferences over science focus versus general education. 
It is possible, by similarly specializing public schools, for the school district to assign students to public schools in a way that is fair with respect to the entire universe of students.

Lemma \ref{lemma:sameoutcomes} below states that if the behavior of the rest of the system is subset individually fair (eg, the private schools fairly assign students to science or humanities schools), then simply copying that behavior for elements in the subset acted upon by the classifier in question (eg, the public schools use the same logic as the private schools to assign students) will be individually fair.

\begin{lemma}\label{lemma:sameoutcomes}
Consider a subset $V \sim \mathcal{Y}$ and a binary classifier $C^*$ which operates on $U\backslash V$. If $C^*$ is subset individually fair on $U\backslash V$, and if the outcomes of $C^*$ are in the range of the classifier operating on $V$, then there exists a classifier $C$ that is individually fair on $U$.
\end{lemma}
\begin{proof}
Take $C(w)=C^*(w)$ for all elements for which $C^*$ is defined $(U \backslash V)$. For any element $v$ for which $C^*$ is not defined, choose any distribution over outcomes which satisfies $\D(v,w)\geq d(v,w)$ for all other $w \in U$ with currently defined outputs. As $C^*$ is subset individually fair, valid distributions over outcomes can be found per Lemma \ref{lemma:fairadditions}.
\end{proof}

Extending Lemma \ref{lemma:sameoutcomes} to also work for Cohort Selection requires that the cohort size be adjustable depending on $\mathcal{Y}$ and $C^*$, which may not be possible in practice. Returning to our public versus private school example, if the public and private schools both have equal distributions of science versus humanities talent, and proportional enrollment capacity in specialized schools, then a universally individually fair solution can be attained in the absence of other contstraints. However, as we saw in Section \ref{section:constrainedcohort}, this problem reduces to constrained cohort selection if the enrollment capacity is artificially tilted towards one schooling track over the other. Furthermore, there may be challenges related to quantization which (eg, if the number of students is not evenly divisible among appropriate campus enrollment sizes).

\paragraph{All elements have positive weight in $\mathcal{Y}$}\label{section:positiveweights}
In the next scenario, we show that if there is always `leftover' probability, and the classifier in question is the only classifier which assigns outcomes for a particular task, then we can find an individually fair solution. 

Consider the case of a classifier that is solely responsible for assigning outcomes for a task for the entire universe. For example, students may be offered a choice to try out for the school's soccer team and orchestra which have practice at the same time. We assume that there is a strong social cost to quitting the soccer team, so any student who is accepted to the soccer team will not be eligible for the orchestra. (That is, the system is a task-competitive composition with a strict preference for soccer.) If soccer try-outs are first, but every student still has some positive probability of trying-out for the orchestra, that is, each student arrives to the orchestra `classifier' with some positive probability, then with sufficient understanding of the distribution $\mathcal{Y}$ with which students arrive to the orchestra `classifier', it may be able to appropriately respond to ensure fairness.

In this setting, it is important to understand the default behavior if the classifier for the task is \textit{not} applied. We assume that if the classifier for the task is not applied, then some default behavior is assigned. For example,  in the case of orchestra versus soccer, no student can join the soccer team unless they are classified positively by the soccer classifier. More complicated default behaviors could be imagined, but for our purposes, we just consider this simple case, which corresponds well to settings where a single entity controls outcomes for a particular task. 

Intuitively, when $\mathcal{Y}$ does have equal probability for each element in the universe, we might try to modify our behavior to simulate as if each element appeared with equal probability. Lemma \ref{lemma:positiveweights} below states that as long as each element has some positive weight under $\mathcal{Y}$, we can devise a procedure which simulates each element appearing with equal probability.
\begin{lemma}\label{lemma:positiveweights}
Given an instance of the universe subset classification problem (Definition \ref{def:partitionclass}) where $\mathcal{Y}$ assigns positive weight to all elements $w \in U$, the following procedure applied to any individually fair classifier $C$ which solely controls outcomes for a particular task will result in fair classification under the input distribution $\mathcal{Y}$.

{\setlength{\parindent}{0cm}\textbf{Procedure:} for each $w \in U$, let $q_w$ denote the probability that $w$ appears in $V$. Let $q_{min} = \min_w q_w$. 
For each element $w \in V$, with probability $q_{min}/q_w$ classify $w$ normally, otherwise output the default for no classification. }
\end{lemma}
\begin{proof}
Let $u = \argmin_w (q_w) $. Then  
$u$ will be classified positively with probability $p_uq_{min}$ where probability is taken over $\cal{Y}$ and $C$. All other elements $v \in V$ will be classified positively with probability $q_v(q_{min}/q_v)p_v = p_vq_{min}$. As positive classification by $C$ is the only way to get a positive outcome for the task, reasoning about $|p_v-p_u|$ is sufficient to ensure fairness.
Therefore, if $|p_v-p_u| \leq \mathcal{D}(u,v)$, then the distance under this procedure is also $\leq \mathcal{D}(u,v)$. 
\end{proof}

Imagine in the soccer versus orchestra example that the soccer team try-outs were first (and students were required to commit before orchestra try-outs). If spots on the soccer team are granted to each student with maximum probability $50\%$, then the orchestra classifier still has $50\%$ probability left over, even for the most talented soccer player. In such cases, the procedure in Lemma \ref{lemma:positiveweights} will suffice to make the orchestra classifier fair with respect to all elements in $U$, even with the interference of the earlier classifier. Of course, the allocation (the number of expected positive classifications) of the classifier may need to be tuned for improved performance, but fairness can be maintained.

A critical observation of the procedure in Lemma \ref{lemma:positiveweights} is that the behavior of $C$ on a distribution \textit{other than} $\mathcal{Y}$ may significantly violate individual fairness constraints. For example, if $\mathcal{Y}$ assigns different weights to $u,v$ where $\D(u,v)=0$, then $C$ will appear to `split' these equal elements on any distribution $\mathcal{Y}'$ where their weights \textit{are} equal.

\paragraph{Implications of Universe Subset Classification Problems}
There are two important implications of the settings described above.
First, Lemmas \ref{lemma:sameoutcomes} and \ref{lemma:positiveweights} give us the following theorem:
\begin{theorem}\label{theorem:nontrivialy}There exist solutions to the Universe Subset Classification Problem for non-trivial $\mathcal{Y}.$
\end{theorem}

We emphasize this point because these results imply that augmented classifiers or families of classifiers, which specify which $\mathcal{Y}$ and $\mathcal{X}$ they behave well on, can be used in these difficult settings in practice. However, determining $\mathcal{Y}$ and $\mathcal{X}$ precisely may be difficult in cases where components are controlled by potentially competitive or uncooperative parties and may have significant privacy implications. 

Second, classifiers which appear unfair in isolation may be fair under composition. In fact, we can say something even stronger: in cases where $\mathcal{Y}$ does not provide a uniform selection of elements from $U$ and includes similar individuals in the metric with significantly different probabilities, there exist classifiers which appear unfair in isolation, but are fair for the input distribution $\mathcal{Y}$. Thus any auditing process which doesn't take $\mathcal{Y}$ into account would potentially raise a false alarm, even though a classifier may have been explicitly constructed to function fairly under composition.

\subsection{Extensions to Group Fairness}\label{section:group}
The anecdotal examples we have visited throughout the preceding sections give us some intuition that meaningful solutions to composition problems are inherently difficult without coordination between classifiers or thoughtfully designed compositions. 
We now formally extend these results to group fairness definitions (and discuss several cases where they do not extend), following the generalized Conditional Parity notion of \cite{ritov2017conditional}, which captures popular group fairness definitions in the literature such as Equality of Opportunity and Equalized Odds \cite{hardt2016equality} and Counterfactual Fairness \cite{kusner2017counterfactual}.

Recall Definition \ref{def:conditionalparity} (Section \ref{section:groupdefs}), which states that a predictor $Y$ satisfies conditional parity with respect to a stratification set $\mathcal{Z}$ for protected attributes $\mathcal{A}$ if for all $a_1,a_2 \in \mathcal{A}$ and for all $z \in \mathcal{Z}$:
\[\Pr[Y=1 | \mathbf{a}=a_1, \mathbf{z}=z] = \Pr[Y=1 | \mathbf{a}=a_2, \mathbf{z}=z]\]
In this section, we say that a classifier ``satisfies group fairness'' or is ``group fair'' if it meets this requirement.

In order to draw the analogy from individual fairness composition results to conditional parity composition, we will show that there are many classifiers that satisfy conditional parity in isolation, but fail to satisfy conditional parity under composition.  
We will also show that in some cases, composition of classifiers which satisfy conditional parity may result in systems which nominally satisfy the fairness requirement, but have troubling behavior from a subgroup perspective, and alternatively may result in systems which do not nominally satisfy the fairness requirement, but would satisfy the even stricter notion of individual fairness. 

In many cases in the literature, group-based notions of fairness are used not because they capture group-based incentives or decision-making, but because they are more practical for implementation and measurement. To that end, concerns about the treatment of socially meaningful subgroups are addressed in several lines of work using more granular group requirements in order to provide more meaningful guarantees for sufficiently sized subgroups \cite{kearns2017gerrymandering, hebert2017calibration,kim2018fairness}.

\subsubsection{Functional Composition}
The results for individual fairness for functional composition largely extend to the group fairness setting, with a few caveats due to technicalities of the definition.
\paragraph{Same-Task Functional Composition}
As with individual fairness, we first consider same-task composition. Recall our college admissions example from Section \ref{section:sametask}. Consider a pair of classifiers $C$ and $C'$ which both satisfy conditional parity with respect to the same set of sensitive attributes $\A$ and stratification set $\mathcal{Z}$. That is
\[\Pr[C(u)=1 |\mathbf{a}=a_1,\mathbf{z}=z ]=\Pr[C(u)=1 |\mathbf{a}=a_2,\mathbf{z}=z ]\]
\[\Pr[C'(u)=1 |\mathbf{a}=a_1,\mathbf{z}=z ]=\Pr[C'(u)=1 |\mathbf{a}=a_2,\mathbf{z}=z ]\]
where the probability is taken over the choice of $u$ and the randomness of the classifier $C$, for all $a_1,a_2 \in \mathcal{A}$ and for all $z \in \mathcal{Z}$ for some appropriately chosen set of protected attributes $\mathcal{A}$ and appropriately chosen set of stratification attributes $\mathcal{Z}$. For example, $\mathcal{A}$ may be the set of genders and $\mathcal{Z}$ may be the set of intrinsic qualification levels for college. 

Now, imagine that we apply both $C$ and $C'$ to all members of the universe, will the system satisfy OR fairness? To see it clearly, let us write out the conditional parity constraints above as sums.
For notation convenience, we denote the set of all elements in $U$ such that $\mathbf{a}=a_i$ and $\mathbf{z}=z$ as $U_{a_i,z}$ and denote the probability that $C(u)=1$ as $p_u$, as in previous sections. 
\[\frac{1}{|U_{a_1,z}|}\sum_{u \in U_{a_1,z} }p_u
=
\frac{1}{|U_{a_2,z}|}\sum_{u \in U_{a_2,z} }p_u \tag{6.1} \label{eq:6.1}\]

\[\frac{1}{|U_{a_1,z}|}\sum_{u \in U_{a_1,z} }p'_u
=
\frac{1}{|U_{a_2,z}|}\sum_{u \in U_{a_2,z} }p'_u \tag{6.2} \label{eq:6.2}\]

Now let's consider the OR-fairness requirement:
\[\frac{1}{|U_{a_1,z}|}\sum_{u \in U_{a_1,z} }\Pr[C(u) = 1 \,\vee\, C'(u)=1]
=
\frac{1}{|U_{a_2,z}|}\sum_{u \in U_{a_2,z} }\Pr[C(u) = 1 \,\vee\, C'(u)=1] \tag{6.3} \label{eq:6.3}\]
Since the randomness of the classifiers is independent, we can rewrite this as
\[\frac{1}{|U_{a_1,z}|}\sum_{u \in U_{a_1,z} }p_u+ (1-p_u)p'_u
=
\frac{1}{|U_{a_2,z}|}\sum_{u \in U_{a_2,z} }p_u+ (1-p_u)p'_u \tag{6.4} \label{eq:6.4}\]
Using the fact that $C$ satisfies conditional parity, Equation \ref{eq:6.1}, Equation \ref{eq:6.4} reduces to
\[\frac{1}{|U_{a_1,z}|}\sum_{u \in U_{a_1,z} } (1-p_u)p'_u
=
\frac{1}{|U_{a_2,z}|}\sum_{u \in U_{a_2,z} } (1-p_u)p'_u \tag{6.5} \label{eq:6.5}\]
And then using the property that $C'$ satisfies conditional parity, Equation \ref{eq:6.2}, Equation \ref{eq:6.5} reduces to
\[\frac{1}{|U_{a_1,z}|}\sum_{u \in U_{a_1,z} } p_up'_u
=
\frac{1}{|U_{a_2,z}|}\sum_{u \in U_{a_2,z} } p_up'_u \tag{6.6} \label{eq:6.6}\]
So our key question of consideration becomes characterizing when satisfying
 Equations \ref{eq:6.1} and \ref{eq:6.2} (the independent conditional parity conditions) imply Equation \ref{eq:6.6} (the composed OR fairness conditional parity condition).

\paragraph{When elements with $\mathbf{z}=z$ are treated equally.} 

Any classifier that treats all elements with identical settings of $\mathbf{z}$ identically satisfies Conditional Parity in isolation. 
Under OR composition, such classifiers will also satisfy conditional parity because Equation \ref{eq:6.3} is satisfied, as all elements $u\in U_{a_1,z_i}\cup U_{a_2,z_i}$ receive positive classification with probability $p_u$.

\begin{proposition}\label{prop:conditionalparitysamez}
If a pair of classifiers $C$, $C'$ treat all elements with $\mathbf{z}=z$ equally, and if $C$ and $C'$ satisfy conditional parity in isolation, then the OR of the two satisfies conditional parity under same-task composition. 
\end{proposition}

Proposition \ref{prop:conditionalparitysamez} states that \textit{equal} individuals (those with the same qualification $\mathbf{z}=z_i$) are treated equally. The same is true for individual fairness. However same-task composition, particularly with a small degree of composition, can result in significant stretches between very \textit{similar} individuals. 
Analogous difficulties arise in the case of group fairness, although they do not specifically violate the group fairness criterion.

For example, consider two groups $a_1$ and $a_2$ as specified in Table \ref{tab:groupsamez}. 
Assume that highly qualified candidates have probability $p_h=0.9$ of acceptance, and low qualification candidates have probability $p_l=0.1$. After a single composition, the difference in acceptance rate for each group,
$|\E_{u \sim a_1}\Pr[u$ accepted at least once$] - \E_{v \sim a_2}\Pr[v$ accepted at least once$]|$, increases, and does not decrease back to the original level until after 15 compositions.

Group $a_2$ quickly approaches nearly 100\% acceptance and group $a_1$ follows more slowly. Although the ratio of the acceptance rates isn't increasing, the absolute gains for group $a_2$ under composition outstrip the absolute gains for group $a_1$. Individual fairness penalizes such absolute increases, although there may be cases in the group setting where the relative difference in measurements is the more appropriate indicator of unfairness.

\begin{table}[H]
    \centering
    \begin{tabular}{|l|l|l|l|}\hline
        \textbf{Qualification} & \textbf{True probability}& \textbf{$a_1$ \%} & \textbf{$a_2$ \%}  \\\hline
        High & $.9$ & 10\% &  85\%\\ \hline 
        Low & $.1$ & 90\% &  15\%\\ \hline 
    \end{tabular}
    \caption{Sample groups with members with qualification high or low. The true probability indicates the probability that a person with a particular qualification will succeed for the task. $a_1$ \% and $a_2$ \% indicate the percentage of each group's members with that qualification level.}
    \label{tab:groupsamez}
\end{table}

\paragraph{When elements with $\mathbf{z}=z$ are not treated equally.}

In contrast to the case above, there is no guarantee that Conditional Parity will be satisfied under 'or' composition when individuals with the same $\mathbf{z}=z$ are not all treated equally. 

There are many natural cases where we might want to treat elements with the same $z$ differently, for example, if the randomness of the environment results in a bimodal distribution for one group, and a unimodal distribution for the other. 
Let us imagine that each $z \in \mathcal{Z}$ represents a range of acceptance probabilities. In each range, individuals are classified as $p_h$, high probability within this range, $p_m$, medium probability within this range, and $p_l$, low probability within this range. Each group may consist of a different mix of individuals mapped to $p_h,p_m,$ and $p_l$. 

Consider the following simple universe: for a particular $z \in \mathcal{Z}$, group $A$ has only elements with medium qualification $q_m$, group $B$ has half of its elements with low qualification $q_l$ and half with high qualification $q_h$. Choosing $p_h=1$, $p_m=.75$, $p_l=.5$, satisfies Conditional Parity for a single application. However, for two applications, the the squares in each group diverge ($.9375\neq.875$):
$1-(1-p_m)^2 \neq \frac{1}{2}(1-(1-p_h)^2) + \frac{1}{2}(1-(1-p_l)^2)$. 
Thus, Conditional Parity is violated. 
Note, however, that many of 
 the individuals with $\mathbf{z}=z$ have been drawn closer together under composition, and none have been pulled further apart. 
 In general, $\sum_{i}x_i = \sum_{i}y_i \nrightarrow \sum_{i}x_i^t = \sum_{i}y_i^t$ for any $t>1$, so this brittleness is not unexpected.

\begin{figure}
\vskip 0.2in
\begin{center}
\centerline{\includegraphics[width=.5\textwidth]{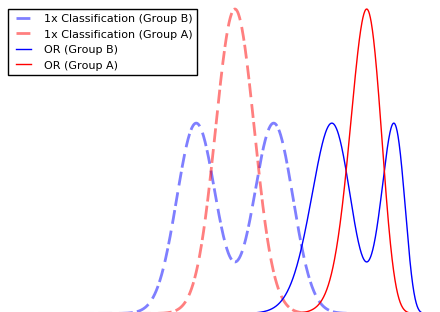}}
\caption{An illustration of the shift in groups from a single classification to the OR of two applications of the same classifier. The OR results in a more marked shift towards 1 in the unimodal population than the bimodal population. Although the two groups originally had the same mean probability of positive classification, this breaks down under OR composition.}
\label{fig:bimodal}
\end{center}
\vskip -0.2in
\end{figure}

In order to satisfy Conditional Parity under OR-composition, the classifier could sacrifice accuracy by treating all individuals with $\mathbf{z}=z$ equally. However, this necessarily discards useful information about the individuals in $A$ to satisfy a technicality.

This simple observation implies that in some cases we may observe failures under composition for conditional parity, even when individual fairness is satisfied. Indeed, notice that an individually fair classifier that treats all elements with $\mathbf{z}=z$ equally satisfies conditional parity, and if all of the probabilities of positive classification are greater than $\frac{1}{2}$, then we have seen that the OR of two such classifiers will also be individually fair. However, as we see above, those same two classifiers will only satisfy conditional parity independently.

\paragraph{Same-task Functional Composition Summary}

Although group fairness definitions make no guarantees on the treatment of individuals, the contrast between how Conditional Parity behaves under OR-composition when individuals with the same value of $\mathbf{z}$ are treated equally or not is worth considering. In some cases we may observe failures under OR-composition for Conditional Parity, even when Individual Fairness is satisfied, and failure to satisfy Individual Fairness when Conditional Parity is satisfied. 
This brittleness extends to other settings like selecting a cohort of exactly $n$ elements and satisfying calibration under composition, and to other logical functions as well as constrained settings.

\paragraph{Multiple Task Functional Composition}
In multiple-task functional composition, classifiers (for potentially distinct tasks) are combined to form a single output with its own fairness considerations. In the group setting, we are also faced with the question of the appropriate choice of protected set and stratification set when more than one task influences an outcome. For example, although protected sets may frequently overlap (eg, race and gender), stratification sets may be very different.

The results in the single outcome setting are very similar to the multiple outcome setting, which we discuss next. Briefly, there are cases both where interactions between different stratification and protected sets result in the predicted unfairness we would expect from the individual fairness results, and cases where no unfairness is detected. However, in the case in which no unfairness is detected, it is not necessarily clear whether this is due to a weakness of the requirements (in particular for socially meaningful subgroups) or a genuinely uninteresting statistical artifact. As the proof techniques and results are nearly identical, we focus our discussion in the multiple-task setting.

\subsubsection{Multiple Task Composition}
For multiple task composition, we consider two issues. First, we show how to extend our results from individual fairness to show that for a large set of tasks and \tiebreakingfunctions, classifiers which are group fair independently can result in unfair \competitivecompositions. Second, we discuss cases where conditional parity based definitions may \textit{fail} to detect multiple task composition problems that are intuitively unfair, and consider how subgroup fairness is impacted by composition.

\paragraph{Extending Individual Fairness Results}
For illustrative purposes, we consider two tasks $T,T'$ with protected attribute sets $\mathcal{A},\mathcal{A'}$ and stratification sets $\mathcal{Z},\mathcal{Z'}$ combined with \competitivecomposition~with a \tiebreakingfunction~$\Tiebreak$ to solve the single-slot composition problem. Concretely, let us consider the tasks of advertising home goods, with protected attributes of race and gender and stratification set denoting interest in purchasing home goods, and advertising jobs in technology, with protected attributes of race and gender, and stratification set denoting applicant qualification. We denote the probabilities of positive classification as $p_u,p_u'$ as in previous sections. Recall $\Tiebreak_u(T)$ denotes the probability of choosing $T$ if classified positively for both $T$ and $T'$.

For the system to satisfy multiple task conditional parity, it must be the case that the probabilities of positive classification for each task satisfy both
\[\frac{1}{|U_{a_1,z}|}\sum_{u \in U_{a_1,z} }p_u(1-p_u')+p_up_u'\Tiebreak_u(T)
=
\frac{1}{|U_{a_2,z}|}\sum_{u \in U_{a_2,z} }p_u(1-p_u')+p_up_u'\Tiebreak_u(T) \]
for all $a_1,a_2 \in \mathcal{A}$ and for all $z \in \mathcal{Z}$, and 
\[\frac{1}{|U_{a_1',z'}|}\sum_{u \in U_{a_1',z'} }p_u'(1-p_u)+p_up_u'\Tiebreak_u(T')
=
\frac{1}{|U_{a_2',z'}|}\sum_{u \in U_{a_2',z'} }p_u'(1-p_u)+p_up_u'\Tiebreak_u(T') \]
for all $a_1',a_2' \in \mathcal{A'}$ and all $z' \in \mathcal{Z'}$.

These two equations simplify, using the conditional parity of the original classifiers, to 

\[\frac{1}{|U_{a_1,z}|}\sum_{u \in U_{a_1,z} }p_up_u'(\Tiebreak_u(T)-1)
=
\frac{1}{|U_{a_2,z}|}\sum_{u \in U_{a_2,z} }p_up_u'(\Tiebreak_u(T)-1) \tag{6.7} \label{eq:6.7}\]
for all $a_1,a_2 \in \mathcal{A}$ and for all $z \in \mathcal{Z}$, and 
\[\frac{1}{|U_{a_1',z'}|}\sum_{u \in U_{a_1',z'} }p_up_u'(\Tiebreak_u(T')-1)
=
\frac{1}{|U_{a_2',z'}|}\sum_{u \in U_{a_2',z'} }p_up_u'(\Tiebreak_u(T')-1) \tag{6.8} \label{eq:6.8}\]
for all $a_1',a_2' \in \mathcal{A'}$ and all $z \in \mathcal{Z'}$.

In order to show failure to satisfy conditional parity under \competitivecomposition, we need to show how to construct $C$ and $C'$ such that the sums in Equations \ref{eq:6.7} and \ref{eq:6.8} above do not balance out, violating the equalities.

\begin{figure}
\begin{center}
\centerline{\includegraphics[width=\textwidth]{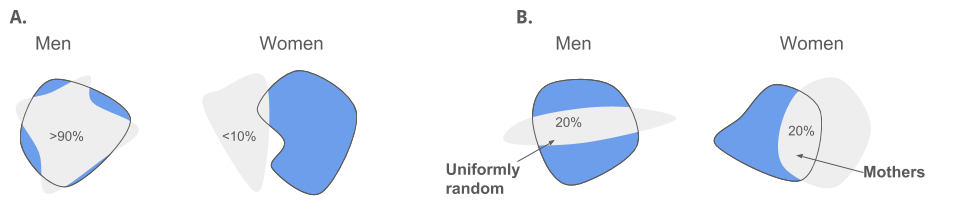}}
\caption{\textbf{A.} When the two tasks are related, one will `claim' a larger fraction of one gender than another, leading to a smaller fraction of men remaining for classification in the other task (shown in blue). Conditional parity will detect this unfairness. \textbf{B.} When the tasks are unrelated, one task may `claim' the same fraction of people in each group, but potentially select a socially meaningful subgroup, eg parents. Conditional parity will fail to detect this subgroup unfairness, unless subgroups, including any subgroups targeted by classifiers composed with, are explicitly accounted for.}
\label{fig:group_related}
\end{center}
\vskip -0.4in
\end{figure}

There are many cases where failing to satisfy conditional parity under \competitivecomposition~is clearly a violation of our intuitive notion of group fairness. For example, let's return to our advertising example where home-goods advertisers have no protected set, but high-paying jobs have gender as a protected attribute. Under composition, home-goods out-bidding high-paying jobs ads for women will clearly violate the conditional parity condition for the job ads (See Figure \ref{fig:group_related}). We formalize this idea in the theorem below.

\begin{theorem}\label{thm:multiconditionalparity} For any two tasks $T,T'$ with protected sets $\mathcal{A},\mathcal{A'}$, stratification sets $\mathcal{Z},\mathcal{Z'}$ such that $\cup_{z \in \mathcal{Z}}u_z = U$ and $\cup_{z \in \mathcal{Z'}}u_z = U$, and \tiebreakingfunction~$\Tiebreak$ such that there exist $a_1,a_2\in \mathcal{A}$, $z \in \mathcal{Z}$, $z' \in \mathcal{Z}$ such that at least one of $U_{a_1,z} \cap U_{z'}$ and $U_{a_2,z} \cap U_{z'}$ is nonempty and 
\[\frac{1}{|U_{a_1,z}|}\sum_{u \in U_{a_1,z}\cap U_{z'}}\Tiebreak_u(T) \neq \frac{1}{|U_{a_2,z}|}\sum_{u \in U_{a_2,z}\cap U_{z'}}\Tiebreak_u(T)\]
there exists a pair of classifiers $C,C'$ which satisfy conditional parity in isolation but not under \competitivecomposition.
\end{theorem}
\begin{proof}
We construct $C,C'$ that satisfy conditional parity in isolation, but not when combined with \competitivecomposition for the setting outlined above. First, construct $C,C'$ such that each element with $\mathbf{z}=z$ and each element with $\mathbf{z'}=z'$ is treated equally under $C$ and $C'$, respectively. Furthermore, require that every element has probability of positive classification $<1$.

Assume that the classifiers still satisfy conditional parity under \competitivecomposition. Therefore we have by Equations \ref{eq:6.7} and \ref{eq:6.8}:
\[\frac{1}{|U_{a_1,z}|}\sum_{u \in U_{a_1,z} }p_up_u'(\Tiebreak_u(T)-1)
=
\frac{1}{|U_{a_2,z}|}\sum_{u \in U_{a_2,z} }p_up_u'(\Tiebreak_u(T)-1)\]
for all $a_1,a_2 \in \mathcal{A}$ and for all $z \in \mathcal{Z}$, and 
\[\frac{1}{|U_{a_1',z'}|}\sum_{u \in U_{a_1',z'} }p_up_u'(\Tiebreak_u(T')-1)
=
\frac{1}{|U_{a_2',z'}|}\sum_{u \in U_{a_2',z'} }p_up_u'(\Tiebreak_u(T')-1)\]
for all $a_1',a_2' \in \mathcal{A'}$ and all $z \in \mathcal{Z'}$.

Let's first consider positive classification for $T$. Since by construction each element $u$ with $\mathbf{z}=z$ is treated equally by $C$ (has the same value for $p_u$), we can simplify the equation above to
\[\frac{1}{|U_{a_1,z}|}\sum_{u \in U_{a_1,z} }p_u'(\Tiebreak_u(T)-1)
=
\frac{1}{|U_{a_2,z}|}\sum_{u \in U_{a_2,z} }p_u'(\Tiebreak_u(T)-1) \tag{6.9} \label{eq:6.9}\]
Now let's rewrite the sums in terms of the intersecting sets in $\mathcal{Z'}$, letting $p_{z'}'$ denote the probability of positive classification for elements with $\mathbf{z'}=z'$ by $C'$.
\[\frac{1}{|U_{a_1,z}|}\sum_{z'\in\mathcal{Z'}}\sum_{\substack{u \in U_{a_1,z}\\ \cap U_{z'} }}p_{z'}'(\Tiebreak_u(T)-1)
=
\frac{1}{|U_{a_2,z}|}\sum_{z'\in\mathcal{Z'}}\sum_{\substack{u \in U_{a_2,z}\\ \cap U_{z'} }}p_{z'}'(\Tiebreak_u(T)-1)\]

By assumption, $\exists$ $a_1,a_2,z'$ such that
\[\frac{1}{|U_{a_1,z}|}\sum_{u \in U_{a_1,z}\cap U_{z'}}\Tiebreak_u(T) \neq \frac{1}{|U_{a_2,z}|}\sum_{u \in U_{a_2,z}\cap U_{z'}}\Tiebreak_u(T)\]

Now, let us modify $p_{z'}'$ for all $u \in U_{z'}$ by adding $\alpha$, that is $p_{z'}' = p_u' +\alpha$. (This is possible because by assumption all acceptance probabilities are strictly less than 1.) Thus we add 
\[\frac{1}{|U_{a_1,z}|}\sum_{u \in U_{a_1,z}\cap U_{z'}}\alpha\Tiebreak_u(T) \neq \frac{1}{|U_{a_2,z}|}\sum_{u \in U_{a_2,z}\cap U_{z'}}\alpha\Tiebreak_u(T)\]
breaking the equality in Equation \ref{eq:6.9}, which completes the construction and the proof.
\end{proof}

If (1) the representation is proportional but the \tiebreakingfunction~weights are not or (2) if the representation was not proportional to begin with, then it is easy to violate conditional parity under \competitivecomposition. In the simplest case, where the \tiebreakingfunction~is strict and the same for all individuals, this reduces to different size of intersection for each group. In our advertising example, if more women are targeted for home good advertisements then men, but each person prefers job advertisements, then the system will fail to satisfy conditional parity under \competitivecomposition. 

As in the case of individual fairness, \randomizeandclassify will fix this problem.

\begin{theorem}For any two tasks $T,T'$ with protected sets $\mathcal{A},\mathcal{A'}$ and $\mathcal{Z},\mathcal{Z'}$, if classifiers $C,C'$ satisfy conditional parity independently but not under \competitivecomposition, then \randomizeandclassify of the two classifiers will satisfy conditional parity for both tasks.
\end{theorem}
\begin{proof}
Consider the probability that $S(u)_T=1$ using \randomizeandclassify. 
\[\Pr[S(u)_T=1] = \Pr[t \sim \mathcal{X} = T]\Pr[C(u)=1]\]
As the probability that $t \sim \mathcal{X}$ is identical for all settings of $a\in \mathcal{A},z\in \mathcal{Z}$, and the original classifier $C$ satisfied conditional parity, it follows that \[\Pr[S(u)_T=1|\mathbf{a}=a_1,\mathbf{z}=z] =\Pr[S(u)_T=1|\mathbf{a}=a_2,\mathbf{z}=z] \] 
for all $a_1,a_2\in \mathcal{A}$ and for all $z \in \mathcal{Z}$.

The argument for $S(u)_{T'}$ proceeds symmetrically.
\end{proof}

Thus, we've shown, for many \tiebreakingfunctions~and for many non-identical tasks (either with same $\mathcal{A}$ or different $\mathcal{A}$), that composition can result in  group-level unfairness and that the same strategy that was effective for mitigating this unfairness for individual fairness is also effective for conditional parity based definitions.

\paragraph{Detecting Subgroup Unfairness under Composition with Conditional Parity}
Conditional parity is not always a reliable test for fairness at the subgroup level under composition.
In general, we expect conditional parity based definitions of group fairness to detect unfairness in multiple task compositions reasonably well when there is an obvious interaction between protected groups and task qualification, as observed empirically in \cite{lambrecht2016algorithmic} and \cite{datta2015automated}. However, there are some cases where it will fail to detect what we might intuitively feel to be unfairness at the subgroup level, which we discuss below. 

\begin{definition}[Subgroup Unfairness] Given a protected attribute set $\mathcal{A}$ and a stratification set $\mathcal{Z}$, a predictor $Y$ is subgroup unfair if for any socially meaningful subgroup $S$, $Y$ fails to satisfy conditional parity with respect to the attribute set $\A \cup \{1_S\}$.
\end{definition}

Socially meaningful subgroups can take many forms and will likely depend on context. For example, consider two advertising campaigns, one run by a hospital intended to increase cancer screening in older adults, and one run by a grocery chain aimed at getting subscribers for a new home grocery delivery service. The hospital is primarily targeting older adults and is aware that in the past there have been different levels of outreach to different racial groups, despite similar cancer risk. The grocery service, on the other hand, is primarily targeting women, the default grocery shopper for most households, and is also concerned with preventing racial discrimination. We'll now see how simply satisfying conditional parity can lead to poor outcomes for a \textit{subgroup}, namely older women.

First, we introduce the notion of unrelated tasks.
\begin{definition}[Unrelated tasks]\label{def:unrelatedtasks} Two tasks $T,T'$ are considered unrelated if for all $a\in \mathcal{A}$, $a'\in \mathcal{A}'$, $z \in \mathcal{Z}$ and $z' \in \mathcal{Z'}$
\[\Pr[\mathbf{z}=z | \mathbf{a}'=a', \mathbf{z}'=z'  ] = \Pr[\mathbf{z}=z]\]
\[\Pr[\mathbf{z}'=z' | \mathbf{a}=a, \mathbf{z}=z] = \Pr[\mathbf{z}'=z']\]
where probability is taken over selection of members of the universe.
\end{definition}

Definition \ref{def:unrelatedtasks} captures the intuition that two tasks are \textit{unrelated} if the protected set and stratification membership for the either task is not predictive of quality for the other task. Returning to our groceries versus cancer screening public service announcement example, if the two advertisers use conditional parity with the protected set $\{$race$\}$, then it's clear that the combination of protected attribute and quality for one task isn't indicative of quality for the other task, as there are equal numbers of men and women of each race at each age (see Figure \ref{fig:psa}). 

\begin{figure}
\vskip 0.2in
\begin{center}
\centerline{\includegraphics[width=90mm]{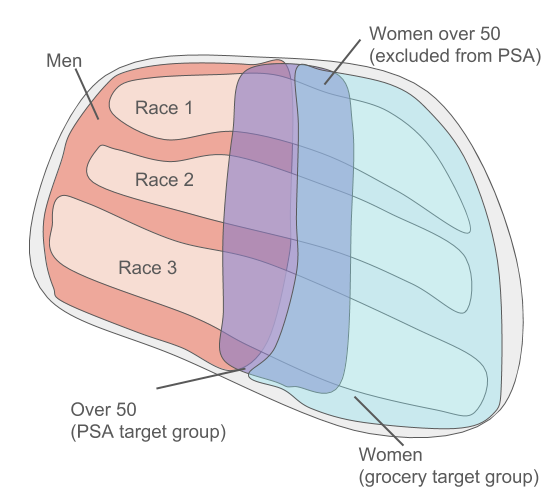}}
\caption{There are equal numbers of men and women over age 50 and of each race. The public service announcement classifier (ensuring parity based on race) and the grocery classifier, targeting primarily women, will not interfere with each other from the perspective of Conditional Parity. However, women over 50 are excluded from the public service announcement when grocery targeting wins.}
\label{fig:psa}
\end{center}
\vskip -0.2in
\end{figure}

\begin{lemma}\label{lemma:unrelatedgroup}
If two tasks $T,T'$ are \textit{unrelated}, and  $\Tiebreak_u(T)=\Tiebreak_v(T)$ for all $u,v \in U$, then any pair of classifiers $C,C'$ that treat each universe element with qualification $\mathbf{z}=z$ equally will always satisfy conditional parity under \competitivecomposition.
\end{lemma}

Lemma \ref{lemma:unrelatedgroup} is interesting to us because in our example of groceries versus cancer public service announcements above, we have seen a case where unrelated tasks interact to cause subgroup unfairness, and yet the Lemma tells us conditional parity will be satisfied anyway.

\begin{proof}
We want to show that if two tasks are unrelated and the \tiebreakingfunction~is the same for all individuals, then if two classifiers satisfy conditional parity in isolation, they will satisfy conditional parity under \competitivecomposition. Equivalently, if
\[\Pr_{C, u \sim U_{a_1,z}}[C(u)=1 |\mathbf{a}=a_1,\mathbf{z}=z ]=\Pr_{C, u \sim U_{a_2,z}}[C(u)=1 |\mathbf{a}=a_2,\mathbf{z}=z ]\]
\[\Pr_{C', u \sim U_{a_1,z}}[C'(u)=1 |\mathbf{a}'=a_1',\mathbf{z}'=z' ]=\Pr_{C', u \sim U_{a_2,z}}[C'(u)=1 |\mathbf{a}'=a_2',\mathbf{z}'=z' ]\]
for all $a_1,a_2 \in \mathcal{A}$ and for all $a_1',a_2' \in \mathcal{A}'$ and for all $z \in \mathcal{Z}$ and for all $z' \in \mathcal{Z'}$ then
\[\Pr[S(u)_T=1 |\mathbf{a}=a_1,\mathbf{z}=z ]=\Pr[S(u)_T=1 |\mathbf{a}=a_2,\mathbf{z}=z ]\]
\[\Pr[S(u)_{T'}=1 |\mathbf{a}'=a_1',\mathbf{z}'=z' ]=\Pr[S(u)_{T'}=1 |\mathbf{a}'=a_2',\mathbf{z}'=z' ]\]
for all $a_1,a_2 \in \mathcal{A}$ and for all $a_1',a_2' \in \mathcal{A}'$ and for all $z \in \mathcal{Z}$ and for all $z' \in \mathcal{Z'}$, where probability is taken over the randomness of the composed system and the choice of individual in each protected set and stratification set setting. Here we have used $\Pr_C[E]$ to denote the probability of the event $E$ taken over the randomness of $C$.

First let us consider $\Pr[S(u)_T=1]$.

As $\Tiebreak_u(T)$ is the same for all $u \in U$, we replace this with the constant $\rho$ for simpler notation. 

\[\Pr_{\Tiebreak,C,C'}[S(u)_T=1 ] = \rho\Pr[C(u)=1 \wedge C'(u)=1] + \Pr[C(u)=1 \wedge C'(u)=0]\]

where probability is taken over the randomness of the classifiers and the \tiebreakingfunction.

As the random bits of the two classifiers $C$ and $C'$ are independent, we can write
\[\Pr[C(u)=1 \wedge C'(u)=0 ] = \Pr[C(u)=1]\Pr[C'(u)=0 ]\]
\[\Pr[C(u)=1 \wedge C'(u)=1] = \Pr[C(u)=1] \Pr[C'(u)=1]\]
for each part of the sum.

Now we want to reason about $\Pr[S(u)_T=1 ] $ for  particular settings of $\mathbf{a}$ and $\mathbf{z}$.
Let us write out these conditions for each protected set and stratification pair as sums to see more clearly. 
\[\frac{1}{|U_{a_1,z}|}\sum_{u \in U_{a_1,z}}\Pr[C(u)=1]\Pr[C'(u)=0] =\frac{1}{|U_{a_2,z}|}\sum_{u \in U_{a_2,z}}\Pr[C(u)=1]\Pr[C'(u)=0] \]
\[\frac{1}{|U_{a_1,z}|}\sum_{u \in U_{a_1,z}}\Pr[C(u)=1]\Pr[C'(u)=1] =\frac{1}{|U_{a_2,z}|}\sum_{u \in U_{a_2,z}}\Pr[C(u)=1]\Pr[C'(u)=1] \]
where probability is taken over the randomness of the classifiers.

Using the usual notation of $p_u,p_u'$ to simplify, we have  

\[\frac{1}{|U_{a_1,z}|}\sum_{u \in U_{a_1,z}}p_u(1-p_u') =\frac{1}{|U_{a_2,z}|}\sum_{u \in U_{a_2,z}}p_u(1-p_u')\]
\[\frac{1}{|U_{a_1,z}|}\sum_{u \in U_{a_1,z}}p_up_u' =\frac{1}{|U_{a_2,z}|}\sum_{u \in U_{a_2,z}}p_up_u' \]

By our assumption that each element with $\mathbf{z}=z$ is treated equally in $C$, we can simplify the above to 

\[\frac{1}{|U_{a_1,z}|}\sum_{u \in U_{a_1,z}}(1-p_u') =\frac{1}{|U_{a_2,z}|}\sum_{u \in U_{a_2,z}}(1-p_u')\]
\[\frac{1}{|U_{a_1,z}|}\sum_{u \in U_{a_1,z}}p_u' =\frac{1}{|U_{a_2,z}|}\sum_{u \in U_{a_2,z}}p_u' \]
Recall that the two tasks are unrelated, so the qualification of an element $u$, that is which $z' \in \mathcal{Z'}$ it lands in, is unrelated to its memberships in $\mathcal{A}$ and $\mathcal{Z}$. Thus the share of each $\mathbf{z'}=z'$, and therefore each $p_u'$, for $a_1$ and $a_2$ is the same, satisfying the equality. To see this more clearly, we split up the sums by membership in $\mathcal{Z}'$.
\[\frac{1}{|U_{a_1,z}|}\sum_{z' \in \mathcal{Z}'}\sum_{u \in U_{a_1,z}\cap U_{z'}}p_u' =\frac{1}{|U_{a_2,z}|}\sum_{z' \in \mathcal{Z}'}\sum_{u \in U_{a_2,z}\cap U_{z'}}p_u' \]
By our assumption of unrelatedness, $\frac{1}{|U_{a_1,z}|}*|U_{a_1,z}\cap U_{z'}| = \frac{1}{|U_{a_2,z}|}*|U_{a_2,z}\cap U_{z'}| $,
call these values $t_{a,z,z'}$ for convenience. Now we have that 
\[\sum_{z' \in \mathcal{Z}'}t_{a_1,z,z'}p_u' =\sum_{z' \in \mathcal{Z}'}t_{a_2,z,z'}p_u' \]
and the equality clearly holds.

The argument for $\Pr[S(u)_{T'}=1 ] $ proceeds analogously. 
\end{proof}

The key point of this lemma is that such classifiers will never fail to satisfy conditional parity under \competitivecomposition, even though in some cases a subgroup is clearly treated poorly, as in the groceries versus public service announcement example.

Notice that many pairs of individually fair classifiers meet the requirements for Lemma \ref{lemma:unrelatedgroup}. In the previous discussion of individual fairness, we also observed how \competitivecomposition, even with equal preferences for all universe elements, results in significant violations of individual fairness. Indeed, the characterization of Lemma 
\ref{lemma:unrelatedgroup} is incomplete, and other settings may similarly not violate conditional parity, but still intuitively be unfair. Of particular concern in practice is the possibility that classifiers may \textit{learn} to exclude certain well-defined subgroups in order to achieve conditional parity by \textit{simulating} unrelatedness.

Recall our simple advertising system with two types of advertisers, employers and home-goods advertisers. 
If, in response to gender disparity caused by task-competitive composition, classifiers iteratively adjust their bids to try to achieve Conditional Parity, they may unintentionally \textit{learn} themselves into a state that satisfies Conditional Parity with respect to gender, but behaves poorly for a socially meaningful subgroup. (See Figure \ref{fig:group_unrelated}). For example, let's imagine that home goods advertisers aggressively advertise to women who are new parents, as their life-time value to the advertiser ($\mathcal{Z}$) is the highest of all universe elements. A competing advertiser for  jobs, noticing that its usual strategy of recruiting all people with skill level  $\mathbf{z'}=z'$ equally is failing to reach enough women, bids more aggressively  on women. 
By bidding more aggressively, the advertiser increases the probability of showing ads to women (for example by outbidding low-value competition), but not to women who are bid for by the home goods advertiser (a high-value competitor), resulting in a high concentration of ads for women who are {\em not} mothers, while still failing to reach women who {\em are} mothers. Furthermore, the systematic exclusion of mothers from job advertisements can, over time, be even more problematic,  
as it may contribute to the stalling of careers.
In this case, the system discriminates against mothers without necessarily discriminating against fathers.

\begin{figure*}
        \centering
        \begin{subfigure}[b]{0.475\textwidth}
            \centering
            \includegraphics[width=\textwidth]{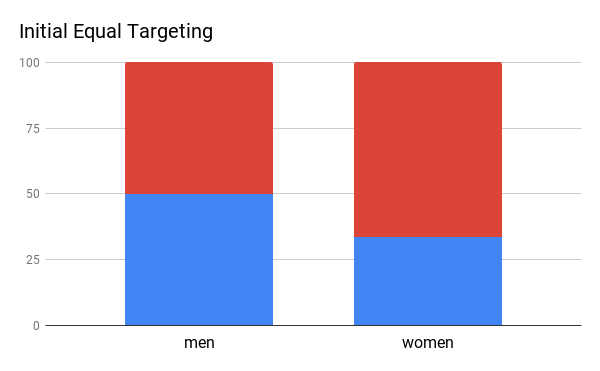}
            \caption{Initial equal targeting of qualified men and women results in violation of conditional parity, as there are unequal rates of ads shown (blue).}    
            \label{fig:initialsettingscoarse}
        \end{subfigure}
        \hfill
        \begin{subfigure}[b]{0.475\textwidth}  
            \centering 
            \includegraphics[width=\textwidth]{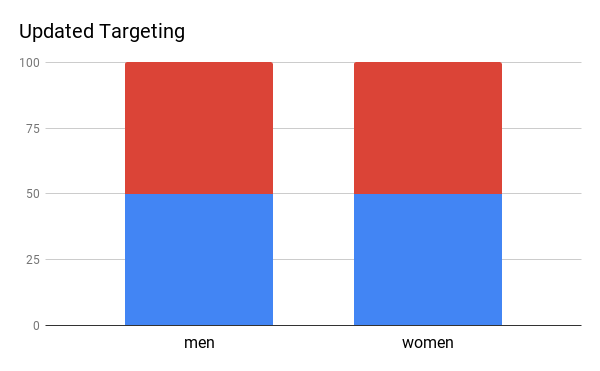}
            \caption
            {By increasing the targeting of women, the jobs advertiser ``fixes'' conditional parity at the coarse group level.}    
            \label{fig:updatedsettingscoarse}
        \end{subfigure}
        \vskip\baselineskip
        \begin{subfigure}[b]{0.475\textwidth}   
            \centering 
            \includegraphics[width=\textwidth]{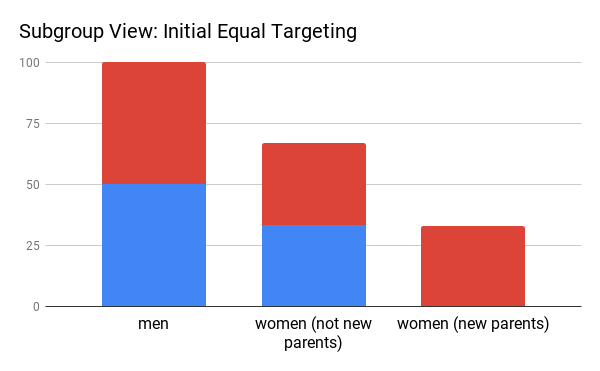}
            \caption{At the subgroup level, it's clear that the lack of conditional parity is due to ``losing'' all of the new parent women to the home-goods advertiser.}    
            \label{abs:fig:initialsettingssub}
        \end{subfigure}
        \quad
        \begin{subfigure}[b]{0.475\textwidth}   
            \centering 
            \includegraphics[width=\textwidth]{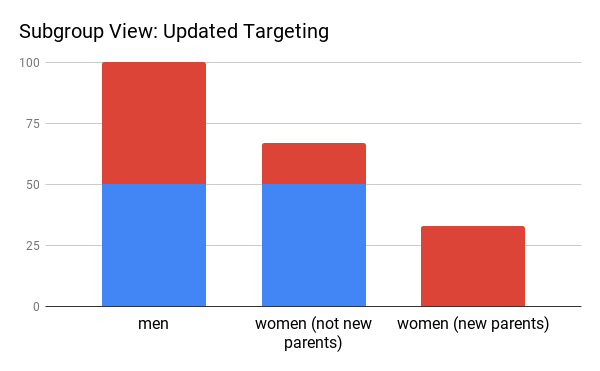}
            \caption{New targeting strategy increases ads shown to non new-parent women, but continues to exclude new parent women.}   
            \label{abs:fig:updatedsettingscoarse}
        \end{subfigure}
        \caption
        {Home-goods advertisers aggressively target mothers, out-bidding the jobs advertiser. When the jobs advertiser bids more aggressively on `women' (b) the overall rate of ads shown to `women' increases, but mothers may still be excluded (d), so $\Pr[\text{ad }| \text{qualified, woman}]> \Pr[\text{ad }| \text{ qualified, mother}]$.} 
        \label{fig:group_unrelated}
    \end{figure*}

Excluding subgroups is not specifically a problem of composition. It is certainly possible that a malicious advertiser could take the same approach even without composition coming into the picture. However, we stress this potential for subgroup exclusion because subgroups are likely to be targeted or have higher competition in practice, and predicting or identifying all such possible subgroups may be difficult. Furthermore, the attributes used to define these subgroups may be unavailable to some learning procedures, exacerbating the problem of detection. Practitioners whose previous strategies (treat everyone of equal qualification equally) may fall apart under composition, leading them to pursue strategies that lead to such subgroup unfairness, even though their strategy and statistical results are nominally fair. 

Although problematic (large) subgroup semantics are part of the motivation for \cite{kearns2017gerrymandering,hebert2017calibration}, the danger of composition is that the features describing this subset may be missing from the feature set of the jobs classifier, rendering the protections proposed in \cite{kearns2017gerrymandering} and \cite{hebert2017calibration} ineffective.
In particular, we expect sensitive attributes like parental status are unlikely to appear (or are illegal to collect) in employment-related training or testing datasets.

\subsubsection{Dependent Composition}
The extension of Individual Fairness results for dependent compositions have several of the same caveats we have seen in the extensions of functional composition and multiple-task composition. We do not belabor these points, and instead focus on the more interesting aspects of the extensions.

\paragraph{Cohort Selection}
The Cohort Selection Problem in the group setting is subtly different from the individual fairness setting, in that we want to maintain an equivalence in probability, rather than preventing an increase in distance between probability distributions. As we saw previously for same-task functional composition, these problems don't always align. 

\paragraph{The Offline Cohort Selection Problem}
\begin{enumerate}
    \item \textbf{When all elements with $\mathbf{z}=z$ are treated equally.} In this case, the extension of \permutethenclassify to the group setting is straightforward. For each permutation we consider, the probability that a given element appears in a particular location in the ordering is independent of its protected attributes, and so the equivalence is maintained.
    \item \textbf{When elements with $\mathbf{z}=z$ are not treated equally.} A simple counterexample suffices to show that \permutethenclassify will not satisfy conditional parity in this case. Consider the universe with three elements $a,b_1,b_2$. Consider a classifier which satisfies conditional parity across the groups $\{a\}$ and $\{b_1,b_2\}$ by accepting $a$ with probability $0.75$, and $b_1$ with probability $1$, and $b_2$ with probability $0.5$. That is $\E[C(a)]=0.75]$, $\E[C(b_1)]=1$, and $\E[C(b_2)]=0.5$. Under \permutethenclassify with $n=1$, the probability of acceptance for $a$ is $0.31$, whereas the probability for acceptance for the group of $\{b_1,b_2\}$ is $0.34$, and the equivalence between the groups is broken, even though all individual elements are pulled closer together. 
\end{enumerate}

\paragraph{Online Cohort Selection}
The most interesting point in the extensions in the online setting is that in some cases Statistical Parity, which  corresponds to conditional parity with $|\mathcal{Z}|=1$, can be satisfied when individual fairness or more general settings of Conditional Parity do not have fair solutions. Consider Theorem \ref{thm:adversarialunknownlength}, which states that if the ordering of the elements of the universe is adversarial and the stream is of an unknown length, that we cannot select exactly $n$ elements and satisfy individual fairness. Indeed, this extends to the more interesting cases of conditional parity in a straightforward way, but not to statistical parity if the desired proportions are known.

\begin{theorem}\label{conditionalparityunknownlength} For any ordering, there exists a solution to the online or offline cohort selection problem for Statistical Parity for unknown length as long as the required proportion of the output for each protected attribute setting is known.
\end{theorem}
\begin{proof}
Consider the system which knows that a $p_{a}$ fraction of all elements chosen should have a particular protected attribute setting $a$. In order to satisfy statistical parity in the online setting, the system simply selects the first $p_an$ elements with protected attribute setting $p_a$. The system will select the desired $p_a$ fraction for each protected attribute setting $a$. 
\end{proof}

This solution will clearly violate individual fairness, as well as many variants of conditional parity. To see how this technique fails for conditional parity more generally, imagine that an adversary ordered men from highest qualification to lowest, and women from lowest to highest. This system would select the most qualified men and the least qualified women, clearly violating conditional parity for a stratification set relating to job qualification, but not violating Statistical Parity. 

Although this case is  contrived, it's important to notice that such a system can appear to be fair (if one is satisfied with statistical parity as a notion of fairness), but clearly results in undesirable long-term effects. In our example above, deliberately hiring under-qualified women (as opposed to the qualified women later in the stream) can poison future decisions, and be used to justify hiring fewer women in the future.\footnote{This particular problem was called out in the original `Catalog of Evils' in \cite{dwork2012fairness}.} 

Furthermore, it may be difficult to determine that the ordering is adversarial when the relevant subgroup attributes are missing. For example, a system may assume that the ordering of its inputs is drawn uniformly at random because the distribution of observed attributes is statistically indistinguishable from random. However, an adversary may still be able to manipulate the ordering to benefit or harm socially meaningful subgroups which are not explicitly described by the feature set of a particular system. For example, an adversary may select an ordering which appears random with respect to talent and gender, but places parents later in the ordering than non-parents. Without a clear signal of parental status, a system will have difficulty determining that the ordering can have negative consequences on the parent subgroup. 

In practice, many cases of seemingly adversarial ordering will be difficult to identify and may not even arise through malicious intent. For example, imagine that a bank, in an effort to fairly process loan applications, has used a fair classifier to assign individuals to descriptive `bins' which indicate their probability of repaying a loan. Loan officers interact with applications through the loan review tool, which only displays the relevant bin information. The loan review tool is based on spreadsheet software and sorts the applications by last name by default. Even though the last name is not displayed to the loan officer (only the descriptive bin), the resulting system will still be unfair if a limited number of loans are available, as loans are more likely to be granted to the first applications processed (Adams and Alvarez) rather than the last applications processed (Zhang and Zou). 

\paragraph{Subset Classification Problem}
\paragraph{Positive weight in $\mathcal{Y}$ for all elements.} Consider a distribution $\mathcal{Y}$ over subsets of $U$ in which at least one element $u \in U$ is contained in each protected attribute and stratification pair with positive weight.
In this case, the same procedural adjustment proposed in Section \ref{section:positiveweights} will suffice, as each element $w$ will be selected with probability $q_{min}p_w$, where $p_w$ is the probability of selection with uniform inputs. Thus, the equality $\frac{q_{min}}{|U_{a_1,z}|}\sum_{u \in U_{a_1,z}}p_u = \frac{q_{min}}{|U_{a_2,z}|}\sum_{u \in U_{a_2,z}}p_u$ holds.

\paragraph{Comparable outcomes}
Unfortunately, the extension for the procedure for comparable outcomes is not as straightforward. Consider again the universe we described above with $a$ classified positively with probability $0.75$, and $b_1$ with probability $1$, and $b_2$ with probability $.5$. Take $U\backslash V$ to be $\{a,b_1,b_2\}$ and $V$ to be $\{a,b_1\}$. If the system uses such a classification procedure on $U \backslash V$, it will satisfy conditional parity. However, the same classification procedure applied to $V$ alone will not satisfy conditional parity. This setting can still be handled with a slightly stronger requirement on the behavior on $U \backslash V$.

\begin{lemma}\label{lemma:sameoutcomesconditionalparity}
Consider a distribution over subsets $\mathcal{Y}(U)$, and a classifier $C^*$ which operates on $U\backslash \mathcal{Y}(U)$. If $C^*$ satisfies conditional parity when applied to $U$, and if all outcomes of $C^*$ are in the potential outcome space of the classifier operating on $V$, then then there exists a classifier $C'$ such that the system which applies $C^*$ to $U \backslash V$ and $C'$ to $V$ satisfies conditional parity.
\end{lemma}
\begin{proof}
Take $C'$ to be identical to $C^*$. As $C^*$ satisfies conditional parity when applied to $U$, the combination of $C'$ applied to $V $ and $C^*$ applied to $U \backslash V$ also satisfies conditional parity.
\end{proof}

\subsubsection{Group Fairness: Summary}
In this section, we've shown that composition issues are not merely artifacts of the individual fairness definition, and indeed are observed in many natural settings for group fairness definitions. 

We have shown that classifiers which appear to satisfy group fairness properties in isolation may not compose well with other fair classifiers, that the signal provided by group fairness definitions under composition is not always reliable, and that composition requires additional considerations for subgroup fairness. In particular, even if one is satisfied with a notion of actuarial fairness at the group (or individual) level, we have shown that no guarantees can be made under composition.
A promising direction for future work is the augmentation of classifiers  group fairness for large, intersecting, groups~\cite{kearns2017gerrymandering, hebert2017calibration}, as well as classifiers with Individual Fairness for large subgroups) \cite{kim2018fairness},
to incorporate contextual information, with the goal of improving composition.

\subsection{Summary of Composition Results}\label{section:summary}
Now that we have all of the core results in place, we can make several observations about what it means to have a `fair' classifier.

\paragraph{Asserting fairness or unfairness requires context.} As we've seen in the preceding sections, there are many cases where classifiers are either individually fair or  satisfy conditional parity in isolation, but fail to satisfy these definitions under composition. Furthermore Lemma \ref{lemma:positiveweights} highlights a very natural setting where a classifier which appears unfair in isolation is the right choice for constructing a fair system with composition. In particular, classifiers which seem to heavily rely on attributes ``inappropriate'' for the task (like parental status or sexual orientation), may specifically be doing so in order to prevent composition failures with other classifiers legitimately targeting based on these features. Conversely, classifiers which seem to be free of influence from ``inappropriate'' attributes in isolation may fail to provide the same protections under composition. In either direction, it's clear that certifying a classifier as fair or stating unequivocally that it is unfair requires significant understanding of the composition context in which the classifier will be employed. 

\paragraph{Augmented classifiers.} Given the need for additional context, one possible path forward is to create augmented classifiers, which provide additional information about their anticipated input distributions ($\mathcal{Y}$), operating mode and ordering ($\mathcal{X}$), and expected post-processing. In cases where a particular input ordering is not fixed in advance, we could also imagine a family of classifiers parametrized by $\mathcal{Y}$ from which the appropriate classifier may be selected at a later point when more information about $\mathcal{Y}$ is available. However, such additional information about expected input or output distributions may have significant privacy implications in cases where the entire output distribution may not have been available to other parties in the past.

\subsection{Conclusions and Future Work}\label{section:futurework}

\begin{conclusion} There is no guarantee of fairness under composition, fairness under post-processing, or resilience of fairness to arbitrary side information, for either individual fairness or group fairness.
\end{conclusion}

We have shown that na\"{\i}ve composition of fair classifiers can result in unfairness, both for individual fairness and a large class of group fairness definitions. We have also shown mechanisms to mitigate this unfairness in several settings. Finally, we concluded by suggesting that fairness is not a property of classifiers in isolation, and that to construct fair systems in practice augmented definitions of fairness with sufficient context for fair composition are desirable.

\begin{conclusion}
Fairness violations can be corrected.
\end{conclusion}
Unlike privacy, where a breach of privacy must be considered permanently irreparable, fairness is far more robust to mistakes. 
In many cases, we can remedy unfairness after the fact, and the harm does not have to be permanent.
Consider the difference between a breach of a credit reporting company's databases and a free school lunch program. When the credit reporting agency loses control of sensitive data, the best they can do is try to limit the \textit{impact} of the privacy loss, either by providing monetary compensation or credit monitoring. On the contrary, operating under the assumption that all students are equally deserving of lunch, a free or subsidized school lunch program \textit{repairs} the underlying unfairness of access to lunch money or lunch from home. In this case, the unfairness of access to lunch is repaired entirely, which is impossible with privacy loss.
Given that it is possible, either through coordinated algorithmic solutions or external interventions, to remedy unfairness in the system, it makes sense to consider not only the behavior of each component of the system, but the system as a whole.

\begin{conclusion}
Auditing and definition choices must take composition into account.
\end{conclusion}
Throughout this work, we have shown that the {\it choice of outcomes} on which to enforce fairness is critical to constructing systems which reflect the true intent of the original fairness requirement. Furthermore, we have shown that in some cases group fairness definitions may behave unexpectedly under composition. Thus, any choice of auditing or enforcement must not only carefully consider the points at which fairness is measured or enforced, but whether those conclusions will hold under composition.

\subsubsection{Future Work}
We see several directions for future work. 
First, there are likely many more mechanisms for fair composition with or without coordination in training procedures for the problems we described. In particular, investigating alternatives for \randomizeandclassify that improve allocation and have more practical utility guarantees will likely be necessary for practical adoption. We also did not explicitly show mechanisms for fair composition for constrained cohort selection, for example, in assigning students to public and private schools with limited flexibility in campuses and with potential conflicts in student preferences, which are likely to be common problems. Even if perfectly fair solutions cannot be found, there may be acceptable relaxations.

Augmented classifiers provide another avenue for exploration; how can we specify the requirements, and what are the privacy implications of the additional information that they require? Fairness at the expense of a total loss of privacy is unlikely to be an acceptable solution in practice, so understanding how tradeoffs must be made and whether parties with existing access to private information can enforce fairness is an important question to answer.

A number of the impossibility results, in particular those of constrained cohort selection, could be addressed by requiring similarity over other measures. For example, we could require that similar individuals have similar ``rewards to effort.'' 
There are many potential alternatives to explore in the economics literature.

This paper largely ignores the problem of generalization, as our results are primarily negative. However, it is important to understand the generalization properties of the constructions proposed and to understand how generalizable metrics for individual fairness can be learned and represented.


\section{Multiple Task Fairness - Empirical Intuition}\label{section:empirical}
To more clearly illustrate the potential for unfairness in realistic \competitivecompositions, we devised a simple empirical setting. 
For our motivating example, we'll consider the problem of inviting students to a seminar or a free pizza lunch offered in the same time slot on opposite sides of campus (the graduate student's dilemma). As no student can attend both, the goal is to design a system that fairly allocates at most a single invitation (for either event) to each student.

We generated a  sample data set of 100 students, each with pizza and seminar intrinsic qualifications drawn independently from $\mathcal{N}(0.5,0.25)$, that is $q_{u,p} \sim \mathcal{N}(0.5,0.25)$ and independently $q_{u,s} \sim \mathcal{N}(0.5,0.25)$.\footnote{Any values exceeding 1 or less than zero were clamped to keep all distances less than or equal to 1. If we instead discarded these values, we would have fewer equal pairs and fewer values in the extremes. Although the impact is observable empirically, the effect is not significant enough as to impact the overall trends or results.} 
We considered differences in intrinsic qualification to be each pair's true distance under the metrics for pizza and seminar, that is $\mathcal{D}_p(u,v) = |q_{u,p} - q_{v,p}|$ and $\mathcal{D}_s(u,v) = |q_{u,s} - q_{v,s}|$ respectively. 
Using these metrics, we learned two fair classifiers for each task by solving a linear program maximizing a simple objective function for each task as specified in \cite{dwork2012fairness}. 
We designed our objective functions to maximize the qualification of the recipients of invitations, while keeping to an expected number of invitations of at most $t_s=30$ and $t_p=40$ for the two tasks. We then composed the two classifiers in several compositions, the results of which are discussed below and summarized in Table \ref{tab:compositionempiricalresults}.

\begin{table}[H]
    \centering
    \begin{tabular}{|l|l|l|l|l|}\hline
        \textbf{Composition Type} & \textbf{Task} & \textbf{\% pairs in} & \textbf{Average } & \textbf{Max }  \\ 
         &  & \textbf{violation} & \textbf{violation} & \textbf{violation}  \\\hline
        $(^*)\Tiebreak(u,pizza)=0$ & Pizza & 23.0\% & 0.061 & 0.35\\ \hline 
        $(^*)\Tiebreak(u,pizza)=0$ & Seminar & - & - & -\\ \hline \hline
        $(^*)\Tiebreak(u,pizza)=1$ & Pizza & - & - & -  \\ \hline
        $(^*)\Tiebreak(u,pizza)=1$ & Seminar & 19.76\% &0.069&0.387  \\ \hline \hline
        $\Tiebreak(u,pizza)=.5$ & Pizza & 20.5\% & 0.030 & 0.173 \\ \hline
        $\Tiebreak(u,pizza)=.5$ & Seminar & 14.0\% & 0.034 & 0.185\\ \hline \hline
        $\Tiebreak(u,pizza)=q_{u,p}$ & Pizza & 41.8\% & 0.032 & 0.115 \\ \hline
        $\Tiebreak(u,pizza)=q_{u,p}$ & Seminar & 21.0\% & 0.067 & 0.413 \\ \hline \hline
        $\Tiebreak(u,pizza)=q_{u,s}$ & Pizza & 25.0\% & 0.043 & 0.284 \\ \hline
        $\Tiebreak(u,pizza)=q_{u,s}$ & Seminar & 16.3\% & 0.036 & 0.23 \\ \hline
    \end{tabular}
    \caption{Summary of composition data for our small empirical example averaged over 100 randomly generated universes of size 100. ($^*$ indicates averaged over 150 trials) \textbf{\% pairs in violation} is the fraction of pairs which whose distances increased under the composition exceeded their distances under the metric. 
    The \textbf{average violation} is the average difference between the distance under the composition and in the original metric. Note that this value is not fractional, and so \textit{underestimates} the relative increase in distance. For example, a 0.01 increase in distance for a pair originally at distance 0.25 is equivalent to a $4\%$ relative increase in distance. The \textbf{maximum violation} is the maximum difference between the distance under the composition and in the original metric, which again, is not fractional. The maximum reported is the average of all maximums observed, not the maximum of all observed. (The maximum exceeded 0.55).}
    \label{tab:compositionempiricalresults}
\end{table}

\textbf{Strict Preference (strict ordering):} Recall that $\Tiebreak(u,pizza)$ is the probability of selecting pizza if a positive classification is received for both pizza and seminar. The first two compositions considered strict ordering or strict preference. For these cases, note that the strictly preferred task or the first task in the ordering has no fairness violations, as its outcomes  are equivalent to the setting where the classifier was run independently.\footnote{The difference in the values reported for the two strict compositions are due to the asymmetric $t_p$ and $t_s$. If we had used identical $t_p$ and $t_s$, we would have expected nearly identical classifiers (given that we used the same distribution of qualifications).} In the Figures \ref{fig:pizzawithoutseminar} and \ref{fig:pizzawithseminar}, we illustrate the probability that an individual with a particular $q_{u,p},q_{u,s}$ pair is invited for pizza with the intensity of the color. The difference in the pizza allocation between the independent classification and the strictly ordered composition (with seminar invitations issued first) is clearly visible in the change of color intensity. 
\begin{figure}[H]
    \centering
    \includegraphics[width=0.7\textwidth]{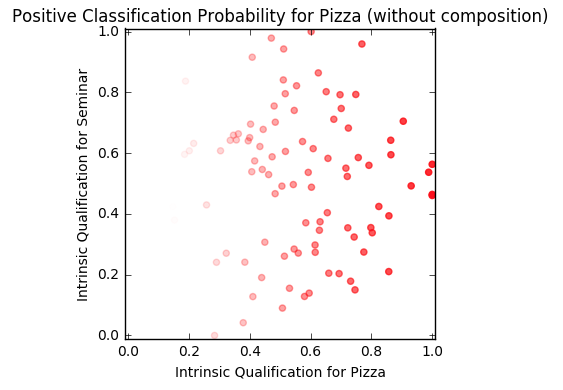}
    \caption{The intensity of the color shows the probability that an individual will be classified positively for pizza when seminar invitations are not considered. The color intensity fades from right to left as qualification for pizza decreases. Note that the vertical line through each $x-$coordinate includes individuals who are equally qualified for pizza with different qualifications for seminar, and that the color intensity is the same along the whole line. } \label{fig:pizzawithoutseminar}
\end{figure}

\begin{figure}[H]
    \centering
    \includegraphics[width=0.7\textwidth]{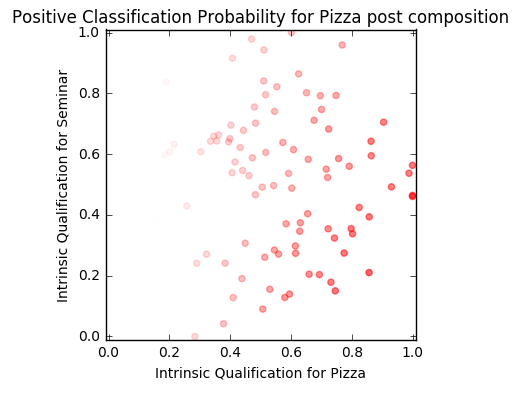}
    \caption{Compared with Figure \ref{fig:pizzawithoutseminar}, the impact of first inviting to seminar and then to pizza in a strict order is clearly visible. We can see that instead of having equal color along vertical lines, there is a significant lightening of the color as seminar qualification increases. That is, the color intensity fades from right to left \textit{and} from bottom to top, rather than only right to left.} \label{fig:pizzawithseminar}
\end{figure}



If the ordering of the preference is switched, we can see the same pattern for seminar invitations in Figures \ref{fig:seminarnocomp} and \ref{fig:seminarchoicecomp}.

\begin{figure}[H]
    \centering
    \includegraphics[width=0.7\textwidth]{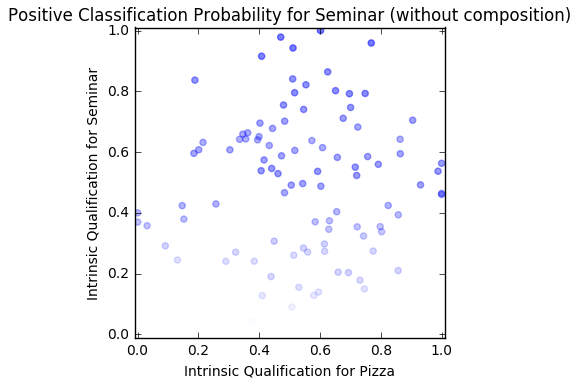}
    \caption{The intensity of blue color indicates the probability of invitation to seminar without composition. Note that the color intensity fades only from top to bottom as qualification for seminar decreases. Each horizontal line corresponds to a particular qualification for seminar, and thus has the same color intensity across the line.} \label{fig:seminarnocomp}
\end{figure}

\begin{figure}[H]
    \centering
    \includegraphics[width=0.7\textwidth]{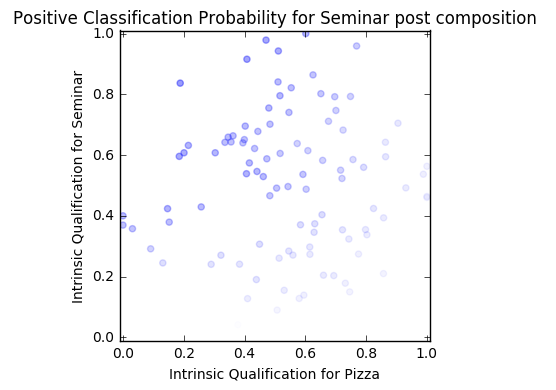}
    \caption{In contrast to Figure \ref{fig:seminarnocomp}, each horizontal line does not have the same color intensity, and intensity now fades from bottom to top \textit{and} left to right.} \label{fig:seminarchoicecomp}
\end{figure}

\textbf{Nontrivial Preferences:} If we consider instead a \competitivecomposition where the preference for pizza and seminar are equal, we see a less dramatic, but two-sided impact as \textit{both} tasks now have pairs with distance violations, as illustrated in Figures \ref{fig:pizzaiq} and \ref{fig:seminariq}.

\begin{figure}[H]
    \centering
    \includegraphics[width=0.7\textwidth]{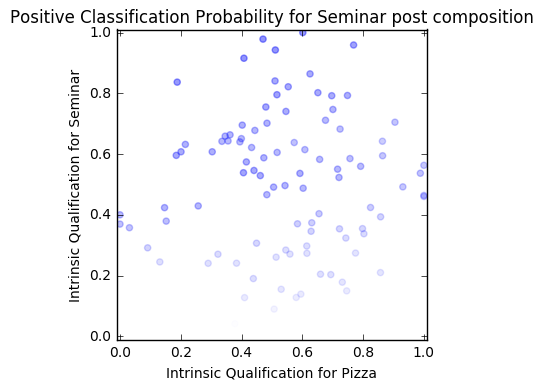}
    \caption{Task-competitive composition with $\Tiebreak(u,pizza)=.5$. Now we see the intensity of color fade in a less dramatic, diagonal pattern, which analogously appears in Figure \ref{fig:pizzaequalchoicecomp}. } \label{fig:seminarequalchoicecomp}
\end{figure}

\begin{figure}[H]
    \centering
    \includegraphics[width=0.7\textwidth]{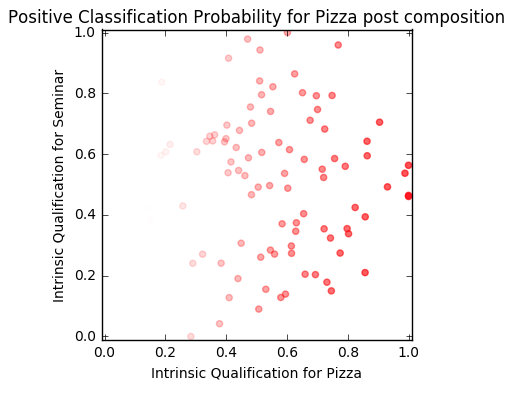}
    \caption{Task-competitive composition with $\Tiebreak(u,pizza)=.5$. The intensity of color fades in a less dramatic, diagonal pattern. } \label{fig:pizzaequalchoicecomp}
\end{figure}

Figures \ref{fig:pizzaequalchoicecomp} and \ref{fig:seminarequalchoicecomp} are the first setting we've seen where the composition results in unfairness for \textit{both} tasks, not just one or the other. As noted in Table  \ref{tab:compositionempiricalresults}, 
the maximum violations are smaller than in \competitivecomposition, but occur in both tasks, and the total number of pairs impacted is still significant.

Finally, we examine the \tiebreakingfunction which, when both pizza and seminar are options, selects pizza with probability equal to the intrinsic qualification for pizza. That is, if a person has qualification $q_{u,p}$, then their preference is $\Tiebreak(u,pizza)=q_{u,p}$. In figures \ref{fig:pizzaiq} and \ref{fig:seminariq} below, we see the same problem of color variation changing over vertical (or horizontal) lines, but also stretch in the scale of the intensity variation for pizza. 

\begin{figure}[H]
    \centering
    \includegraphics[width=0.7\textwidth]{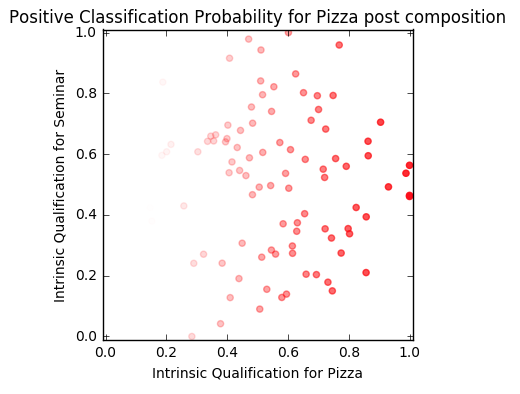}
    \caption{Task-competitive composition with $\Tiebreak(u,pizza)=q_{u,p}$. As in the previous compositions, we see a decrease in intensity from bottom to top. However, there is also a stretch in the intensity of the color, particularly visible when comparing the leftmost and rightmost elements. In particular, we see that the saturation on the right-hand side is more intense, and the saturation on the left-hand side is less intense than Figure \ref{fig:pizzawithoutseminar}.} \label{fig:pizzaiq}
\end{figure}

\begin{figure}[H]
    \centering
    \includegraphics[width=0.7\textwidth]{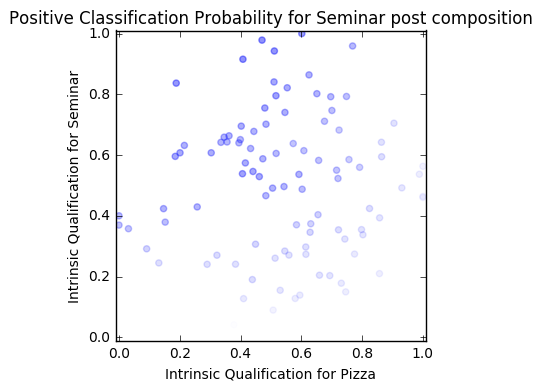}
    \caption{Task-competitive composition with $\Tiebreak(u,pizza)=q_{u,p}$. As in the previous compositions, we see that the intensity fades both from top to bottom and left to right.} \label{fig:seminariq}
\end{figure}

Our simple experiment gives good intuition for how easily a simple composition can result in unintended unfairness. Furthermore, as shown in Table \ref{tab:compositionempiricalresults}, the magnitude of the violations gives us intuition that a small $\varepsilon,\delta$ approximate definition is unlikely to fix the problem, as maximum violations are routinely larger than $0.2$.

\subsection{Fair Composition}
If we run \randomizeandclassify with pizza and seminar tasks each having probability $0.5$ in $\mathcal{X}$, we should see a reduction in expected utility of about 50\%, as each classifier only gets access to approximately half of the candidates. Empirically, we observed this to be about $50\%$ over $15$ trials. However, we can adjust the learned classifiers to try to compensate for this. Because our utility functions are very simple in our setting (giving out pizza or seminar invitations always has positive utility), these adjustments can be very simple. For example, we could increase the probability of positive classification in each classifier by $10\%$ for each candidate (with a maximum of $100\%$). This modification reduces the loss, as expected, to about $40\%$. Admittedly, our experimental setting has a very simple objective function, so more complicated settings may require more nuanced modifications to their training procedures to improve performance under \randomizeandclassify.
\end{document}